\documentclass[11pt]{article}

\usepackage[margin=1in]{geometry}

\usepackage[numbers,sort&compress]{natbib}

\usepackage[utf8]{inputenc} 
\usepackage[T1]{fontenc}    
\usepackage{hyperref}       
\hypersetup{hidelinks}
\usepackage{url}            
\usepackage{booktabs}       
\usepackage{makecell}       
\usepackage{tabularx}       
\usepackage[ruled,vlined,linesnumbered]{algorithm2e} 
\usepackage{amsmath}        
\usepackage{amsfonts}       
\usepackage{dsfont}         
\usepackage{amsthm}         
\usepackage{nicefrac}       
\usepackage{microtype}      
\usepackage{xcolor}
\usepackage{pifont}
\usepackage{graphicx}
\usepackage{subcaption}     
\usepackage{float}          
\usepackage{placeins}
\usepackage{enumitem}
\usepackage{tikz}
\usetikzlibrary{arrows.meta,positioning,calc,fit}

\setlist[enumerate]{leftmargin=0pt, labelsep=0.5em, align=left}

\newcommand{\bmedit}[1]{}
\DeclareMathOperator*{\argmin}{arg\,min}

\newtheorem{theorem}{Theorem}[section]
\newtheorem{proposition}[theorem]{Proposition}
\newtheorem{lemma}[theorem]{Lemma}
\newtheorem{corollary}[theorem]{Corollary}

\theoremstyle{definition}
\newtheorem{definition}[theorem]{Definition}

\newtheorem{example}[theorem]{Example}

\theoremstyle{remark}
\newtheorem{remark}[theorem]{Remark}

\newcommand{\cmark}{\ding{51}} 
\newcommand{\xmark}{\ding{55}} 

\begin{document}

\title{Sparse Dictionary Learning for Measure-Valued Data via Linearized Optimal Transport: Efficient Algorithms and Sample Complexity}
\author{
Brendan Mallery$^{1}$ \quad Shuchin Aeron$^{2,3,*}$ \quad James M. Murphy$^{4,*}$\\[0.6em]
\small $^1$Independent Researcher\\
\small $^2$Department of Electrical and Computer Engineering, Tufts University\\
\small $^3$The NSF AI Institute for Artificial Intelligence and Fundamental Interactions (IAIFI)\\
\small $^4$Department of Mathematics, Tufts University\\[0.4em]
\small $^*$Equal contribution.
}
\date{}
\maketitle

\begin{abstract}
We study dictionary learning for probability measure-valued data by embedding measures into Hilbert spaces using features derived from optimal transport geometry.  We learn data-adapted dictionaries through sparse autoencoder parameterizations, producing fast, scalable representations of probability measures as sparse finite-dimensional codes. The resulting framework combines the rich, geometric structure of optimal transport with the interpretability and efficiency of sparse representation learning. We demonstrate its utility across diverse data modalities, including hyperspectral images, geometric point clouds, and document embeddings in large language model latent spaces. We also prove generalization bounds for sparse dictionary learning objectives on broad classes of $L^2$-valued measure embeddings, expressed in terms of the metric complexity of the atom class. For optimal transport and entropic transport map dictionaries, we derive new covering number bounds that yield explicit rates.  Together, our computational and theoretical results support a new, scalable model for representation learning in the space of probability measures via embedding into Hilbert spaces.
\end{abstract}

\section{Introduction}

There has been much recent interest in learning and inference on probability measure-valued data via methods that utilize the rich, nonlinear geometry of the Wasserstein-2 space of probability measures \citep{arjovsky2017wasserstein, schiebinger2019optimal, lambert2022variational, lipmanflow2023, jiang2025algorithms}.  An emerging approach to data analysis in Wasserstein-2 space is \emph{Wasserstein dictionary learning (WDL):} learning data-driven \emph{dictionaries of probability measures} that are representative of the underlying dataset. These dictionaries allow for the \textit{synthesis} of new samples from latent codes, as well as the efficient \textit{analysis} of data through their representations with respect to the dictionary, enabling natural notions of dimensionality reduction, sparse representation, and compression \citep{agueh2011barycenters, bonneel2016wasserstein, schmitz2018wasserstein, cheng2020optimal, cheng2021dynamical, werenski2022measure, mueller2023geometrically,  cheng2023nonparametric, gunsilius2024tangential, werenski2025linearized, mallery2025synthesis}.
Unfortunately, WDL suffers from the well-known computational and statistical challenges
associated with scalable learning in the Wasserstein-2 space \citep{hutter2021minimax, altschuler2022wasserstein} as well as foundational challenges, such as understanding what data and modeling tasks benefit from a WDL-based approach.

\begin{figure}[t]
    \centering
    \resizebox{\textwidth}{!}{\input{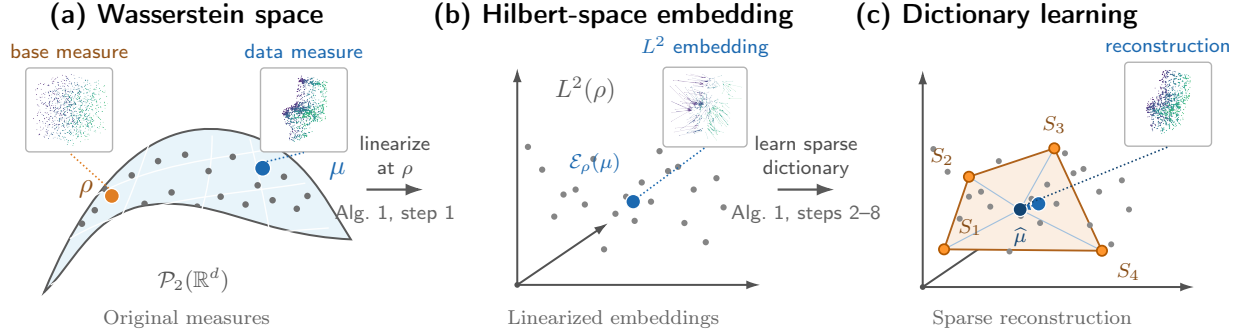}}
    \caption{Schematic of linearized Wasserstein dictionary learning. We define an embedding $\mathcal{E}_\rho$ from the nonlinear
    Wasserstein space into $L^2(\rho)$, viewed as the tangent space at some base measure $\rho$. In the same linearized space, a set of dictionary elements
    $\{S_j\}_{j=1}^m$ are learned; their sparse
    nonnegative combinations approximate the embeddings, which may then be synthesized into reconstructed measures.}
    \label{fig:lwdl-schematic}
\end{figure}

In this paper, we introduce novel formulations of \emph{linearized Wasserstein Dictionary Learning (LWDL)} that promote sparsity in the latent representation.  Our methods proceed via nonparametric embedding of probability measures into a suitable Hilbert space via the \emph{linearized optimal transport (LOT)} framework \citep{wang2013linear, moosmuller2023linear, delalande2023quantitative, werenski2025linearized}. We exploit the embedded Hilbert space structure to define a multi-layer sparse autoencoder \citep{daubechies2010iteratively, gregor2010learning} parameterization of the ensuing dictionary learning problem, which enables fast learning of the underlying atoms when compared to WDL methods based on the nonlinear Wasserstein geometry. Theoretically, the sample complexity of LWDL can be analyzed through the flat geometry of the underlying Hilbert space structure, enabling explicit generalization bounds that depend on covering numbers of certain subsets of the underlying Hilbert space.  Our main contributions are:
\begin{enumerate}[
    leftmargin=*,
    itemsep=0pt,
    topsep=2pt,
    parsep=5pt,
    label=\arabic*).
]
\item Introducing novel LWDL models that leverage sparsity-promoting penalties on the latent codes.
\item Efficient parametrizations of our linearized models via sparse autoencoders (SAEs) that enable faster representation learning compared to baselines.
\item Applications to point clouds, text documents, and hyperspectral image processing that reveal and leverage intrinsic structure in these problems.\footnote{Code to recreate the experiments is available on \href{https://github.com/brendanmallery9/Linearized-Wasserstein-Dictionary-Learning}{\textcolor{blue}{GitHub}}.}
\item Novel generalization bounds for LWDL that illustrate the dependence on the linearization model.
\end{enumerate}

 \section{Models, Methods, and Related Work}\label{sec:ModelsMethodsRelatedWork}

\subsection{Notation and Background}
We consider data $\{\mu_{i}\}_{i=1}^{n}$ lying in $\mathcal{P}(\mathbb{R}^{d})$, the space of probability measures over $\mathbb{R}^{d}$.  We let $M_{p}(\mu):=\int_{\mathbb{R}^{d}}\|x\|^{p}d\mu(x)$ be the $p^{th}$ moment of a probability measure $\mu$.  We let $\mathcal{P}_{2}(\mathbb{R}^{d})\subset\mathcal{P}(\mathbb{R}^{d})$ be the probability measures with finite second moment ($M_{2}(\mu)<\infty$).  For $\rho\in\mathcal{P}(\mathbb{R}^{d})$, we let $L^{2}(\rho;\mathbb{R}^{d}\rightarrow\mathbb{R}^{\ell})$ consist of functions $f:\mathbb{R}^{d}\rightarrow\mathbb{R}^{\ell}$ such that $\int_{\mathbb{R}^{d}}\|f(x)\|_{2}^{2}d\rho(x)<\infty$; we will typically write $L^2(\rho)$ when $d,\ell$ are clear from context.  Let $\Delta^{m-1}:=\{\lambda \in \mathbb{R}^m_{\geq 0}\;| \sum_{j=1}^{m}\lambda_{j}=1\}$ denote the standard simplex in $\mathbb{R}^{m}$.  For two measures $\mu,\nu$, we let $\Pi(\mu,\nu)\subset\mathcal{P}_{2}(\mathbb{R}^d\times\mathbb{R}^d)$ be the set of couplings with marginals $\mu$ and $\nu$.  We will use the notation $X_{1},\dots,X_{n}\sim \mu$ to mean that $X_{1},\dots,X_{n}$ is an i.i.d. sample from probability measure $\mu$.  We let $\mathcal{G}(\mathbb{R}^d)$ be the set of subgaussian probability measures on $\mathbb{R}^{d}$; see Appendix \ref{appendix:AdditionalTechnicalBackground} for details.  We let $|\Omega|$ denote the diameter of a set $\Omega$.

Given two probability measures $\mu,\nu\in\mathcal{P}_{2}(\mathbb{R}^{d})$, their \emph{Wasserstein-2 distance} (with respect to the ground cost $\tfrac{1}{2}\|x-y\|^2$) \citep{villani2009optimal} is:
$W_2(\mu,\nu):=\min_{\zeta\in \Pi(\mu,\nu)} \sqrt{\int \tfrac{1}{2}\|x-y\|^2d\zeta(x,y)}.$  By Brenier's theorem \citep{brenier1991polar}, if $\mu$ is absolutely continuous (a.c.) then there exists a
\emph{(Wasserstein-2) optimal transport map} $T_{\mu\rightarrow\nu}:\mathbb{R}^{d}\rightarrow\mathbb{R}^{d}$ between $\mu$ and $\nu$ satisfying $W_{2}^{2}(\mu,\nu)=\inf_{T_\#\mu=\nu}\int\tfrac{1}{2}\|T(x)-x\|^{2}d\mu(x)=\int\tfrac{1}{2}\|T_{\mu\rightarrow\nu}(x)-x\|^{2}d\mu(x)$.  See Appendix \ref{appendix:AdditionalTechnicalBackground} for details.

For any $\epsilon>0$, the \emph{entropy-regularized Wasserstein-2 cost} (EOT cost) is \citep{nutz2021introduction}
\[
W_{2,\epsilon}(\mu,\nu)
:=\min_{\zeta\in\Pi(\mu,\nu)}
\Bigl(\int \tfrac{1}{2}\|x-y\|^2\,d\zeta(x,y)+\epsilon KL(\zeta\|\mu\otimes\nu)\Bigr)^{1/2}.
\]
where $KL(\zeta\|\mu\otimes\nu)$ is the KL-divergence/relative entropy \citep{cover1999elements}. For $\pi_{\epsilon}$ the solution to the EOT cost, define the \textit{entropic OT (EOT) map} from $\mu$ to $\nu$ to be $T_{\mu\rightarrow\nu}^{\epsilon}(x)=\mathbb{E}_{(X,Y)\sim\pi_{\epsilon}}[Y | X=x]$; see Appendix \ref{appendix:AdditionalTechnicalBackground} for further details and alternative constructions.  EOT maps are smooth surrogates of OT maps, with $T_{\mu\rightarrow\nu}^{\epsilon}$ converging to $T_{\mu\rightarrow\nu}$ as $\epsilon\rightarrow 0^{+}$ \citep{pooladian2021entropic} under mild assumptions.  Fix an a.c. probability measure $\rho$.  Then the \textit{$\rho$-linearized optimal transport metric} ($\rho$-LOT) is $d_{LOT}(\mu,\nu):=\|T_{\rho\rightarrow\mu}-T_{\rho\rightarrow\nu}\|_{L^2(\rho)}$ \citep{wang2013linear}.  See Appendix \ref{appendix:Background on LOT} for details.

 \subsection{Linearized Models for WDL}\label{sec:linearized models for wdl}

 Fix an a.c. probability measure $\rho\in\mathcal{P}_2(\mathbb{R}^{d})$ and consider an embedding function $\mathcal{E}_\rho: \mathcal{P}_2(\mathbb{R}^d)\rightarrow L^2(\rho)$ that embeds a measure into $L^2(\rho)$.  Let $ \mathbb{T}\subseteq L^2(\rho)$.  Given a probability measure over $\mathcal{P}_{2}(\mathbb{R}^{d})$, $\Theta\in\mathcal{P}(\mathcal{P}_{2}(\mathbb{R}^{d}))$, the corresponding \emph{$(m,\mathbb{T}, \Theta, \rho)$-Linearized Measure-Valued Dictionary Learning  (LMDL)} problem is:
\begin{equation}\label{eqn:LDL objective}
\min_{(S_j)_{j=1}^m\in\mathbb{T}^m}
\underset{\mu\sim\Theta}{\mathbb{E}}\!
\Bigl[
\min_{\lambda\geq0}
\tfrac12
\|\mathcal{E}_\rho(\mu)-\sum_{j=1}^m\lambda_jS_j\|_{L^2(\rho)}^2+g(\lambda)
\Bigr].
\end{equation}
where $g:\mathbb{R}_{\ge 0}^{m} \rightarrow \mathbb{R}_{\ge 0}$ is a sparsity-inducing penalty function, e.g. $g(\lambda):=c\|\lambda\|_{1}, c > 0$.
Thus, LMDL learns $m$ atoms whose sparse, non-negative linear combinations approximate well the embedded measures.  This is an analogue of sparsely used overcomplete dictionary learning for probability-measure-valued data.  We focus on several choices of $\mathcal{E}_\rho$ and $\mathbb{T}$, motivated by the Wasserstein-2 geometry, which we give the name \emph{linearized Wasserstein dictionary learning (LWDL)}, laid out in Table \ref{tab:lwdl-embeddings}.

\begin{table}[t]
\centering
\small
\begin{tabular}{lll}
\toprule
Variant & $\mathcal E_\rho(\mu)$ & Setting \\
\midrule
LWDL-OT  & $T_{\rho\to\mu}$ & exact, 1D \\
LWDL-EOT & $T^\epsilon_{\rho\to\mu}$& general $d$ \\
LWDL-KP  & $f_{\rho\to\mu}$ & high $d$ \\
\bottomrule
\end{tabular}
\caption{LWDL embeddings used in this work.}
\label{tab:lwdl-embeddings}
\end{table}
One may interpret LWDL-OT, LWDL-EOT, and LWDL-KP as different linearizations of the WDL problem; see Appendix \ref{appendix:LWDL as WDL Linearization} for details. The different formulations are best suited for different modeling contexts: LWDL-OT keeps exact optimal transport-map structure but is implemented only in one dimension; LWDL-EOT gives a finite-dimensional map parameterization that works in dimensions higher than one; LWDL-KP replaces vector-valued maps by scalar potentials, which is useful for very high-dimensional coding.

\section{Algorithms for LWDL}
\label{sec:Algorithms}

Given sampled measures $\{\mu_i\}_{i=1}^N$, we implement LWDL using a shared LISTA-type encoder for the inner sparse-coding problem and method-specific atom parameterizations for LWDL-OT, LWDL-EOT, and LWDL-KP. We train the encoder and atoms jointly using the empirical version of (\ref{eqn:LDL objective}).

\begin{algorithm}[t]
\caption{Generic LWDL training algorithm}
\label{alg:generic-lwdl}

\KwIn{
Samples $\{\mu_i\}_{i=1}^N$;
reference measure $\rho$;
embedding map $\mathcal E_\rho$;
dictionary atoms
$\mathcal S^{\theta_a}
=\bigl(S_1^{\theta_a},\dots,S_m^{\theta_a}\bigr)$,
$S_j^{\theta_a}\in L^2(\rho)$,
parameterized by $\theta_a$;
number of encoder layers $L$;
encoder parameters $\theta_c$;
activation maps
$\{\chi_{r,\theta_c}\}_{r=0}^{L-1}$;
sparsity parameter $c>0$;
batch size $B$;
parameters $\theta=(\theta_a,\theta_c)$.
}

\vspace{0.5em}

Compute the method-specific embeddings $\mathcal E_\rho(\mu_i)\in L^2(\rho), i=1,\dots,N,$
using the construction for LWDL-OT, LWDL-EOT, or LWDL-KP
given in Algorithms~\ref{alg:OTSAE},
\ref{alg:entropic-lwdl}, and
\ref{alg:KP_SAE_embedding}, respectively.

\While{not converged}{

    Sample a minibatch
    $\mathcal B\subset\{1,\dots,N\}$ with $|\mathcal B|=B$.

    Construct the method-specific dictionary atoms $\mathcal S^{\theta_a}=\bigl(S_1^{\theta_a},\dots,S_m^{\theta_a}\bigr)$ using the parameterization specified in
    Algorithm~\ref{alg:OTSAE},
    \ref{alg:entropic-lwdl}, or
    \ref{alg:KP_SAE_embedding}.

    \For{$i\in\mathcal B$}{
        Compute the sparse nonnegative code using
        Algorithm~\ref{alg:generic-lista-encoder}:
        \[
            \lambda(\mu_i;\theta_a,\theta_c)
            \gets
            \operatorname{LISTAEncoder}
            \bigl(
                \mathcal E_\rho(\mu_i),
                \mathcal S^{\theta_a};
                \theta_c
            \bigr).
        \]
    }

    Compute
    \[
    \mathcal L_{\mathcal B}(\theta)
    \gets
    \frac{1}{|\mathcal B|}
    \sum_{i\in\mathcal B}
    \left[
        \frac12
        \left\|
        \mathcal E_\rho(\mu_i)
        -
        \sum_{j=1}^m
        \lambda_j(\mu_i;\theta_a,\theta_c)
        S_j^{\theta_a}
        \right\|_{L^2(\rho)}^2
        +
        c
        \left\|
        \lambda(\mu_i;\theta_a,\theta_c)
        \right\|_1
    \right].
    \]

    Update $\theta=(\theta_a,\theta_c)$ by batched SGD
    using AdamW and a cosine learning-rate schedule.
}

\vspace{0.5em}

\KwOut{
Learned parameters $\theta=(\theta_a,\theta_c)$;
dictionary atoms $\mathcal S^{\theta_a}$;
sparse nonnegative codes
$\{\lambda(\mu_i;\theta_a,\theta_c)\}_{i=1}^N$.
}

\end{algorithm}

Our meta-algorithm is Algorithm \ref{alg:generic-lwdl}. Here, $\theta_a$ parameterizes the atoms $\mathcal S^{\theta_a}=(S_1^{\theta_a},\dots,S_m^{\theta_a})$, while $\theta_c$ parameterizes a shared encoder. Given $\mathcal E_\rho(\mu)$, the encoder forms the correlation vector
\[
h(\mu,\theta_a)=\bigl(\langle S_j^{\theta_a},\mathcal E_\rho(\mu)\rangle_{L^2(\rho)}\bigr)_{j=1}^m.
\]
Starting from $\lambda^{(0)}=0$, it then applies the $L$ learned updates
\begin{equation}\label{eqn:body LISTA recursion}
\begin{aligned}
\lambda^{(r+1)}=\chi_{r,\theta_c}\!\bigl(&A_{r,\theta_c}h(\mu,\theta_a)
+B_{r,\theta_c}\lambda^{(r)}\\
&+b_{r,\theta_c}\bigr),
\end{aligned}
\end{equation}
and returns $\lambda(\mu;\theta_a,\theta_c)=\lambda^{(L)}\in\mathbb R_{\geq0}^m$.  This encoder is LISTA-type \citep{gregor2010learning}: it retains an iterative affine-thresholding form, but its weights and activations are not tied to a fixed ISTA update.  The decoder is the positive combination $\widehat{\mathcal E}_\rho(\mu;\theta_a,\theta_c)=\sum_{j=1}^m\lambda_j(\mu;\theta_a,\theta_c)S_j^{\theta_a}$; see Appendix \ref{appendix:LISTAEncoder} and Algorithm \ref{alg:generic-lista-encoder} for details.

Our meta-algorithm specializes into our three linearizations via the atom parameterization methods, which match the set of atoms to the choice of embedding $\mathcal E_\rho$:
\begin{itemize}[leftmargin=*,itemsep=1pt,topsep=2pt,parsep=0pt]
\item \textbf{LWDL-OT.} We restrict LWDL-OT to one-dimensional measures, taking $\mathcal E_\rho(\mu)=T_{\rho\to\mu}.$ The dictionary atoms are monotone maps parameterized by normalized positive increments, thereby preserving the structure of exact optimal transport maps (Appendix \ref{appendix:LWDL-OTDetails}; Algorithm \ref{alg:OTSAE}).

\item \textbf{LWDL-EOT.} For LWDL-EOT, each atom is parameterized by a finite-dimensional entropic transport map $T_j^{\epsilon,\theta_a}$ over a fixed target support, yielding a smooth map parameterization in arbitrary dimension. Empirically, we find that displacement fields perform better than the maps themselves, so we use $\mathcal E_\rho(\mu)=T^\epsilon_{\rho\to\mu}-\mathrm{Id},
\qquad
S_j^{\theta_a}=T_j^{\epsilon,\theta_a}-\mathrm{Id}.$ (Appendix \ref{appendix:LWDL-EOTDetails}; Algorithm \ref{alg:entropic-lwdl}.)

\item \textbf{LWDL-KP.} Finally, LWDL-KP uses the centered scalar Kantorovich potential $\mathcal E_\rho(\mu)=f_{\rho\to\mu},$ with each dictionary atom given by a scalar function on $\operatorname{supp}(\rho)$ centered according to the same convention. This replaces vector-valued transport maps with scalar functions, making the representation particularly suitable for high-dimensional settings (Appendix \ref{appendix:LWDL-KPDetails}; Algorithm \ref{alg:KP_SAE_embedding}).
\end{itemize}

We train the encoder and atoms end-to-end by minimizing
\begin{equation}\label{eqn:discrete LWDL objective}
\begin{aligned}
\min_{\theta=(\theta_a,\theta_c)}\frac1N\sum_{i=1}^N\Bigl[&
\tfrac12\left\|\mathcal E_\rho(\mu_i)-\widehat{\mathcal E}_\rho(\mu_i;\theta_a,\theta_c)\right\|_{L^2(\rho)}^2\\
&{}+c\|\lambda(\mu_i;\theta_a,\theta_c)\|_1\Bigr].
\end{aligned}
\end{equation}
We optimize (\ref{eqn:discrete LWDL objective}) by minibatch SGD with AdamW \citep{loshchilov2017decoupled} and a cosine learning-rate schedule, as summarized in Algorithm \ref{alg:generic-lwdl}. In our numerical implementations, we take the base measure to be finitely supported, $\rho=\sum_{k=1}^M w_k\delta_{x_k}$, and represent every embedded datum and dictionary atom $F\in L^2(\rho;\mathbb R^d\to\mathbb R^\ell)$ by its evaluations $(F(x_k))_{k=1}^M\in\mathbb R^{M\times\ell}$. Embeddings and atoms are represented this way throughout Algorithm \ref{alg:generic-lwdl}; consequently, the encoder inner products, decoder combinations, and reconstruction losses are all finite-dimensional, with $L^2(\rho)$ quantities reducing to weighted sums such as
$\|F-G\|_{L^2(\rho)}^2=\sum_{k=1}^M w_k\|F(x_k)-G(x_k)\|_{\mathbb R^\ell}^2$.
Here $\ell=d$ for the map embeddings used by LWDL-OT and LWDL-EOT, while $\ell=1$ for the potential embeddings used by LWDL-KP. Full algorithms are provided in Appendix \ref{appendix:AlgorithmDetails}.

\subsection{Related Work}


A table of related methods to our proposed framework appears in Table \ref{tab:RelatedWork} in Appendix \ref{appendix:RelatedWorkTable}.  The most important comparison is WDL \citep{schmitz2018wasserstein}, which considers, for finite, deterministic training data set $\{\mu_{i}\}_{i=1}^{N}$, $\displaystyle\argmin_{\mathcal{V}\in \mathcal{P}_2(\mathbb{R}^d)^m}\min_{\{\lambda_{i}\}_{i=1}^{N}\subset \Delta^{m-1}}\sum_{i=1}^{N}W_2^{2}\left(\mu_{i},\mu_{\lambda_{i},\mathcal{V}}\right).$
This was generalized in \citet{mueller2023geometrically} to consider geometric regularization on the learning problem to enforce that learned reference measures lie near the data points that are represented using them. These approaches work in the (nonlinear) Wasserstein-2 geometry and are thus very slow, as the synthesis problem can only be solved iteratively by fixed point methods, requiring many expensive computations of the Wasserstein-2 distance; see Table \ref{tab:wdl-comparison-results} for a runtime comparison. Furthermore, these methods are fixed-support, meaning the data, atoms and synthesized reconstructions are distributions on a fixed, finite set of points, as opposed to being allowed to vary freely in the domain. Finally, the complex geometry of the Wasserstein-2 metric also renders theoretical analysis of the problem difficult. Several works have considered fixing reference measures and learning coefficients that optimally represent observed data as Wasserstein barycenters (or entropy-regularized variants thereof) \citep{bonneel2016wasserstein, werenski2022measure, mallery2025synthesis}.  These methods work in the fully nonlinear geometry of the Wasserstein space and do not learn reference measures.  \citet{gunsilius2024tangential, werenski2025linearized} consider learning optimal coefficients with respect to the LOT metric for a fixed set of reference measures.  This can be understood as solving (\ref{eqn:LDL objective}) with no optimization over $\mathcal{V}$ and with $g \equiv 0$.  A related framework is developed in \citet{do2025sparse} with explicit sparsity penalties on $\lambda$, achieved via greedy algorithms in learning the reconstructions from the (fixed) reference measures.

\section{Numerical Results}\label{sec:numerical results}

\subsection{Reconstruction and Identifiability on a Synthetic LBCM}\label{sec:recovering data from synthetic}
We evaluate the proposed algorithms on a synthetic recovery problem in which the generating dictionary is known; details appear in Appendix \ref{appendix:SyntheticLBCMDetails}. We select five MNIST transport-map atoms corresponding to digits $\mathcal D=\{0,2,3,7,8\}.$  We take \(\rho\) to be an empirical measure supported on \(M=400\) points sampled uniformly from \([0,1]^2\). For each \(j\in\mathcal D\), we use a precomputed MNIST transport map \(T_{\rho\rightarrow \nu_{j}}\) as a planted atom. Coefficients are generated by first sampling a two-element subset \(I_i\subset\{1,\dots,5\}\) uniformly at random, then sampling $(\lambda_j^i)_{j\in I_i}\sim\operatorname{Dirichlet}(1,1)$, and setting \(\lambda_j^i=0\) for \(j\notin I_i\). Thus \(\lambda^i\in\Delta^4\) and \(\|\lambda^i\|_0=2\). We define $T_i=\sum_{j=1}^5 \lambda_j^i T_{\rho\rightarrow \nu_{j}}, i=1,\dots,60000.$  The resulting data set is split randomly into \(54000\) training maps and \(6000\) held-out test maps.

\begin{figure}[t!]
\centering

  \begin{subfigure}[t]{0.46\linewidth}
    \centering
    \includegraphics[
      width=\linewidth,
      height=0.24\textheight,
      keepaspectratio
    ]{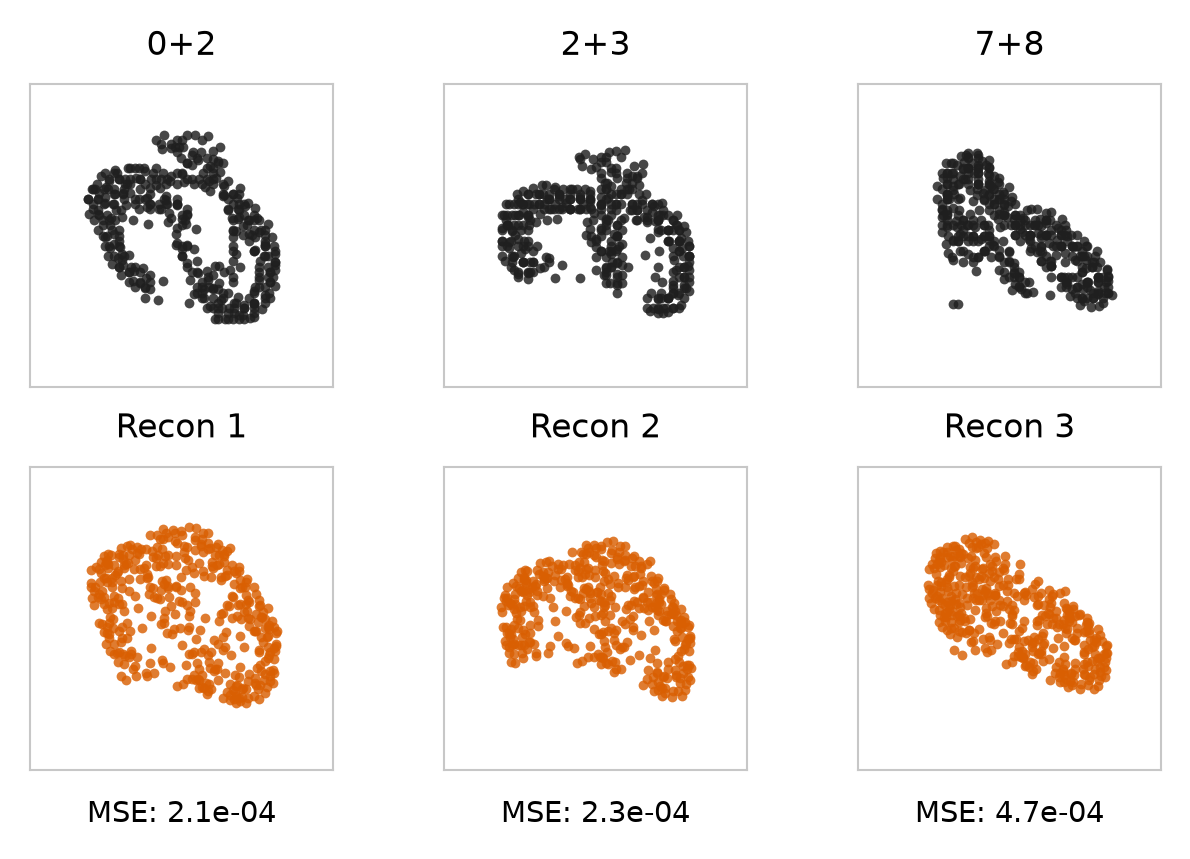}
    \phantomsubcaption
    \label{fig:mnist-ebcm-reconstructions}
  \end{subfigure}
  \hfill
  \begin{subfigure}[t]{0.43\linewidth}
    \centering
    \includegraphics[
      width=\linewidth,
      height=0.24\textheight,
      keepaspectratio
    ]{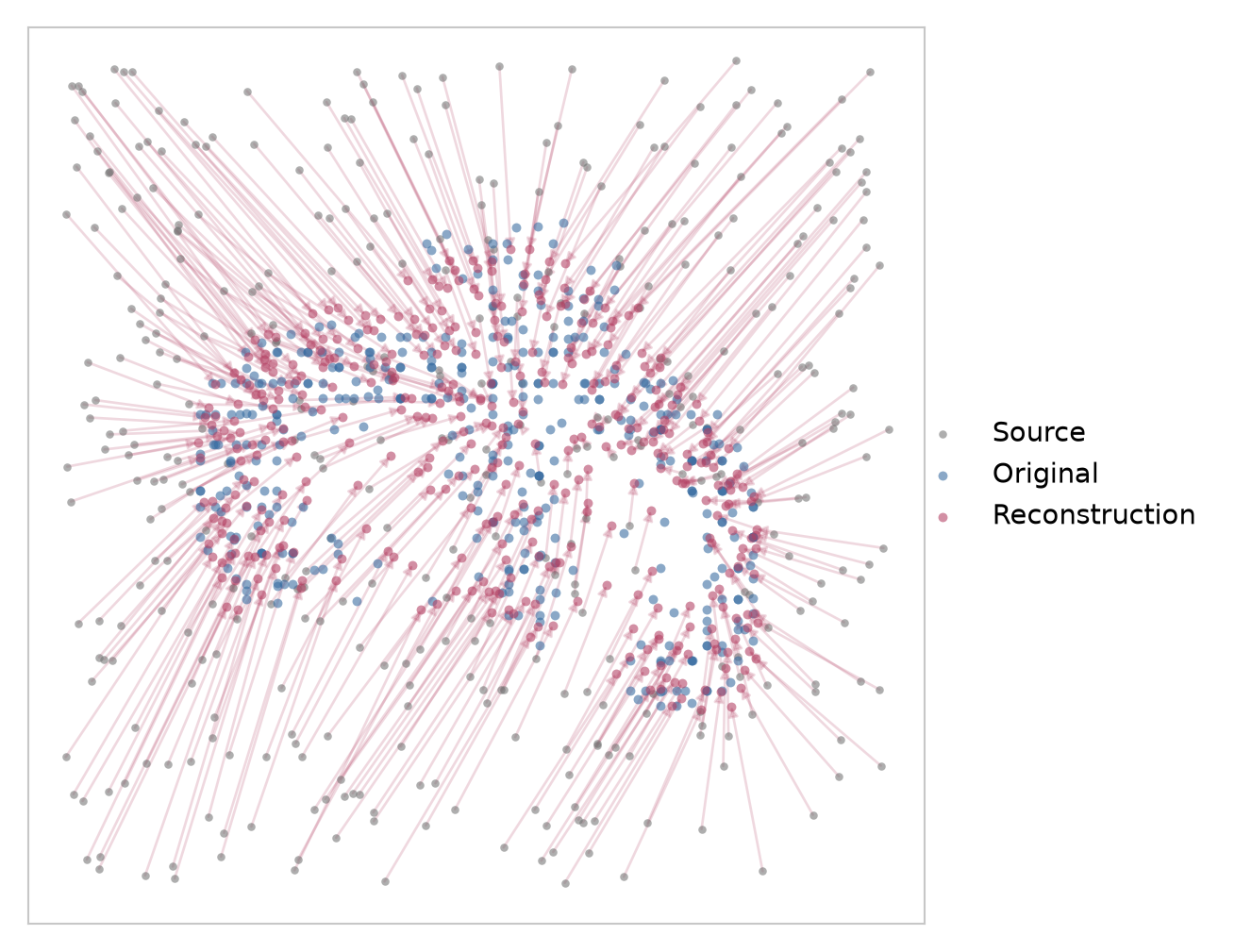}
    \phantomsubcaption
    \label{fig:mnist-ebcm-map-overlay}
  \end{subfigure}

  \vspace{0.35em}

  \begin{subfigure}[t]{\linewidth}
    \centering
    \includegraphics[
      width=0.88\linewidth
    ]{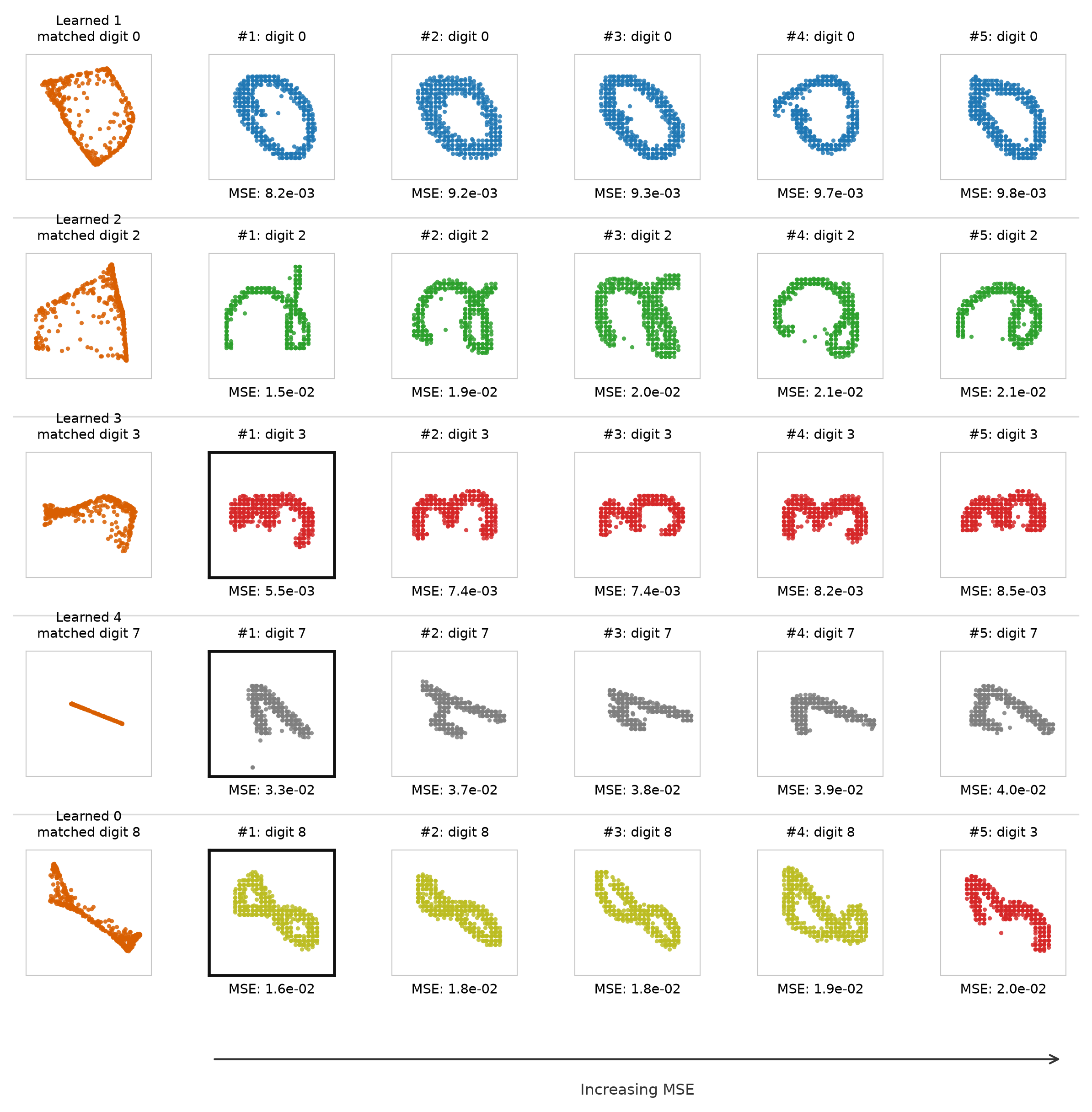}
    \phantomsubcaption
    \label{fig:mnist-ebcm-nearest-neighbors}
  \end{subfigure}

  \caption{
    LWDL-EOT recovery on sparse convex combinations of five MNIST transport-map atoms,
    corresponding to digits \(0,2,3,7,8\).
    (Top left): Representative held-out convex-combination maps and their LWDL-EOT reconstructions.
    (Top right): Source points, original target map, and learned reconstruction map in a held-out example.
    (Bottom): Nearest-neighbor search for learned LWDL-EOT atoms against a held-out pool restricted to digits in $\mathcal D$.
    Original generating atoms are included as reference candidates and outlined in bold.
  }
  \label{fig:mnist-ebcm-synthetic-recovery}
\end{figure}

We run Algorithm \ref{alg:generic-lista-encoder} and evaluate on the test set the mean \(L^2(\rho)\) reconstruction error
\[
\frac{1}{6000}
\sum_{i\in\mathrm{test}}
\frac{1}{M}
\sum_{k=1}^M
\left\|
T_i(x_k)-\widehat T_i(x_k)
\right\|_2^2.
\]
The selected checkpoint achieves held-out error \(\|T_i-\widehat T_i\|_{L^2(\rho)}^2=9.22\times 10^{-4}\), showing that the data maps can be reconstructed to a high degree by our model; see Figures \ref{fig:mnist-ebcm-reconstructions} and \ref{fig:mnist-ebcm-map-overlay}.

Next, we investigate whether the method can correctly learn atoms which resemble the ground truth atoms $\{T_{\rho\rightarrow\nu_{j}}\}_{j=1}^5$. We first compute a matching between learned atoms and true atoms by solving
\begin{equation}\label{eqn:matching}
\min_{\sigma\in S_5}
\frac{1}{5}\sum_{j=1}^5
\frac{1}{M}
\sum_{k=1}^M
\left\|
A_j^\theta(x_k)
-
T_{\rho\rightarrow \nu_{d_{\sigma(j)}}}(x_k)
\right\|_2^2,
\end{equation}
with a final value of $1.74\times 10^{-2}$, which is notably higher than the reconstruction error for the average data point, and when we visually inspect the atoms (Figure \ref{fig:mnist-ebcm-nearest-neighbors}), we see they qualitatively have weak resemblance to digits in the data set. We perform an additional qualitative check by computing the $L^2(\rho)$-nearest neighbors for each learned atom in a held-out pool restricted to digits in $\mathcal{D}$, with the planted atoms included as reference candidates, and plot the top five neighbors (Figure \ref{fig:mnist-ebcm-nearest-neighbors}). Despite the visual dissimilarity to the underlying digits, we see that the nearest neighbors to each learned atom are dominated by data from a unique digit in $\mathcal{D}$. Of the 25 displayed neighbors, 24 have the digit matched to the corresponding learned atom, and the top-ranked neighbor is the matched digit for all five atoms. This gives evidence that the method learns atoms that are able to at least identify the classes of the true atoms, and in some cases can even identify the true atom out of a lineup of highly similar data points. 

Finally, we assess whether the learned coefficients $\{\hat{\lambda}^i\}_{i=1}^{6000}$ recover the coefficients $\{\lambda^i\}_{i=1}^{6000}$ used to generate the data. To align coordinates, we use the matching $\sigma^*$ obtained in~(\ref{eqn:matching}). The average aligned coefficient MSE is $8.50 \times 10^{-3}$. Thus, despite the residual error in the recovered atoms, the learned coefficients closely approximate the ground truth coefficients.

\subsection{Comparison with WDL}
\label{subsec:ComparisonWDL}
We next compare WDL \citep{schmitz2018wasserstein} against LWDL-OT and LWDL-EOT, on 1D and 2D data sets respectively. Our goal is to evaluate both methods in terms of their reconstruction accuracy for probability-measure-valued data, measured by the squared $W_2$ distance, and their scalability with respect to data set size; details for these experiments are in Appendix \ref{appendix:ComparisonWDLDetails}.

{
\setlength{\dbltextfloatsep}{6pt}
\setlength{\abovecaptionskip}{2pt}
\setlength{\belowcaptionskip}{2pt}

\begin{table}[!t]
\centering
\scriptsize
\setlength{\tabcolsep}{5pt}
\renewcommand{\arraystretch}{0.9}

\caption{
Comparison with WDL under a 1000s time budget.
Mean $W_2^2$ is the final squared Wasserstein reconstruction error.
$^{*}$ denotes early termination due to numerical instability.  LWDL uses $B=100$ for $n=100$ and $B=256$ for $n=1000$;
LWDL-EOT uses $\epsilon=0.025$ and $c=5\times10^{-5}$.
}
\label{tab:wdl-comparison-results}

\begin{tabular}{llccc}
\toprule
Data Set & Method & $n$ & Mean $W_2^2$ & Epochs \\
\midrule

Pavia
& WDL ($\gamma=.5$)
& 100
& $4.6\times10^{-3}$
& 117 \\

&
WDL ($\gamma=2$)
& 100
& $4.9\times10^{-3}$
& 84 \\

&
LWDL-OT
& 100
& $\mathbf{5.0\times10^{-5}}$
& 240438 \\

\cmidrule(lr){2-5}

&
WDL ($\gamma=.5$)
& 1000
& $5.7\times10^{-3}$
& $7^{*}$ \\

&
WDL ($\gamma=2$)
& 1000
& $6.0\times10^{-3}$
& $7^{*}$ \\

&
LWDL-OT
& 1000
& $\mathbf{9.3\times10^{-5}}$
& 65183 \\

\midrule

MNIST
& WDL ($\gamma=.5$)
& 100
& $1.2\times10^{-2}$
& 101 \\

&
WDL ($\gamma=2$)
& 100
& $1.0\times10^{-2}$
& 101 \\

&
LWDL-EOT
& 100
& $\mathbf{2.6\times10^{-3}}$
& 50978 \\

\cmidrule(lr){2-5}

&
WDL ($\gamma=.5$)
& 1000
& $1.3\times10^{-2}$
& 10 \\

&
WDL ($\gamma=2$)
& 1000
& $1.3\times10^{-2}$
& 11 \\

&
LWDL-EOT
& 1000
& $\mathbf{2.7\times10^{-3}}$
& 12965 \\

\bottomrule
\end{tabular}

\end{table}
}

For the comparison of WDL with LWDL-OT, we sample a set of one-dimensional probability measures, all supported on 102 points in the unit interval. Specifically, we generate the data by normalizing pixels from the Pavia hyperspectral image data set; see Section \ref{sec:lwdl one dim} and Appendix \ref{sec:additional details hsi} for details. The final reconstruction errors, reported as the squared $W_2$ distance between each data point and its reconstruction after the 1000s time budget or numerical failure, are shown in Table \ref{tab:wdl-comparison-results}.

Across all data sets, the LWDL methods outperform our implementation of WDL within the prescribed time budget, achieving mean $W_2^2$ reconstruction losses that range from roughly $4$ to $100$ times lower .  This can be partially explained by the fact that, within the fixed time budget, the LWDL methods complete substantially more epochs than WDL. While epochs are not directly comparable across the methods because their optimization subroutines differ, the discrepancy illustrates the substantially lower per-epoch computational cost of LWDL. Inspecting the run metadata, we see that the majority of the time spent by the WDL method is on applying dense kernel matrices, a characteristic cost of Sinkhorn algorithm-based Wasserstein barycenter methods such as \cite{schmitz2018wasserstein}. In contrast, the linearized methods incur the expensive OT or EOT map computation only once; thereafter, training reduces to standard autoencoder optimization, enabling efficient parallelization and GPU acceleration.

\section{Applications}\label{sec:applications}

We evaluate our methods on hyperspectral imaging, 3D point clouds, and transformer latent spaces. Each dataset highlights a different linearization strategy, illustrating the flexibility of our framework across diverse data modalities.

\subsection{Hyperspectral Image Analysis with One-Dimensional LWDL-OT}\label{sec:lwdl one dim}

A hyperspectral image has spatial pixels that record an intensity vector across many wavelength bands, so a single pixel can be viewed as a nonnegative spectrum; see Appendix \ref{sec:additional details hsi}.  Each hyperspectral pixel is normalized into a probability measure $\mu$
on the wavelength grid, rescaled to the interval $[0,1]$, and embedded by the one-dimensional transport map $T_{\rho\rightarrow \mu}$ from a fixed reference measure $\rho$. We may pushforward by $[\sum_{j=1}^m\lambda_j T_j]_{\#}(\rho)$ to recover spectra. We consider three HSIs, Pavia Centre, Botswana and Salinas A \citep{ehu_hyperspectral_scenes}. For each data set, the dictionary size is set to one more than the number of non-background classes, which gives the dictionary enough capacity to learn the labeled material classes while keeping learned code dimension comparable across baselines.

\FloatBarrier

\begin{figure}[!b]
\centering

\includegraphics[
  width=\linewidth,
  trim=3pt 8pt 3pt 1cm,
  clip
]{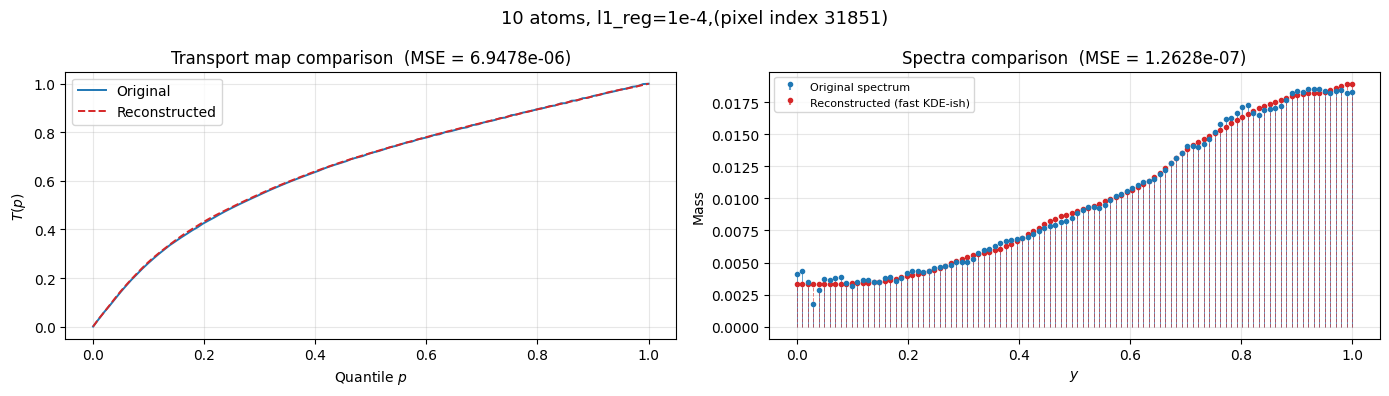}

\vspace{0.4em}

\includegraphics[
  width=\linewidth,
  trim=3pt 8pt 3pt 1cm,
  clip
]{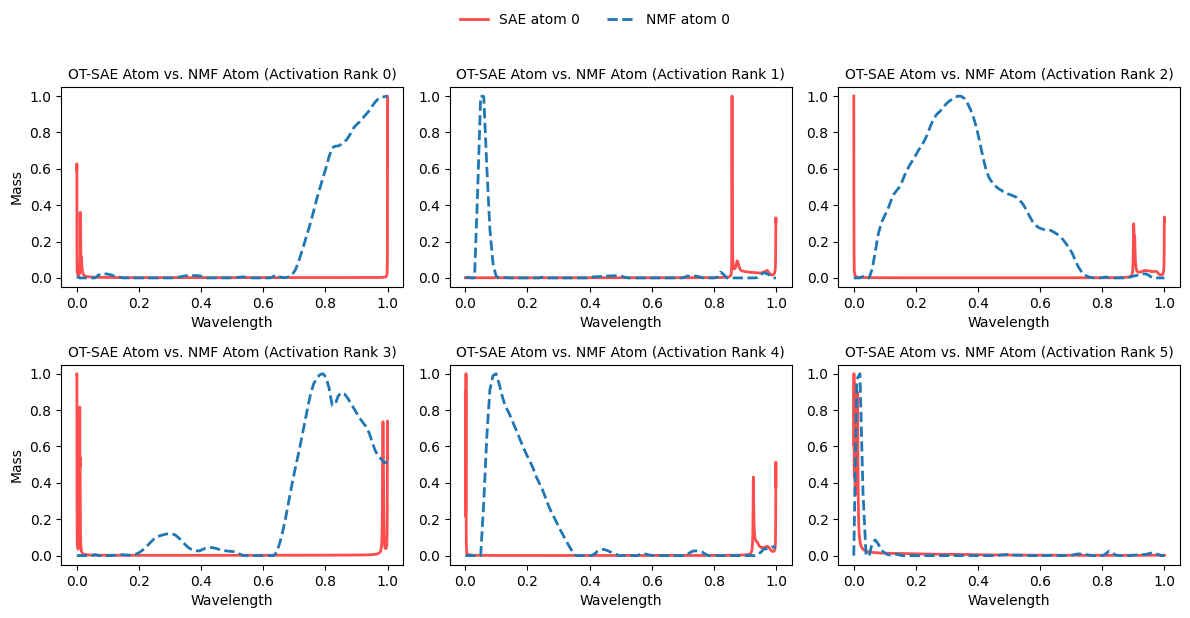}

\caption{\emph{Top}: example spectrum reconstruction. \emph{Bottom}: atoms learned from LWDL-OT and NMF.}
\label{fig:hsi-reconstruction-comparison}
\end{figure}

\begin{table}[t]
\centering
\caption{Reconstruction errors on Pavia spectra; lower is better.}
\label{tab:hsi-reconstruction-errors}

\small
\setlength{\tabcolsep}{6pt}
\renewcommand{\arraystretch}{1.1}

\begin{tabular}{@{}lcc@{}}
\toprule
\textbf{Method} & \textbf{Mean $L^2$} & \textbf{Mean $W_2$} \\
\midrule
LWDL-OT & $2.72 \times 10^{-6}$ & \textbf{$4.71 \times 10^{-5}$} \\
NMF     & \textbf{$2.98 \times 10^{-7}$} & $6.69 \times 10^{-5}$ \\
\bottomrule
\end{tabular}
\end{table}

We compare reconstruction error for LWDL-OT and Nonnegative Matrix Factorization (NMF) dictionaries \citep{berman1994nonnegative, gillis2020nonnegative}, a standard benchmark method for HSI processing, on the Pavia image. We see that both methods are able to reconstruct the clean spectra with low error.  Note that LWDL-OT outperforms NMF for reconstruction with the $W_2$-metric, while being outperformed by NMF with respect to $L^2([0,1])$ (Figure~\ref{fig:hsi-reconstruction-comparison}, right). Interestingly, the resulting transport atoms are visually different from NMF atoms: rather than matching average spectra directly, they often concentrate mass near endpoints of the wavelength interval (Figure \ref{fig:hsi-reconstruction-comparison}). This behavior is consistent with the convex geometry of monotone transport maps on $[0,1]$, where step-like maps (which pushforward to distributions supported on $\{0,1\}$) form extreme points of the one-dimensional transport-map set  \citep{werenski2025linearized}. 

We evaluate whether the learned codes are useful for unsupervised clustering under corruptions. LWDL-OT is compared against NMF and a nonnegative sparse autoencoder (N-SAE), which is trained directly on normalized spectra. Dictionaries are trained on clean spectra and then fixed while evaluating corrupted spectra; downstream labels are obtained by clustering the learned codes and matching clusters to ground truth.  We test on three types of corruptions: DropUniform, which uniformly drops out a percentage of spectral bands, DropContiguous, which drops out a randomly selected contiguous interval of spectral bands, and LogWarp which applies a parameterized diffeomorphism. Clean-data performance is competitive, with each method performing best on one data set at $k=0$. For the fixed configurations reported in Table~\ref{tab:main_dropuniform_summary}, LWDL-OT achieves the highest accuracy at 13 of the 15 nonzero DropUniform corruption levels and the least-negative fitted stability slope on all three data sets. Additional corruption types and configuration tables are given in Appendix~\ref{sec:additional details hsi}.

\begin{table}[t]
\centering
\caption{Accuracy under DropUniform corruption. Each row uses one fixed
configuration across all corruption levels. $k$ is the fraction of spectral
bands dropped. Accuracies are averaged over five corruption seeds and, for
LWDL-OT and N-SAE, five training seeds. Stability is the least-squares slope
of these mean accuracies versus $k$; values closer to zero indicate greater
robustness.}
\label{tab:main_dropuniform_summary}

\footnotesize
\setlength{\tabcolsep}{4pt}
\renewcommand{\arraystretch}{0.95}

\begin{tabular}{llccccccc}
\hline
\textbf{Data}
& \textbf{Method}
& \multicolumn{6}{c}{\textbf{Corruption level $k$}}
& \textbf{Stab.} \\
\cline{3-8}
&& $0$ & $.1$ & $.2$ & $.3$ & $.4$ & $.5$ & \\
\hline

Pavia
& NMF
& .753 & \textbf{.790} & \textbf{.785} & .616 & .644 & .536 & $-.48$ \\
& N-SAE
& .730 & .546 & .503 & .450 & .461 & .464 & $-.47$ \\
& LWDL-OT
& \textbf{.763} & .778 & .754 & \textbf{.735} & \textbf{.715}
& \textbf{.695} & \textbf{$-.16$} \\
\hline

Botswana
& NMF
& .620 & .221 & .162 & .159 & .146 & .139 & $-.75$ \\
& N-SAE
& \textbf{.633} & .312 & .246 & .204 & .182 & .168 & $-.79$ \\
& LWDL-OT
& .556 & \textbf{.338} & \textbf{.281} & \textbf{.244}
& \textbf{.232} & \textbf{.212} & \textbf{$-.59$} \\
\hline

Salinas A
& NMF
& \textbf{.853} & .514 & .417 & .362 & .320 & .293 & $-.98$ \\
& N-SAE
& .653 & .378 & .330 & .305 & .293 & .295 & $-.59$ \\
& LWDL-OT
& .763 & \textbf{.588} & \textbf{.536} & \textbf{.485}
& \textbf{.466} & \textbf{.449} & \textbf{$-.57$} \\
\hline

\end{tabular}
\end{table}

\FloatBarrier
\subsection{Point-Cloud Data with LWDL-EOT}\label{subsec:ELOTDL}

We investigate the performance of LWDL-EOT for modeling three-dimensional point clouds from the ModelNet-10 data set \citep{wu20153d}, which contains CAD meshes of common 3D objects. We model these as probability measures, uniformly supported on 1024 points in $[0,1]^3$ by sampling (see Algorithm \ref{alg:modelnet_mesh_to_measure} for details). We build a data set from all object classes with at least 300 examples in the data set. There are six such classes, corresponding to beds, chairs, monitors, sofas, tables and toilets, giving a data set with 1800 probability measures $\{\mu_i\}_{i=1}^{1800}$. We encode this as a set of maps by computing $T_{\rho\rightarrow \mu_i}^\epsilon-Id$ and run Algorithm \ref{alg:entropic-lwdl}; see Appendix \ref{sec:Additional Details Point Clouds} for details and experiments on the role of the base measure.

In Figure~\ref{fig:application-diagnostics}(\subref{fig:pointcloud-reconstructions}), we display two examples of point clouds and their reconstructions learned by Algorithm \ref{alg:entropic-lwdl}, as well as the top three atoms by activation for each data point. We also report the relative activation value for each of the atoms.  We see that the learned atoms form interpretable components of the data, for instance, the seat and backrest of a chair, the legs of a chair or the base of a toilet (atoms $T_0, T_{29}$ and $T_{27}$ respectively). Figure~\ref{fig:application-diagnostics}(\subref{fig:pointcloud-umap}) further illustrates the class structure of the learned codes. To quantitatively support our qualitative observations, we train a linear classifier on the learned coefficient representations and obtain 96.1\% test accuracy on a held-out subset, indicating that the learned coefficients are highly linearly separable by object class.

\begin{figure}[H]
\centering
    \captionsetup[subfigure]{justification=raggedright,singlelinecheck=false}
    \begin{subfigure}[t]{0.62\linewidth}
        \centering
        \includegraphics[width=\linewidth]{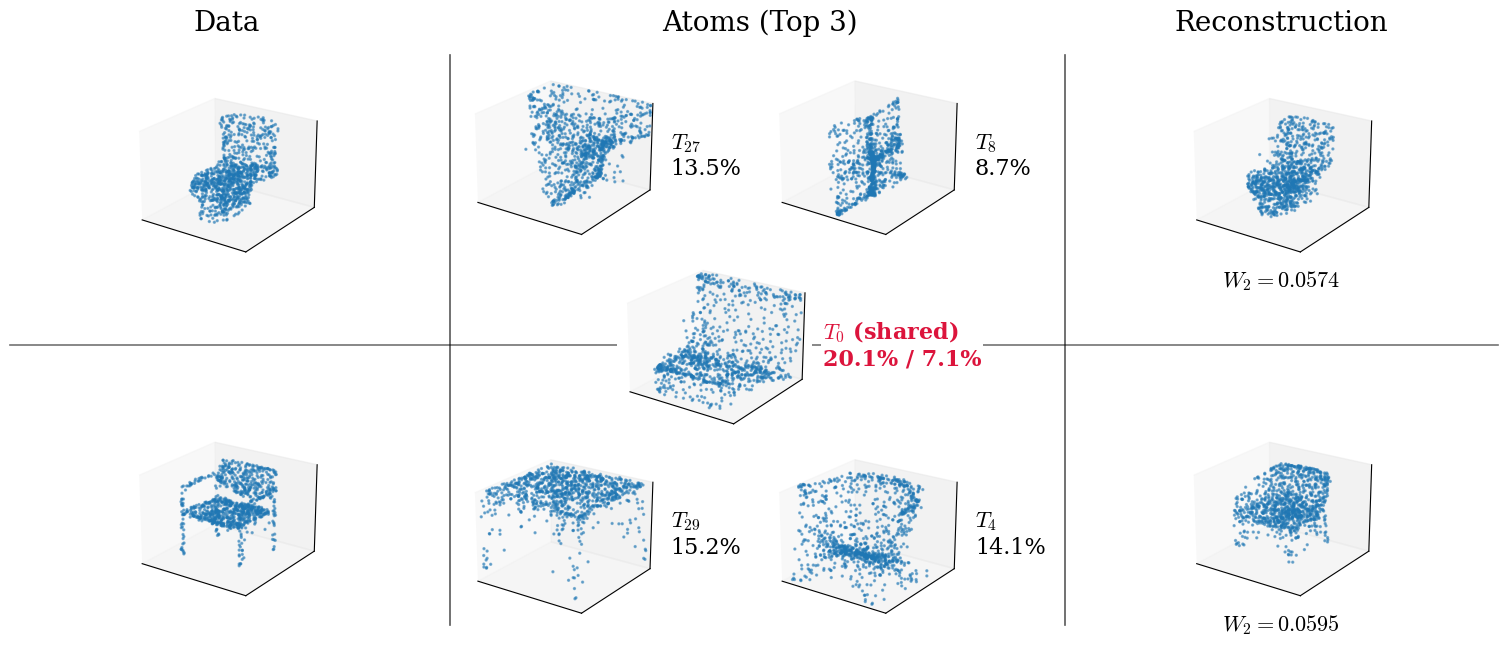}
        \caption{Representative point clouds, top atoms, and reconstructions.}
        \label{fig:pointcloud-reconstructions}
    \end{subfigure}
    \hfill
    \begin{subfigure}[t]{0.33\linewidth}
        \centering
        \includegraphics[width=\linewidth]{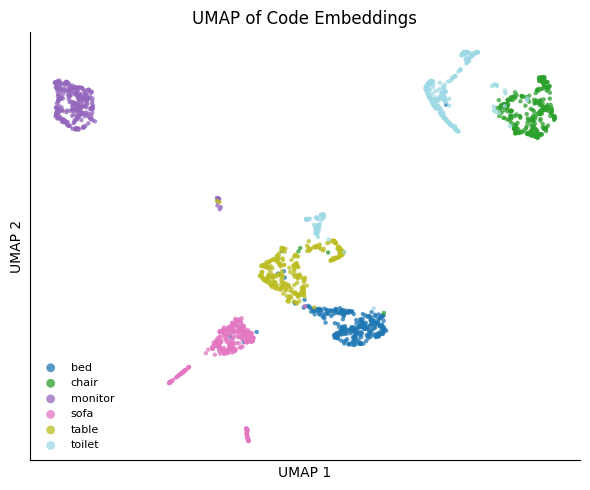}
        \caption{UMAP visualization of the point-cloud codes.}
        \label{fig:pointcloud-umap}
    \end{subfigure}

    \vspace{0.7em}

    \begin{subfigure}[t]{0.58\linewidth}
        \vspace{0pt}
        \centering
        \includegraphics[width=0.492\linewidth,trim={0 2pt 0 2pt},clip]{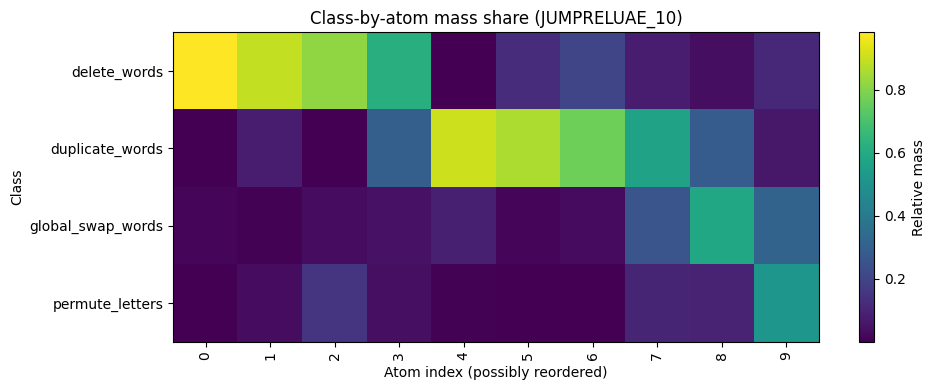}\hfill
        \includegraphics[width=0.492\linewidth,trim={0 2pt 0 2pt},clip]{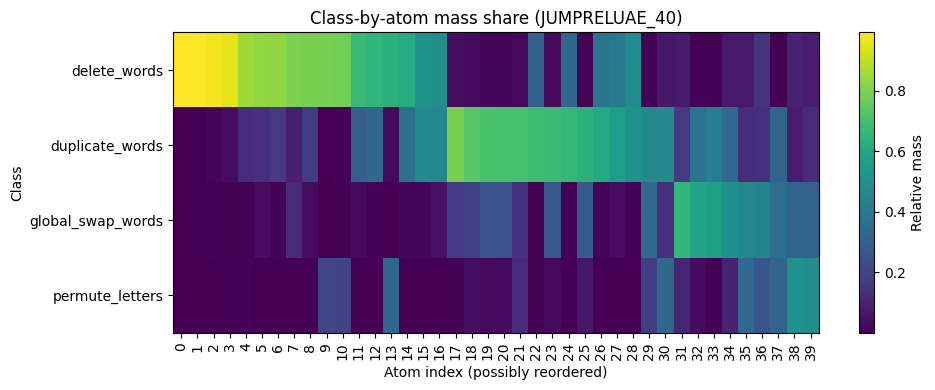}\\[-0.45em]
        \includegraphics[width=0.492\linewidth,trim={0 2pt 0 2pt},clip]{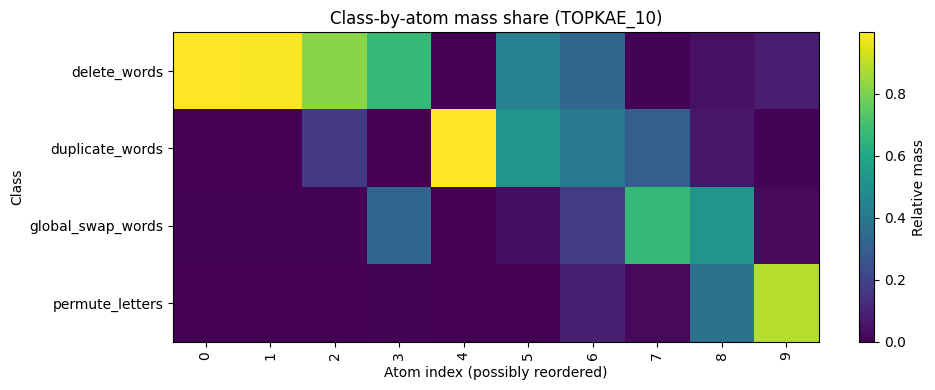}\hfill
        \includegraphics[width=0.492\linewidth,trim={0 2pt 0 2pt},clip]{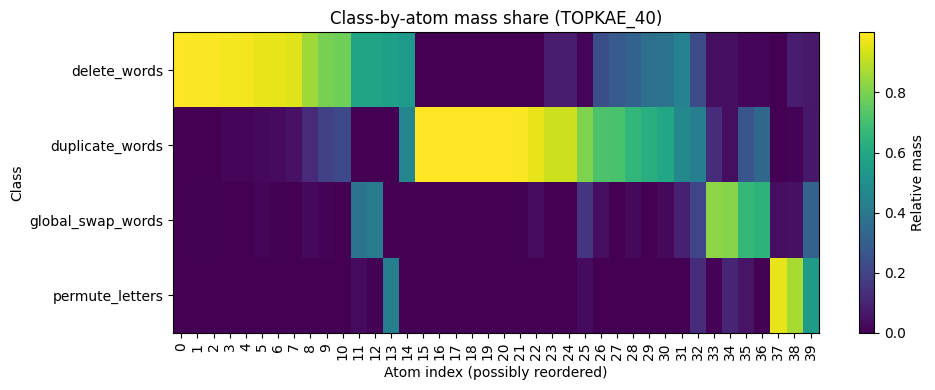}
        \caption{Mean atom usage by corruption label for JumpReLU and TopK with 10 and 40 atoms.}
        \label{fig:corruption-atom-usage}
        \label{fig:atom_usage_grid}
    \end{subfigure}
    \hfill
    \begin{subfigure}[t]{0.36\linewidth}
        \vspace{0pt}
        \centering
        \includegraphics[width=\linewidth,trim={0 0 0 34pt},clip]{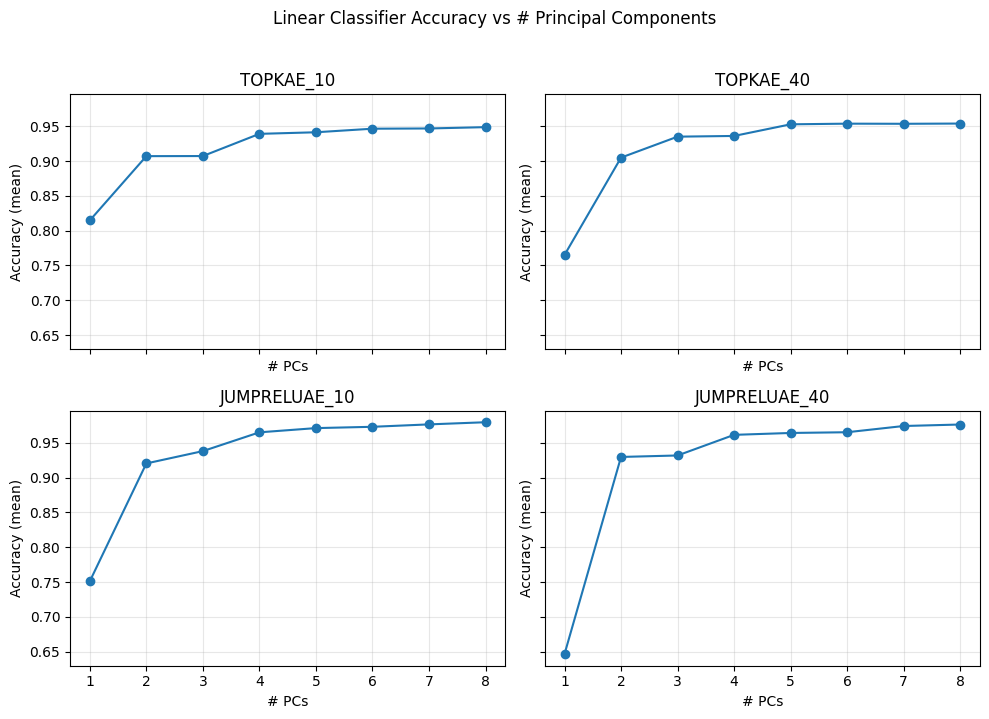}
        \caption{Linear-classifier accuracy versus the number of code principal components.}
        \label{fig:corruption-classifier}
        \label{fig:accuracy_vs_corruption_strength}
    \end{subfigure}

    \caption{Application diagnostics for LWDL-EOT and LWDL-KP. (a)--(b) Point-cloud reconstruction and code structure. (c)--(d) Transformer corruption-code structure and linear separability.}
    \label{fig:application-diagnostics}
    \label{fig:pointcloud-examples-and-codes}
    \label{fig:corruption-code-diagnostics}
\end{figure}

\subsection{LWDL-KP on Transformer Activation Distributions}\label{sec:LWDL-KP}

We investigate the performance of LWDL-KP for modeling distributions of transformer activations. We represent documents as empirical probability measures supported on their token-level hidden states. Given $\rho$, we embed each document by the centered Kantorovich potential $f_{\rho\rightarrow \mu_i}$, as described in Section \ref{sec:linearized models for wdl}, and apply Algorithm \ref{alg:KP_SAE_embedding} to the resulting embeddings. We extract our token-level hidden states from an intermediate layer in the model Pythia-410M \citep{biderman2023pythia}, which was developed specifically for interpretability studies.

\subsubsection{Synthetic Textual Corruptions.}
Our first experiment investigates whether the LWDL-KP codes can reliably separate textual corruptions. We generate 19,200 corrupted variants of a fixed base document with 1130 tokens, by applying word swaps, word deletions, within-word letter permutations and word duplications at twelve intensity levels. This generates a data set $\{\mu_i\}_{i=1}^{19200}\subset \mathcal{P}_2(\mathbb{R}^{1024})$, and a base measure $\rho$ given by the embedding of the base document into $\mathcal{P}_2(\mathbb{R}^{1024}).$ Panels~(\subref{fig:corruption-atom-usage}) and~(\subref{fig:corruption-classifier}) of Figure~\ref{fig:application-diagnostics} summarize the learned codes. Atom usage is strongly class-dependent, and the pattern of code usage appears to persist across different activation functions and numbers of atoms. Furthermore, we see that the codes linearly separate corruption type: across 10 random splits, logistic regression on the first two principal components of the codes exceeds $90\%$ accuracy, and using the first six components gives about $95\%$ accuracy. Full corruption definitions and training details are in Appendix~\ref{sec:additional details llm}.

\FloatBarrier
\subsubsection{Document-Level Atoms on The Pile.}
We next apply LWDL-KP to documents sampled from The Pile \citep{gao2020pile}. We sample $\approx 80,000$ documents between 100 and 2000 tokens and embed these as probability measures $\{\mu_i\}_{i=1}^{80000}\subset \mathcal{P}_2(\mathbb{R}^{1024})$ using Pythia-410M. After PCA reduction, we compute centered potentials (i.e., $\sum_{k=1}^{1000}f_{\rho\rightarrow \mu_i}(x_k)=0$), where $\rho$ is supported on 1000 points sampled from a 1024-dimensional Gaussian with covariance equal to the empirical covariance matrix of the data. We investigate whether the learned atoms $\{f_j\}_{j=1}^{512}$ activate on semantically coherent sets of documents \citep{cunningham2023sparse}. We examine the composition of the top 100 documents, ordered by activation magnitude, for each atom (Table \ref{tab:feature_purity}), and see that many atoms are strongly associated with documents from a single source. 
To test whether source identity alone explains this concentration, we examine contrastive pairs of atoms that are strongly associated with the same source. Their top-activating documents exhibit noticeably different semantic themes, providing evidence that the learned atoms specialize beyond source identity; see Appendix~\ref{sec:pile details} for examples.

\section{Generalization Error Bounds for LWDL}\label{sec:generalization error lwdl}

Our main statistical result gives quantitative bounds on the generalization error of LWDL. A sketch of the proof is given in Appendix~\ref{appendix:generalization proof sketch}, with complete details in Appendix~\ref{appendix:ProofsOfTheorems}.  

Given $\rho\in \mathcal{P}_2(\mathbb{R}^d)$ and $\mathcal{T}:=(T_1,T_2,...,T_m)\in \mathbb{T}^m$, where $\mathbb{T}\subseteq L^2(\rho)$ is a set of candidate atoms, our pointwise LWDL objective is defined
\begin{equation}\label{eqn:ptwise DL objective}
h_\rho^g(\mu,\epsilon,\mathcal{T})=\min_{\lambda\in \mathbb{R}^m_{\geq 0}}\frac{1}{2}\|T^\epsilon_{\rho\rightarrow \mu}-\sum_{j=1}^m\lambda_jT_j\|^2_{L^2(\rho)}+g(\lambda)
\end{equation}
where $g: \mathbb{R}_{\ge 0}^m\rightarrow \mathbb{R}_{\geq 0}$ is a nonnegative function, and $T^\epsilon_{\rho\rightarrow \mu}$ denotes the entropic map between $\rho$ and $\mu\in \mathcal{P}_2(\mathbb{R}^d)$ for $\epsilon>0$ and the optimal transport map if $\epsilon=0$, if it exists. We require moment bounds on our data $\mu$, formalized by the following definition: for any $\vartheta\geq 0$ and $p\geq 1$, let $\mathcal{M}_p(\vartheta)\subset \mathcal{P}_p(\mathbb{R}^d)$ denote the set of $\mu\in \mathcal{P}_p(\mathbb{R}^d)$ with $M_p(\mu)\leq \vartheta.$ We first proceed by establishing a generic dictionary learning generalization complexity result, namely Theorem \ref{thm:samp_complex_dl}.  While this exact statement does not appear in the literature, the main ideas in the proof are well-known (see Appendix Section \ref{appendix:generalization proof sketch} for and comparisons to the literature).

\begin{theorem}\label{thm:samp_complex_dl}
Let $\Theta\in \mathcal{P}(\mathcal{M}_2(\vartheta))$ for some $\vartheta\geq0$. Let $g:\mathbb{R}^m_{\geq 0}\rightarrow \mathbb{R}_{\geq 0}$ be lower semicontinuous, such that $g(0)=0$ and $g(x)\geq c\|x\|_1$, for some $c> 0$. Let $\mathbb{T}\subseteq L^2(\rho)$ be compact, such that $\|T\|_{L^2(\rho)}=1$ for all $T\in \mathbb{T}$. Let $\mu_1,...,\mu_n\sim\Theta$. Define, for any $\epsilon \ge 0$:
\begin{equation*}
    \mathcal{T}^{n}\in\argmin_{\mathcal{T}\in \mathbb{T}^m}\frac{1}{n}\sum_{i=1}^n h^g_\rho(\mu_i,\epsilon,\mathcal{T})
\end{equation*}
\begin{equation*}
\mathcal{T}^*\in\argmin_{\mathcal{T}\in \mathbb{T}^m}\mathbb{E}_{\mu\sim \Theta}[h^g_\rho(\mu,\epsilon,\mathcal{T})],
\end{equation*}
where $h_\rho^g$ is defined as in (\ref{eqn:ptwise DL objective}). Then for any $0<\delta<2\sqrt{\vartheta}$, we have:
\begin{align*}
&\mathbb{E}_{\mu,\mu_1,\ldots,\mu_n\sim\Theta}
\Bigl[h^g_\rho(\mu,\epsilon,\mathcal{T}^{n})
-h^g_\rho(\mu,\epsilon,\mathcal{T}^*)\Bigr]\\
&\leq 8\delta
+\frac{C C_{\vartheta,c}}{\sqrt{n}}
\int_{\delta/(4C_{\vartheta,c})}^{\vartheta/(2C_{\vartheta,c})}
\sqrt{m\log\!\left(\mathcal{N}\left(u;\mathbb{T},L^2(\rho)\right)\right)}\,du\\
&\quad+\frac{4C\vartheta}{\sqrt{n}}.
\end{align*}
where $C_{\vartheta,c}= \frac{\sqrt{\vartheta^4+\vartheta^3}}{\sqrt{2}\min\{c,c^2\}}$ and $C$ is a universal constant. If $\rho$ is a.c., then the same bound holds with $T^\epsilon_{\rho\rightarrow \mu}$ replaced by $T_{\rho\rightarrow \mu}$.
\end{theorem}

By instantiating these bounds for different choices of $\mathbb{T}$---highlighting the relationship between covering number bounds and the quantitative regularity theory of optimal transport---we obtain the following sample complexity result for LWDL:

\begin{theorem}\label{thm:summary theorem}
Under the set-up of Theorem \ref{thm:samp_complex_dl}:
\begin{enumerate}
\item Let $\Omega=\prod_{j=1}^d I_j\subset \mathbb{R}^d$, where $I_j\subset \mathbb{R}$ is a closed interval with $|I_j|=D$, and let $\rho\in \mathcal{P}(\Omega)$ be a.c. such that $\text{supp}(\rho)=\Omega$ and $0<\alpha<\frac{d\rho}{dx}(x)<\beta<\infty$ for Lebesgue a.e. $x\in \Omega$. Let $\mathbb{T}=\{T_{\rho\rightarrow \mu}\;|\; \mu\in\mathcal{P}(\Omega),\|T\|_{L^2(\rho)}=1\}$. Then:
\begin{align*}
&\mathbb{E}_{\mu,\mu_1,\mu_2,\ldots,\mu_n\sim\Theta}
\Bigl[
h^g_\rho(\mu,0,\mathcal{T}^{n})-h^g_\rho(\mu,0,\mathcal{T}^*)
\Bigr]\\
&\qquad\leq
\begin{cases}
\text{$\mathcal{O}(\sqrt{\frac{m}{n}})$}, & \text{if $d=1$},\\[0.8em]
\text{$\mathcal{O}(m^{1/d}n^{-2/(3d)})$}, & \text{if $d>1$}.
\end{cases}
\end{align*}
where the constant factors depend on $D,d$, $\beta,\vartheta, c$.
\item Let $\rho\in \mathcal{G}_\sigma(\mathbb{R}^d)$. Furthermore, let $\mathring{\mathbb{T}}=\{T^\epsilon_{\rho\rightarrow \mu}\;|\; \mu\in \mathcal{G}_\sigma(\mathbb{R}^d),\;\|T\|_{L^2(\rho)}=1\}$ for some $\epsilon>0$, and let $\mathbb{T}$ be the closure of $\mathring{\mathbb{T}}$ in $L^2(\rho)$. Then:
\[
\mathbb{E}_{\mu,\mu_{1:n}\stackrel{\rm iid}{\sim}\Theta}
\!\left[
h^g_\rho(\mu,\epsilon,\mathcal{T}^{n})
-h^g_\rho(\mu,\epsilon,\mathcal{T}^*)
\right]
\leq \mathcal{O}\!\left(\sqrt{\frac{m}{n}}\right).
\]
where the constant factors depend on $\sigma,d, \epsilon,\vartheta, c.$\end{enumerate}
\end{theorem}

\section{Conclusion}
\label{sec:Conclusion}

In this paper, we develop an approach to sparse dictionary learning in the Wasserstein space via linearized optimal transport.  Our empirical results demonstrate improved computational efficiency compared to existing Wasserstein dictionary learning methods, as well as the efficacy and interpretability of the learned atoms and reconstructions of our approaches.  Theoretically, we established sample complexity results that depend on the complexity of map representations of the data considered.  In future work, it is of interest to guarantee the \emph{identifiability} of atoms under an LWDL generative model and also understand the impact of the base measure, potentially introducing it into the learning process itself.

\section*{Acknowledgments}

BM, SA, and JMM acknowledge partial support from NSF DMS-2309519. BM and JMM acknowledge partial support from NSF DMS-2318894. SA acknowledges partial support from the NSF under Cooperative Agreement PHY-2019786 (The NSF AI Institute for Artificial Intelligence and Fundamental Interactions, \url{http://iaifi.org/}). Generative AI tools, including Claude and ChatGPT, were used to assist with code implementation and testing, synthetic data generation and processing, feedback on experiments and mathematical arguments, and manuscript and figure editing. All AI-assisted outputs were reviewed, verified, and revised by the authors as appropriate. The authors take full responsibility for the final content of the manuscript.

\bibliographystyle{plainnat}
\bibliography{mybib}

\clearpage
\appendix

\section{Additional Technical Background}
\label{appendix:AdditionalTechnicalBackground}

\subsection{Additional Results for the Wasserstein Space}

The Wasserstein-2 distance admits a dual form:
\begin{align}\label{eqn:kantorovich duality}
W_2(\mu,\nu)
=
\max_{\substack{
f(x)+g(y)\leq \frac{1}{2}\|x-y\|^2 \\
(f,g)\in L^1(\mu)\times L^1(\nu)
}}
\sqrt{\int f\,d\mu+\int g\,d\nu },
\end{align}
and any solution $(f_{\mu\rightarrow \nu},g_{\mu\rightarrow \nu})$ is referred to as a \textit{Kantorovich potential}.

Similarly, the entropy-regularized transport problem admits a dual form:
\begin{align}
r_{f,g}(x,y)
&:=\tfrac12\|x-y\|^2-f(x)-g(y),\notag\\
W_{2,\epsilon}^{2}(\mu,\nu)
&=\max_{\substack{f\in L^1(\mu)\\g\in L^1(\nu)}}\Biggl\{
\notag\\
&\qquad\int f\,d\mu+\int g\,d\nu\notag\\
&-\epsilon\int
\left[\exp\left(-r_{f,g}(x,y)/\epsilon\right)-1\right]
\,d\mu(x)d\nu(y)\Biggr\}.
\label{eqn:EOT dualform}
\end{align}
We refer to the maximizing functions in (\ref{eqn:EOT dualform}), $(f_{\mu\rightarrow \nu}^\epsilon,g_{\mu\rightarrow \nu}^\epsilon)$, as \emph{entropic (Kantorovich) potentials}.

We recall Brenier's theorem \citep{brenier1991polar}, which characterizes optimal couplings for the Wasserstein-2 distance:

\begin{theorem}\label{thm:brenier}(Brenier)
Let $\mu,\nu \in \mathcal{P}_2(\mathbb{R}^d)$ and suppose that $\mu$ is absolutely continuous w.r.t. the Lebesgue measure.  Then there exists a map $T_{\mu\rightarrow \nu}:\mathbb{R}^d\rightarrow \mathbb{R}^d$, unique up to $\mu$-a.e. equivalence, such that $(Id\times T_{\mu\rightarrow \nu})_{\#}\mu$ solves the optimal coupling problem defining $W_2(\mu,\nu)$. Furthermore, there exists a lower semicontinuous convex function $\phi_{\mu\rightarrow \nu}$ which is differentiable $\mu$-a.e., such that $T_{\mu\rightarrow \nu}(x)=\nabla \phi_{\mu\rightarrow \nu}(x)$ for $\mu$-a.e. $x\in \mathbb{R}^d$.
\end{theorem}

We refer to $\phi_{\mu\rightarrow \nu}$ as the Brenier potential between $\mu$ and $\nu$, and will typically denote the Brenier potential between $\nu$ and $\mu$ as $\psi_{\mu\rightarrow \nu}$.  With consistent normalizations, the Brenier and Kantorovich potentials satisfy
\begin{align}
\phi_{\mu\rightarrow \nu}&=\frac{\|\cdot\|^2}{2}-f_{\mu\rightarrow \nu},\label{eqn:BP KP relation 1}\\
\psi_{\mu\rightarrow \nu}&=\frac{\|\cdot\|^2}{2}-g_{\mu\rightarrow \nu}.\notag
\end{align}

\subsection{Covering Numbers}

\begin{definition}
    Let $(\mathcal{F}, \|\cdot\|)$ be a normed space.  For $\epsilon>0$, the \emph{$\epsilon$-ball centered at $f\in\mathcal{F}$} is the set $B_{\epsilon}(f):=\{g\in\mathcal{F} \ | \ \|f-g\|<\epsilon\}$. A collection of $\epsilon$-balls $\{B_{\epsilon}(f)\}_{f}$ is an \emph{$\epsilon$-net for a set $\mathcal{F}'\subseteq\mathcal{F}$} if $\mathcal{F}'\subseteq\bigcup_{f}B_{\epsilon}(f)$.  The \emph{covering number} of $\mathcal{F}'\subseteq\mathcal{F}$, denoted $N(\epsilon,\mathcal{F}',\|\cdot\|)$, is the minimal number of balls in any $\epsilon$-net for $\mathcal{F}'$.
\end{definition}

We will frequently use the following elementary result that establishes covering numbers scale under Lipschitz mappings.
\begin{lemma}\label{lemma:covering under lipschitz}
Let $(\mathcal{X},d_\mathcal{X}), (\mathcal{Y},d_\mathcal{Y})$ be metric spaces, and suppose there is a $K$-Lipschitz map $\phi:\mathcal{X}\rightarrow \mathcal{Y}$. Then for all $\tau\geq 0$, we have that $\mathcal{N}(K\tau;\phi(\mathcal{X}),d_\mathcal{Y})\leq \mathcal{N}(\tau;\mathcal{X},d_\mathcal{X})$.
\end{lemma}

\begin{proof}
Let $\{x_i\}_{i=1}^{N_\tau}$ be a $\tau$-net in $\mathcal{X}$, so for any $x\in \mathcal{X}$ there is an $x_i$ with $d_\mathcal{X}(x_i,x)\leq \tau$. Then $d_\mathcal{Y}(\phi(x_i),\phi(x))\leq K\tau$. Hence $\{\phi(x_i)\}_{i=1}^{N_\tau}$ is a $K\tau$ net in $\phi(\mathcal{X})\subseteq \mathcal{Y}$.
\end{proof}

\subsection{Subgaussian Probability Measures}

\begin{definition} \label{def:subgauss}
A $\mathbb{R}^d$-valued random variable $X\sim\mu$ is \emph{subgaussian} if any of the following equivalent properties hold \citep{vershynin2018high}:
\begin{enumerate}
\item There exists $\sigma\geq 0$ such that $\sup_{v\in S^{d-1}} \mathbb{E}_{X\sim \mu}\left(e^{\frac{|\langle X,v\rangle|^2}{\sigma^2}}\right)\leq 2$, where $S^{d-1}$ is the unit sphere in $\mathbb{R}^d.$
\item There exists $\sigma\geq 0$ and an absolute constant $c_G$ such that for all $t\geq 0$ we have: \begin{equation*}
\sup_{v\in S^{d-1}}\mathbb{P}_{X\sim \mu}(|\langle X,v\rangle|\geq t)\leq 2\exp\left(-\frac{c_Gt^2}{\sigma^2}\right).
\end{equation*}
\item  There exists $\sigma\geq 0$ and an absolute constant $C_G$ such that for all $p\geq 1$, we have: \begin{equation*}
\sup_{v\in S^{d-1}}\mathbb{E}_{X\sim \mu}(|\langle X,v\rangle|^p)^{1/p}\leq C_G\sigma \sqrt{p}.
\end{equation*}

\end{enumerate}
The smallest $\sigma$ such that condition 1. holds is referred to as the \emph{subgaussian norm} of $X$, denoted $\|X\|_{\mathcal{G}}.$ Up to the absolute constants $c_G$ and $C_G$, $\|X\|_{\mathcal{G}}$ is the smallest constant that makes any of the above inequalities valid.

We say that $\mu$ is \emph{subgaussian} if $X\sim \mu$ is subgaussian, and define the \emph{subgaussian norm} of $\mu$ to be $\|\mu\|_\mathcal{G}:=\|X\|_\mathcal{G}$. The set of probability measures on $\mathbb{R}^d$ with subgaussian norm bounded above by $\sigma$ is referred to as the set of \emph{$\sigma$-subgaussian measures}, and denoted $\mathcal{G}_\sigma(\mathbb{R}^d).$
\end{definition}

\begin{lemma}\label{lemma:normsubgauss}
(Lemma 1 from \citet{jin2019short}) Let $\mu\in \mathcal{G}_\sigma(\mathbb{R}^d)$ and let $X\sim \mu$. Then there exists an absolute constant $q_G$ such that $\|X\|$ is $q_G\sqrt{d}\sigma$-subgaussian, i.e., $\mathbb{E}\left(\exp\left( \frac{\|X\|^2}{2 q_G^2d \sigma^2}\right)\right) \leq 2$.

\end{lemma}

\subsection{Additional Background on LOT}\label{appendix:Background on LOT}

One may interpret $\rho$-LOT as linearizing $W_{2}$ at $\rho$, in the following sense.  The measure $\mu$ is embedded in $L^2(\rho)$---formally the tangent space to the Wasserstein-2 space at $\rho$ \citep{ambrosio2008gradient}---as the optimal displacement map $(T_{\rho\rightarrow \mu}-Id)\in L^2(\rho)$.  This may be formally viewed as the logarithmic map at $\rho$ applied to the measure $\mu$.  Then $\rho$-LOT is simply the $L^2(\rho)$
distance between these logarithmic map embeddings of $\mu$ and $\nu$.  The choice of tangent space is determined by $\rho$, and practical applications of LOT are impacted by $\rho$.

\subsection{Wasserstein Barycenters}

Given reference probability measures $\{\mu_{i}\}_{i=1}^{m}$ and coefficients $\lambda\in \mathbb{R}_{\ge 0}^{m}$, one can define the \emph{Wasserstein-2} barycenter as $\argmin_{\mu}\sum_{i=1}^{m}\lambda_{i}W_{2}^{2}(\mu_{i},\mu).$  We note that the barycenter associated with $\lambda$ and $\lambda/\|\lambda\|_{1}$ are the same, so without loss of generality one often considers $\lambda\in\Delta^{m-1}$.  Such a minimizer always exists and uniqueness follows under mild conditions (e.g., at least one $\mu_{i}$ is absolutely continuous) \citep{agueh2011barycenters}.  Wasserstein-2 barycenters provide a natural geometry-preserving method of synthesizing measures in the Wasserstein-2 space.

\subsection{Synthesizing in the LWDL Framework}

Synthesis of new measures is straightforward in LWDL-OT and LWDL-EOT.  Denoting $\mathcal{T}^*:=(T_1^*,T_2^*,...,T_m^*)\in \mathbb{T}^m$ to be an optimal set of atoms, we may trivially recover probability measure atoms via the pushforwards $\mathcal{V}:=((T_j^*)_{\#}(\rho))_{j=1}^m\in \mathcal{P}_2(\mathbb{R}^d)^m$,
and may generate new, out-of-sample data via $\{[\sum_{j=1}^m\lambda_j T_j]_{\#}(\rho)|\lambda \in \mathbb{R}^m_{\geq 0}\}$.  In contrast, as the learned atoms in LWDL-KP are scalar functions as opposed to transport maps, one cannot easily convert atoms and dictionary codes into probability measures in the ambient space; first one needs to differentiate the potentials, which introduces possible numerical issues.  However, parametrizing high-dimensional measures as scalar functions (as opposed to high-dimensional vector-valued functions) allows LWDL-KP to provide an efficient encoding scheme for high-dimensional distributional data when synthesis is not required.

\section{LWDL as a Linearization of WDL}
\label{appendix:LWDL as WDL Linearization}
Consider the sparse WDL problem
\begin{equation}\label{eqn:WDL original}
\argmin_{\mathcal{V}\in \mathcal{P}_2(\mathbb{R}^d)^m}
\mathbb{E}_{\mu\sim \Theta}\Bigl[
\min_{\lambda \in \Delta^{m-1}}W_2^{2}(\mu,\mu_{\lambda,\mathcal{V}})+g(\lambda)\}\Bigr],
\end{equation}
which was first introduced in \citet{schmitz2018wasserstein} albeit without coefficient regularization (i.e., $g\equiv 0$).  Noting that optimizing over $\Delta^{m-1}$ and $\mathbb{R}_{\ge 0}^{m}$ are equivalent and replacing $W_{2}$ with the $\rho$-linearized OT distance gives:
\begin{equation}\label{eqn:WDL original, modified 1}
\begin{aligned}
&\argmin_{\mathcal{V}\in \mathcal{P}_2(\mathbb{R}^d)^m}
\mathbb{E}_{\mu\sim \Theta}\Biggl[
\min_{\lambda \in \mathbb{R}_{\ge 0}^{m}}\|T_{\rho\rightarrow\mu}
-T_{\rho\rightarrow \mu_{\lambda,\mathcal{V}}}\|_{L^2(\rho)}^{2}
+g(\lambda)\Biggr\}\Biggr].
\end{aligned}
\end{equation}
Next, we replace the Wasserstein-2 barycenter, $\mu_{\mathcal{V},\lambda}$, with , $\left(\sum_{i=1}^{m}\lambda_{i}T_{\rho\rightarrow\mu_{i}}\right)_{\#}\rho$, the barycenter with respect to the $\rho$-linearized OT distance \citep{merigot2020quantitative, werenski2025linearized}.  We may then deduce by Brenier's theorem that since $\sum_{i=1}^{m}\lambda_{i}T_{\rho\rightarrow\mu_{i}}$ is the gradient of a convex function that pushes $\rho$ onto $\left(\sum_{i=1}^{m}\lambda_{i}T_{\rho\rightarrow\mu_{i}}\right)_{\#}\rho$, it follows that $T_{\rho\rightarrow\left(\sum_{i=1}^{m}\lambda_{i}T_{\rho\rightarrow\mu_{i}}\right)_{\#}\rho}=\sum_{i=1}^{m}\lambda_{i}T_{\rho\rightarrow\mu_{i}}$. From this, (\ref{eqn:WDL original, modified 1}) reduces further to

\begin{equation*}
\begin{aligned}
&\argmin_{\mathcal{V}\in \mathcal{P}_2(\mathbb{R}^d)^m}
\mathbb{E}_{\mu\sim \Theta}\Biggl[
\min_{\lambda \in \mathbb{R}_{\ge 0}^{m}}\Biggl\{\|T_{\rho\rightarrow\mu}
-\sum_{i=1}^{m}\lambda_{i}T_{\rho\rightarrow \nu_{i}}\|_{L^2(\rho)}^{2}
+g(\lambda)\Biggr\}\Biggr].
\end{aligned}
\end{equation*}namely (\ref{eqn:LDL objective}) for LWDL-OT.  An analogous argument in which the Wasserstein-2 barycenter is replaced with the entropic Wasserstein-2 barycenter \citep{janati2020debiased, mallery2025synthesis} yields LWDL-EOT.

Similarly, under various assumptions \citep{santambrogio_2015, mallery2025synthesis}, we may realize LWDL-KP by linearizing $W_{2}^{2}$ at $\rho$ with respect to the flat $L^2(\rho)$ geometry via linear perturbations of the form $\tilde{\rho}-\rho\in L^\infty_c(\Omega)$, where $\Omega$ is the support of $\rho$. The resulting notion of derivative $\frac{\delta}{\delta\rho}$ is the \textit{$L^2$-first variation} \citep{santambrogio2017euclidean}. After fixing a normalization constant, we may identify $\frac{\delta}{\delta \rho}W_{2}^{2}(\rho,\mu)$ as the \textit{Kantorovich potential} $f_{\rho\rightarrow \mu}\in L^2(\rho)$. We may then replace $\mu_{\lambda,\mathcal{V}}$ with $\sum_{j=1}^m\lambda_j f_{\rho\rightarrow \nu_j}$, and take $\mathbb{T}:=\{f_{\rho\rightarrow \nu}\;|\; \nu\in \mathcal{P}_2(\mathbb{R}^d)\}$ to obtain the instance of (\ref{eqn:LDL objective}) for LWDL-KP.

\section{Algorithm Details}

\label{appendix:AlgorithmDetails}

We now give the generic encoder and method-specific atom parameterizations underlying Algorithm \ref{alg:generic-lwdl}.

\subsection{LWDL-OT}\label{appendix:LWDL-OTDetails} Our implementation of LWDL-OT is specific to one-dimensional probability measures supported on a compact interval $K$. That is, we assume that $\rho\in \mathcal P_2(K)$, $\operatorname{supp}(\Theta)\subset \mathcal P_2(K)$.  Without loss of generality, we take $K=[0,1]$.  This assumption applies, for instance, to normalized hyperspectral pixels, which are considered in Section~\ref{sec:lwdl one dim}.

By Theorem~\ref{thm:brenier}, optimal transport maps between one-dimensional probability measures are monotone nondecreasing. We therefore parameterize each atom $S_j^{\theta_a}$ as a monotone map on the support of \(\rho\). Denote the decoder matrix $W_{\mathrm{dec}}^{\theta_a}\in \mathbb R^{m\times (M-1)}$, and for each \(1\leq j\leq m\) and \(1\leq k\leq M-1\), define positive increments
\[
\Delta_{j,k}^{\theta_a}
:=
\log\!\left(1+\exp\left([W_{\mathrm{dec}}^{\theta_a}]_{j,k}\right)\right).
\]
We then define the value of the \(j\)-th atom at the support points $x_1< x_{2}<\dots<x_M$ by $S_j^{\theta_a}(x_1):=0$, and, for \(k=2,\dots,M\),
\[
S_j^{\theta_a}(x_k)
:=
\frac{
\sum_{s=1}^{k-1}\Delta_{j,s}^{\theta_a}
}{
\sum_{s=1}^{M-1}\Delta_{j,s}^{\theta_a}
}.
\]
This construction ensures that $S_j^{\theta_a}$ is monotone increasing on the grid, with $S_j^{\theta_a}(x_1)=0$ and $S_j^{\theta_a}(x_M)=1$. The full algorithm is described in Algorithm~\ref{alg:OTSAE}.

\begin{algorithm}[htbp!]
\caption{LWDL-OT in $\mathbb{R}$}
\label{alg:OTSAE}

\KwIn{
Samples $\{\mu_i\}_{i=1}^N \subset \mathcal P_2([0,1])$;
finitely supported reference measure $\rho\in\mathcal P_2([0,1])$ with
$\operatorname{supp}(\rho)=\{x_1<\cdots<x_M\}$;
sparsity parameter $c>0$;
number of atoms $m$;
number of encoder layers $L$;
batch size $B$;
activation maps $\{\chi_{s,\theta_c}\}_{s=0}^{L-1}$;
parameters $\theta=(\theta_a,\theta_c)$;
decoder matrix
$W_{\mathrm{dec}}(\theta_a)\in\mathbb R^{m\times(M-1)}$.
}

\vspace{0.5em}

\For{$j=1,\dots,m$}{
    \For{$k=1,\dots,M-1$}{
        Compute the positive increment $\Delta_{j,k}^{\theta_a}\gets
        \log\!\left(1+\exp\!\left([W_{\mathrm{dec}}^{\theta_a}]_{j,k}\right)\right).$
    }

    Set $Z_j(\theta_a)
        \gets
        \sum_{s=1}^{M-1}\Delta_{j,s}^{\theta_a},
        \qquad
        S_j^{\theta_a}(x_1)\gets 0.$

    \For{$k=2,\dots,M$}{
        Define the monotone atom $S_j^{\theta_a}(x_k)
            \gets
            \frac{1}{Z_j^{\theta_a}}
            \sum_{s=1}^{k-1}\Delta_{j,s}^{\theta_a}.$

    }
}

\For{$i=1,\dots,N$}{
    Compute the optimal transport map $T_{\rho\to\mu_i}$ on $\operatorname{supp}(\rho)$.

    Set $\mathcal E_\rho(\mu_i)
        \gets
        T_{\rho\to\mu_i}.$
}

\For{$i=1,\dots,N$}{
    Compute the sparse nonnegative code using Algorithm~\ref{alg:generic-lista-encoder}:
    \[
        \lambda(\mu_i;\theta_a,\theta_c)
        \gets
        \operatorname{LISTAEncoder}
        \bigl(
        \mathcal E_\rho(\mu_i),
        \mathcal S^{\theta_a};
        \theta_c
        \bigr).
    \]

}

Optimize the following objective over $\theta=(\theta_a,\theta_c)$ by batched SGD, with batch size $B$:
\[
\sum_{i=1}^N
\left[
\frac12
\left\|
T_{\rho\to\mu_i}
-
\sum_{j=1}^m
\lambda_j(\mu_i;\theta_a,\theta_c)\,
S_j^{\theta_a}\right\|_{L^2(\rho)}^2
+
c\|\lambda(\mu_i;\theta_a,\theta_c)\|_1
\right].
\]

\vspace{0.5em}

\KwOut{
Learned parameters $\theta$;
monotone transport atoms $\{S_j^{\theta_a}\}_{j=1}^m$;
sparse nonnegative codes
$\{\lambda(\mu_i;\theta_a,\theta_c)\}_{i=1}^N$.
}

\end{algorithm}

We restrict this exact OT-map parameterization to one dimension because Brenier maps are precisely monotone nondecreasing maps on an interval, so the constraint can be enforced by positive increments. In higher dimensions, enforcing that a learned vector field is the gradient of a convex potential is substantially more involved; this motivates the EOT and KP parameterizations below.

\subsection{LWDL-EOT}
\label{appendix:LWDL-EOTDetails}

We recall that LWDL-EOT first parameterizes entropic maps and then uses the corresponding displacement fields as dictionary atoms $\mathcal{S}^{\theta_a}$. Unlike the LWDL-OT implementation, which is restricted to one-dimensional measures, the LWDL-EOT implementation applies to measures in \(\mathcal P_2(\mathbb R^d)\) for any \(d\geq 1\). We assume that the base measure $\rho$ is supported on a finite set $\operatorname{supp}(\rho)\subset \mathbb R^d$, and we fix a set of points $\mathcal{Y}:=\{y_k\}_{k=1}^r\subset \mathbb R^d$ such that the convex hull of $\mathcal{Y}$ contains $\text{supp}(\mu)$ for almost every $\mu\sim\Theta$. We fix an entropic regularization parameter $\epsilon>0$ and write the decoder matrix $W_{\mathrm{dec}}^{\theta_a}
\in \mathbb R^{m\times r}.$  For each $1\leq j\leq m$, we define
\[
T_j^{\epsilon,\theta_a}(x)
:=
\sum_{k=1}^r
\zeta_{j,k}^{\epsilon,\theta_a}(x)y_k,
\]
where
\[
\begin{aligned}
a_{j,k}(x)
&:=-\frac{\|x-y_k\|^2+[W_{\mathrm{dec}}^{\theta_a}]_{j,k}}{2\epsilon},\\
\zeta_{j,k}^{\epsilon,\theta_a}(x)
&=\frac{\exp(a_{j,k}(x))}
{\sum_{\ell=1}^{r}\exp(a_{j,\ell}(x))}.
\end{aligned}
\]
In our experiments, we observed improved performance by parameterizing dictionary atoms with \textit{displacement fields} as opposed to the raw maps $T_j^{\epsilon,\theta_a}$. For this fixed $\epsilon$, define the displacement field
\[
V_j^{\epsilon,\theta_a}
:=
T_j^{\epsilon,\theta_a}-\mathrm{Id} \in L^2(\rho),
\]
and set the learned dictionary atom to be this displacement field,
\[
S_j^{\theta_a}:=V_j^{\epsilon,\theta_a},
\]
suppressing the fixed dependence on $\epsilon$ in the atom notation. Accordingly, the input embedding for a sample $\mu_i$ is the displacement field
\[
V_{\rho\to\mu_i}
:=
T_{\rho\to\mu_i}-\mathrm{Id}.
\]
The model then reconstructs $V_{\rho\to\mu_i}$ by a sparse nonnegative combination of entropic displacement atoms,
\[
V_{\rho\to\mu_i}
\approx
\sum_{j=1}^m
\lambda_j(\mu_i;\theta)S_j^{\theta_a}.
\]
This parameterization is analogous in spirit to residual parameterizations in neural networks: the identity map provides a natural baseline transport map, while the learned atom $S_j^{\theta_a}$ represents the residual displacement away from this baseline. The full algorithm is described in Algorithm \ref{alg:entropic-lwdl}.

\begin{algorithm}[htbp!]
\caption{LWDL-EOT}
\label{alg:entropic-lwdl}

\KwIn{
Base measure $\rho$;
samples $\{\mu_i\}_{i=1}^N\sim\Theta$;
target support $\mathcal Y=\{y_k\}_{k=1}^r\subset\mathbb R^d$;
entropic regularization parameter $\epsilon>0$;
sparsity parameter $c>0$;
number of atoms $m$;
number of encoder layers $L$;
batch size $B$;
activation maps $\{\chi_{s,\theta_c}\}_{s=0}^{L-1}$;
parameters $\theta=(\theta_a,\theta_c)$;
decoder matrix
$W_{\mathrm{dec}}(\theta_a)\in\mathbb R^{m\times r}$.
}

\vspace{0.5em}

\For{$i=1,\dots,N$}{
    Compute the transport map, or entropic barycentric projection,
    $T_{\rho\to\mu_i}$ on $\operatorname{supp}(\rho)$.

    Define the displacement embedding $\mathcal E_\rho(\mu_i)\gets
        V_{\rho\to\mu_i}:=T_{\rho\to\mu_i}-\mathrm{Id}.$
}

\For{$j=1,\dots,m$}{
    \For{$x\in\operatorname{supp}(\rho)$}{
        Compute normalized entropic weights
        \[
            \zeta_{j,k}^{\epsilon,\theta_a}(x)
            \gets
            \frac{
            \exp\!\left(
            -\frac{\|x-y_k\|^2+[W_{\mathrm{dec}}^{\theta_a}]_{j,k}}{2\epsilon}
            \right)
            }{
            \sum_{\ell=1}^r
            \exp\!\left(
            -\frac{\|x-y_\ell\|^2+[W_{\mathrm{dec}}^{\theta_a}]_{j,\ell}}{2\epsilon}
            \right)
            },
            \qquad k=1,\dots,r .
        \]

        Define the entropic map $T_j^{\epsilon,\theta_a}(x)\gets
            \sum_{k=1}^r
            \zeta_{j,k}^{\epsilon,\theta_a}(x)y_k.$

        Define the entropic displacement atom $S_j^{\theta_a}(x)\gets T_j^{\epsilon,\theta_a}(x)-x .$
    }
}

\For{$i=1,\dots,N$}{
    Compute the sparse nonnegative code using Algorithm~\ref{alg:generic-lista-encoder}:
    \[
        \lambda(\mu_i;\theta_a,\theta_c)
        \gets
        \operatorname{LISTAEncoder}
        \bigl(
        \mathcal E_\rho(\mu_i),
        \mathcal S^{\theta_a};
        \theta_c
        \bigr)
    \]

}

Optimize the following objective over $\theta=(\theta_a,\theta_c)$ by batched SGD, with batch size $B$:
\[
\frac1N
\sum_{i=1}^N
\left[
\frac12
\left\|
V_{\rho\to\mu_i}
-
\sum_{j=1}^m
\lambda_j(\mu_i;\theta_a,\theta_c)S_j^{\theta_a}
\right\|_{L^2(\rho)}^2
+
c\|\lambda(\mu_i;\theta_a,\theta_c)\|_1
\right].
\]

\vspace{0.5em}

\KwOut{
Learned parameters $\theta$;
entropic displacement atoms $\{S_j^{\theta_a}\}_{j=1}^m$;
sparse nonnegative codes
$\{\lambda(\mu_i;\theta_a,\theta_c)\}_{i=1}^N$.
}

\end{algorithm}

\subsection{LWDL-KP}
\label{appendix:LWDL-KPDetails}

Our implementation of LWDL-KP embeds each input distribution $\mu_i\sim\Theta$ as a Kantorovich potential $f_{\rho\rightarrow \mu_i}\in L^2(\rho)$. Since Kantorovich potentials are defined only up to an additive constant, we center each potential by subtracting its $\rho$-mean:
\[
f_{\rho\to\mu_i}
\leftarrow
f_{\rho\to\mu_i}
-
\mathbb E_{X\sim\rho}[f_{\rho\to\mu_i}(X)].\]
We again assume that the base measure $\rho$ is supported on a finite set $\{x_q\}_{q=1}^M\subset\mathbb R^d$, and we write the decoder matrix as $
W_{\mathrm{dec}}^{\theta_a}\in\mathbb R^{m\times M}.$ For each $1\leq j\leq m$, the $j$-th atom is defined on the support of $\rho$ by
\[
S_j^{\theta_a}(x_q)
:=
[W_{\mathrm{dec}}^{\theta_a}]_{j,q},
\qquad q=1,\dots,M.
\]
As with the input potentials, we center each atom by subtracting its $\rho$-mean:
\[
S_j^{\theta_a}
\leftarrow
S_j^{\theta_a}
-
\mathbb E_{X\sim\rho}[S_j^{\theta_a}(X)].
\]
The encoder then produces sparse nonnegative coefficients $\lambda(\mu_i;\theta_a,\theta_c)\in\mathbb R_{\geq 0}^m$ as described in Appendix \ref{appendix:LISTAEncoder}. The model reconstructs the centered Kantorovich potential by a sparse nonnegative combination of learned potential atoms:
\[
f_{\rho\to\mu_i}
\approx
\sum_{j=1}^m
\lambda_j(\mu_i;\theta)S_j^{\theta_a}.
\]
The full algorithm is described in Algorithm \ref{alg:KP_SAE_embedding}.

\begin{algorithm}[htbp!]
\caption{LWDL-KP}
\label{alg:KP_SAE_embedding}

\KwIn{
Samples $\{\mu_i\}_{i=1}^N\subset\mathcal P_2(\mathbb R^d)$;
discrete base measure
$\rho=\sum_{q=1}^M w_q\delta_{x_q}$ with $x_q\in\mathbb R^d$;
sparsity parameter $c\geq0$;
number of atoms $m$;
number of encoder layers $L$;
batch size $B$;
activation maps $\{\chi_{s,\theta_c}\}_{s=0}^{L-1}$;
parameters $\theta=(\theta_a,\theta_c)$;
decoder matrix
$W_{\mathrm{dec}}(\theta_a)\in\mathbb R^{m\times M}$.
}

\vspace{0.5em}

\For{$i=1,\dots,N$}{
    Compute a Kantorovich potential
    $f_{\rho\to\mu_i}\in L^2(\rho)$.

    Center the potential: $f_{\rho\to\mu_i} \gets f_{\rho\to\mu_i}-\mathbb E_{X\sim\rho}[f_{\rho\to\mu_i}(X)].$

    Set $\mathcal E_\rho(\mu_i)\gets f_{\rho\to\mu_i}.$
}

\For{$j=1,\dots,m$}{
    \For{$q=1,\dots,M$}{
        Define $S_j^{\theta_a}(x_q)\gets [W_{\mathrm{dec}}^{\theta_a}]_{j,q}.$
        }

    Center the atom: $S_j^{\theta_a}\gets S_j^{\theta_a}-\mathbb E_{X\sim\rho}[S_j^{\theta_a}(X)] .$

}

\For{$i=1,\dots,N$}{
    Compute the sparse nonnegative code using Algorithm~\ref{alg:generic-lista-encoder}:
    \[
        \lambda(\mu_i;\theta_a,\theta_c)
        \gets
        \operatorname{LISTAEncoder}
        \bigl(
        \mathcal E_\rho(\mu_i),
        \mathcal S^{\theta_a};
        \theta_c
        \bigr).
    \]

}

Optimize the following objective over $\theta=(\theta_a,\theta_c)$ by batched SGD, with batch size $B$:
\[
\frac1N
\sum_{i=1}^N
\left[
\frac12
\left\|
f_{\rho\to\mu_i}
-
\sum_{j=1}^m
\lambda_j(\mu_i;\theta_a,\theta_c)S_j^{\theta_a}
\right\|_{L^2(\rho)}^2
+
c\|\lambda(\mu_i;\theta_a,\theta_c)\|_1
\right].
\]

\vspace{0.5em}

\KwOut{
Learned parameters $\theta$;
centered Kantorovich potential atoms $\{S_j^{\theta_a}\}_{j=1}^m$;
sparse nonnegative codes
$\{\lambda(\mu_i;\theta_a,\theta_c)\}_{i=1}^N$.
}

\end{algorithm}

\subsection{LISTA-Type Encoder}
\label{appendix:LISTAEncoder}
We describe the shared encoder while leaving the atom parameterization abstract.  Write the full parameter as $\theta=(\theta_a,\theta_c)$, where $\theta_a$ parameterizes the learned atoms $\mathcal S^{\theta_a}=(S_1^{\theta_a},\dots,S_m^{\theta_a})$ and $\theta_c$ denotes the encoder parameters.  Given $\mu$, define the correlation vector
\[
h(\mu,\theta_a):=\bigl(\langle S_1^{\theta_a},\mathcal E_\rho(\mu)\rangle_{L^2(\rho)},\dots,\langle S_m^{\theta_a},\mathcal E_\rho(\mu)\rangle_{L^2(\rho)}\bigr)\in\mathbb R^m.
\]
Starting from $\lambda^{(0)}=0$, the encoder produces coefficients through the $L$-layer recursion
\[
\lambda^{(r+1)}(\mu;\theta_a,\theta_c)=\chi_{r,\theta_c}\!\left(A_{r,\theta_c}h(\mu,\theta_a)+B_{r,\theta_c}\lambda^{(r)}(\mu;\theta_a,\theta_c)+b_{r,\theta_c}\right),\quad r=0,\dots,L-1,
\]
with final output $\lambda(\mu;\theta_a,\theta_c)=\lambda^{(L)}(\mu;\theta_a,\theta_c)$.  Here $A_{r,\theta_c},B_{r,\theta_c}\in\mathbb R^{m\times m}$ and $b_{r,\theta_c}\in\mathbb R^m$ are learned affine parameters, while $\chi_{r,\theta_c}$ is an activation map into $\mathbb R_{\geq0}^m$ (e.g., ReLU, Jump ReLU, or a $k$-sparse activation).  This architecture is LISTA-type \citep{gregor2010learning}: it retains the iterative affine-thresholding form, but its learned weights and permitted activations need not correspond to a fixed ISTA update.

The decoder reconstructs the embedding as $\widehat{\mathcal E}_\rho(\mu;\theta_a,\theta_c)=\sum_{j=1}^m\lambda_j(\mu;\theta_a,\theta_c)S_j^{\theta_a}$.  We train $(\theta_a,\theta_c)$ end-to-end using the objective (\ref{eqn:discrete LWDL objective}), with AdamW \citep{loshchilov2017decoupled} and a cosine learning-rate schedule.  Algorithm \ref{alg:generic-lista-encoder} records the encoder explicitly.

\begin{algorithm}[htbp!]
\caption{LISTA-type encoder for LWDL}
\label{alg:generic-lista-encoder}

\KwIn{
Embedded samples $\mathcal E_\rho(\mu_i)\in L^2(\rho)$, $i=1,\dots,N$;
dictionary atoms
$\mathcal S^{\theta_a}=\bigl(S_1^{\theta_a},\dots,S_m^{\theta_a}\bigr)$,
$S_j^{\theta_a}\in L^2(\rho)$;
number of layers $L$;
encoder parameters
$\theta_c=\{A_{r,\theta_c},B_{r,\theta_c},b_{r,\theta_c}\}_{r=0}^{L-1}$,
where $A_{r,\theta_c},B_{r,\theta_c}\in\mathbb R^{m\times m}$ and
$b_{r,\theta_c}\in\mathbb R^m$;
activation maps
$\chi_{r,\theta_c}:\mathbb R^m\to\mathbb R_{\geq 0}^m$.
}

\vspace{0.5em}

\For{$i=1,\dots,N$}{
    Compute the correlation vector
    \[
        h(\mu_i,\theta_a)
        \gets
        \Bigl(
        \bigl\langle S_1^{\theta_a},\mathcal E_\rho(\mu_i)\bigr\rangle_{L^2(\rho)},
        \dots,
        \bigl\langle S_m^{\theta_a},\mathcal E_\rho(\mu_i)\bigr\rangle_{L^2(\rho)}
        \Bigr)
        \in \mathbb R^m .
    \]

    Initialize
    \[
        \lambda^{(0)}(\mu_i;\theta_a,\theta_c)
        \gets 0 \in \mathbb R^m .
    \]

    \For{$r=0,\dots,L-1$}{
        Update
        \[
            \lambda^{(r+1)}(\mu_i;\theta_a,\theta_c)
            \gets
            \chi_{r,\theta_c}
            \left(
                A_{r,\theta_c}h(\mu_i,\theta_a)
                +
                B_{r,\theta_c}\lambda^{(r)}(\mu_i;\theta_a,\theta_c)
                +
                b_{r,\theta_c}
            \right).
        \]
    }

    Set
    \[
        \lambda(\mu_i;\theta_a,\theta_c)
        \gets
        \lambda^{(L)}(\mu_i;\theta_a,\theta_c).
    \]
}

\vspace{0.5em}

\KwOut{
Sparse nonnegative codes
$\{\lambda(\mu_i;\theta_a,\theta_c)\}_{i=1}^N
\subset \mathbb R_{\geq 0}^m$.
}

\end{algorithm}

\FloatBarrier

\section{Related Work Table}
\label{appendix:RelatedWorkTable}

\begin{table}[H]
\centering
\small
\setlength{\tabcolsep}{5pt}
\renewcommand{\arraystretch}{1.15}
\begin{tabularx}{\textwidth}{|X||c|c|c|c|c|}
\hline
Method
& \makecell{Learns\\measures}
& \makecell{Learns\\codes}
& \makecell{Free\\support}
& Linearized
& \makecell{Sparsity\\promoting} \\
\hline\hline

\citep{schmitz2018wasserstein, mueller2023geometrically}
& \cmark & \cmark & \xmark & \xmark & \xmark \\
\hline

\citep{bonneel2016wasserstein, werenski2022measure, mallery2025synthesis}
& \xmark & \cmark & \xmark & \xmark & \xmark \\
\hline

\citep{gunsilius2024tangential, werenski2025linearized}
& \xmark & \cmark & \xmark & \cmark & \xmark \\
\hline

\citep{do2025sparse}
& \xmark & \cmark & \xmark & \xmark & \cmark \\
\hline

\textbf{This paper}
& \cmark & \cmark & \cmark & \cmark & \cmark \\
\hline
\end{tabularx}
\caption{Summary of related methods and comparison with the proposed LWDL framework.}
\label{tab:RelatedWork}
\end{table}

\section{Additional Details: HSI}\label{sec:additional details hsi}

The HSI experiments were run on a MacBook Pro with an Apple M2 Pro processor and 16 GB of unified memory, using the PyTorch MPS backend.

\begin{figure}[htbp!]
\centering
\begin{tikzpicture}[
    scale=0.7,
    transform shape,
    >=Latex,
    font=\small,
    panel/.style={inner sep=0pt, anchor=south west},
    note/.style={draw, rounded corners=2pt, align=center, inner sep=3.5pt, fill=white}
]

\node[panel] (front) at (0,0)
    {\includegraphics[width=4.9cm]{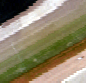}};

\foreach \dx/\dy/\img in {
    0.20/0.13/salinas_band11.png,
    0.40/0.26/salinas_band71.png,
    0.60/0.39/salinas_band151.png,
    1.35/0.88/salinas_band151.png
}{
    \node[panel, opacity=0.35]
        at ($(front.south west)+(\dx,\dy)$)
     {\includegraphics[width=4.9cm]{figures/hsi/IntroImage/\img}};
}

\node[panel] (front2) at (front.south west)
     {\includegraphics[width=4.9cm]{figures/hsi/IntroImage/salinas_rgb.png}};

\coordinate (bandOneCorner) at (front2.south east);
\coordinate (bandBCorner) at ($(front2.south east)+(1.35,0.88)$);
\node[font=\large, rotate=33, text=black!70] at ($(front2.north east)+(0.94,0.61)$) {$\cdots$};

\node[note] (callout) at ($(front2.north west)+(0.85,1.25)$)
    {Single-band\\image slices};
\draw[-{Latex[length=2.2mm]}, line width=0.7pt]
    (callout.east)
    .. controls ($(callout.east)+(0.45,0.05)$)
       and ($(front2.north west)+(2.45,0.65)$) ..
    ($(front2.north west)+(2.72,0.65)$);

\coordinate (pix) at ([xshift=3.70cm,yshift=3.30cm]front2.south west);
\draw[red, line width=1.2pt] (pix) circle (0.15cm);
\fill[red] (pix) circle (0.035cm);
\node[fill=white, inner sep=1pt, text=red!80!black, font=\scriptsize, below=2pt] at (pix) {$(65,25)$};

\node[anchor=north west, inner sep=0pt] (spec) at ($(front2.north east)+(1.45,0.60)$)
    {\includegraphics[width=4.65cm]{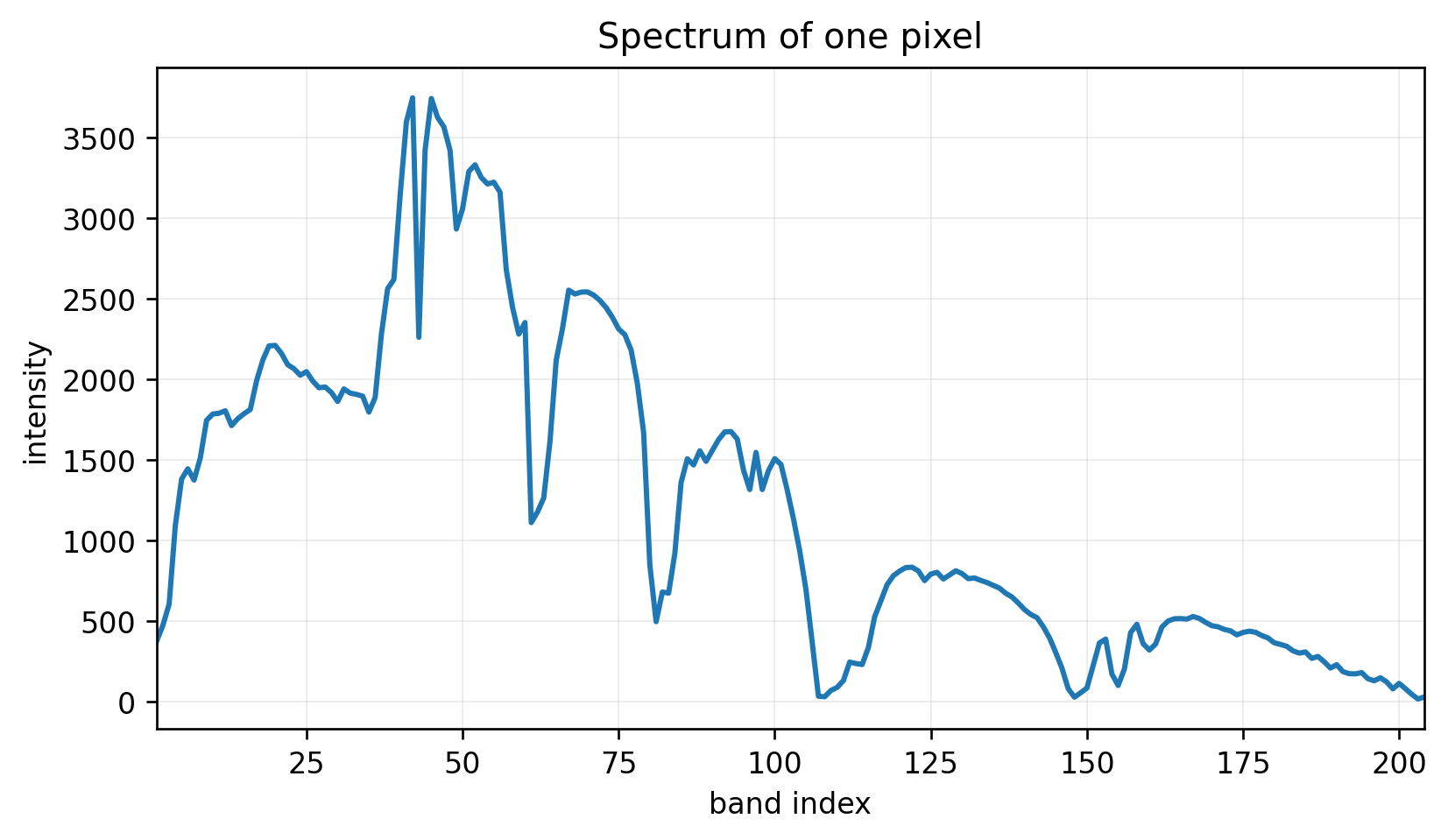}};
\node[font=\scriptsize, above=2pt of spec] {Spectrum of Pixel (65,25)};
\draw[red, line width=0.9pt, -{Latex[length=2.5mm]}]
    (pix) -- ($(spec.west)+(0.00,-0.05)$);

\node[anchor=north west, inner sep=0pt] (b11) at ($(spec.south west)+(0.00,-0.75)$)
    {\includegraphics[width=1.55cm]{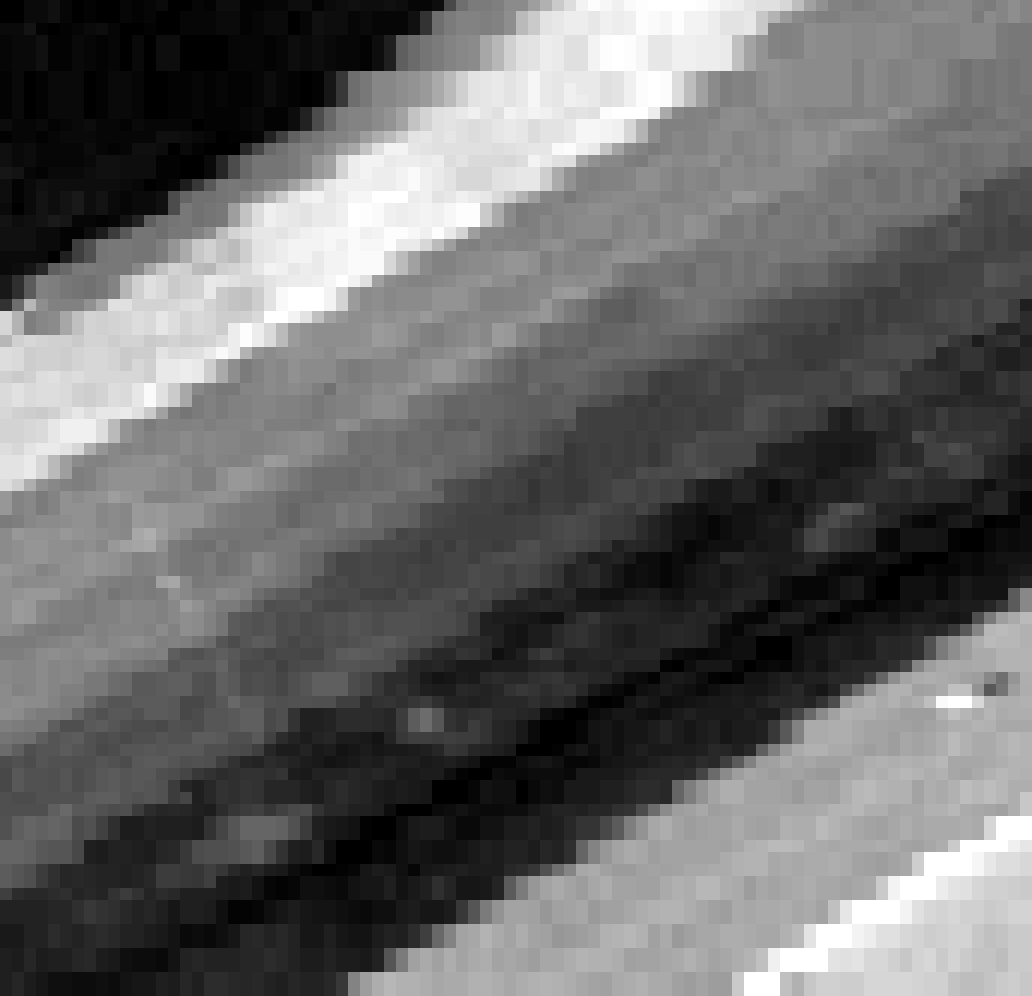}};
\node[anchor=west, inner sep=0pt] (b71) at ($(b11.east)+(0.22,0)$)
    {\includegraphics[width=1.55cm]{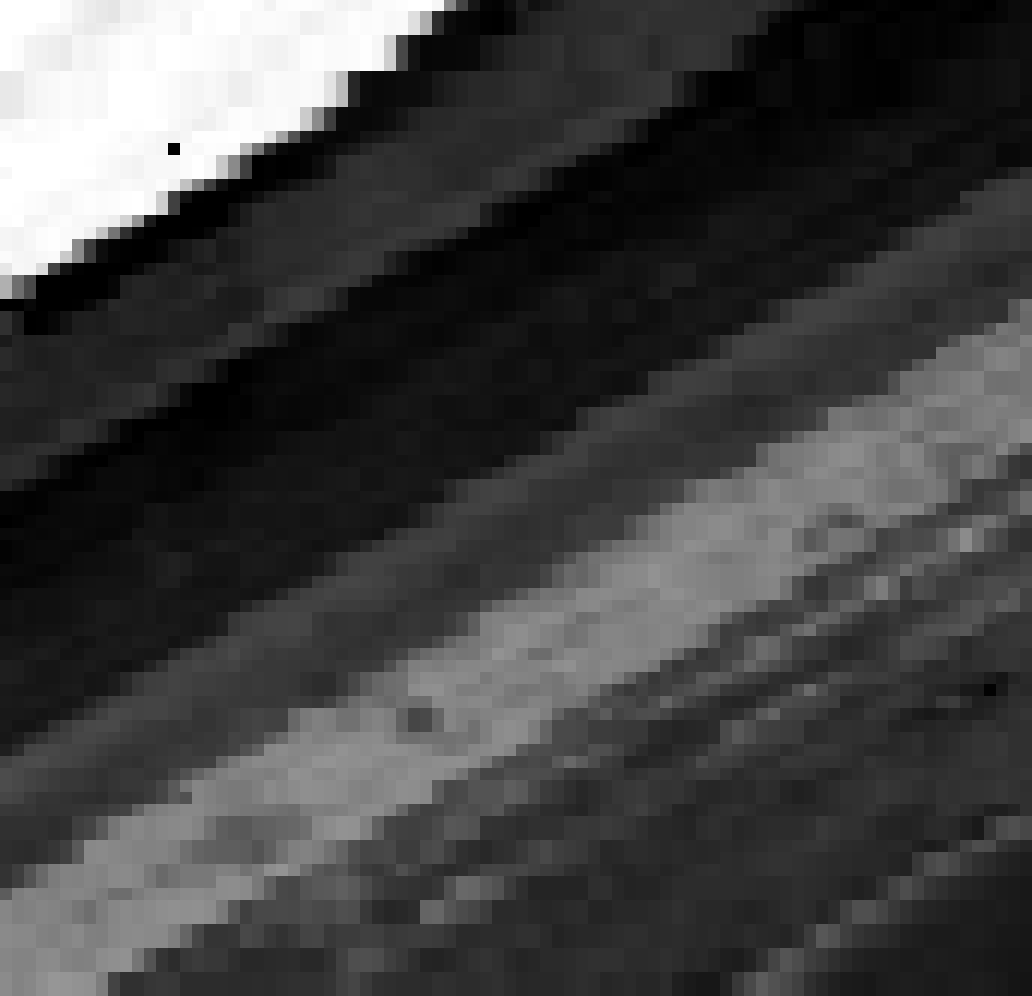}};
\node[anchor=west, inner sep=0pt] (b151) at ($(b71.east)+(0.22,0)$)
   {\includegraphics[width=1.55cm]{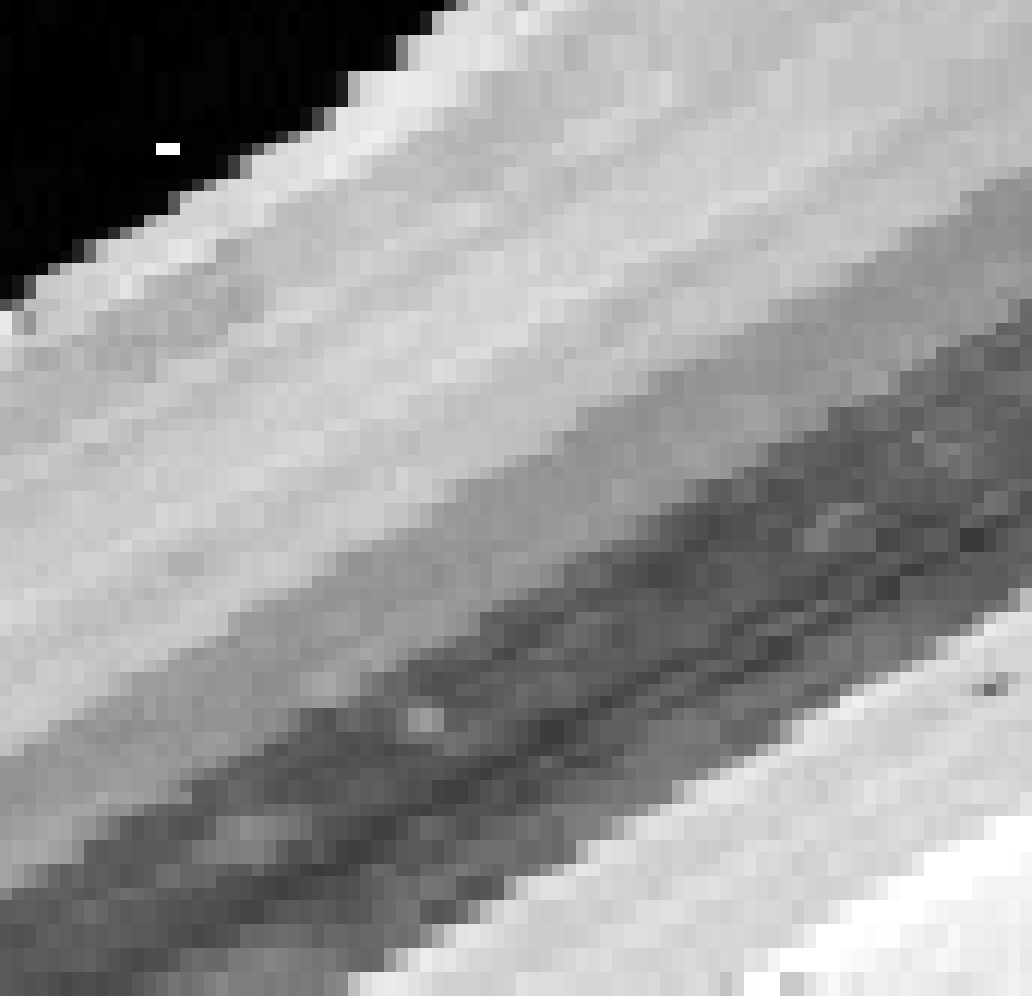}};

\node[font=\scriptsize] at ($(b11.north)!0.5!(b151.north)+(0,0.56)$)
    {Example Spectral Bands};
\node[font=\scriptsize, above=0pt of b11] {Band 11};
\node[font=\scriptsize, above=0pt of b71] {Band 71};
\node[font=\scriptsize, above=0pt of b151] {Band 151};

\foreach \nodeName in {b11,b71,b151}{
    \coordinate (tmp) at ([xshift=1.17cm,yshift=1.05cm]\nodeName.south west);
    \draw[red, line width=0.75pt] (tmp) circle (0.055cm);
}

\end{tikzpicture}
\caption{Hyperspectral image structure. A hyperspectral data cube records a spectrum at each spatial pixel: the spatial coordinates $(x,y)$ index image locations, the spectral axis indexes wavelength bands, and a selected pixel yields both a one-dimensional spectrum and corresponding intensities across individual bands.}
\label{fig:HSI_Intro}
\end{figure}

\begin{algorithm}[htbp]
\caption{Encode hyperspectral pixel with an OT map}\label{alg:pixel_to_prob_ot}
\textbf{Input:} Pixel spectrum $\mathcal{S}_{u,v,*}\in\mathbb{R}_{\ge 0}^b$, base measure $\rho\in\mathcal{P}_2([0,1])$.\\

$\{x_i\}_{i=1}^b \gets \left\{\frac{j-1}{b-1}\right\}_{j=1}^b$ \tcp*{Uniform support on $[0,1]$}

$Z \gets \sum_{i=1}^b \mathcal{S}_{u,v,i}$ \tcp*{Total intensity}

$w_i \gets \mathcal{S}_{u,v,i}/Z \quad \text{for all } i=1,\dots,b$ \tcp*{Normalize intensities}

$\mu_{u,v} \gets \sum_{i=1}^b w_i\,\delta_{x_i}$ \tcp*{Discrete probability measure}

\Return $T_{\rho\rightarrow \mu_{u,v}}$ \tcp*{Optimal transport map from $\rho$ to $\mu_{u,v}$}
\end{algorithm}

\begin{algorithm}[htbp!]
\caption{Kernel Density Estimate from Transport Map}\label{alg:transport_sample_kde}
\textbf{Input:} Samples $\{Q(p_i)\}_{i=1}^m$ of a map $Q:[0,1]\to[0,1]$ on the uniform grid
$p_i=\frac{i-1}{m-1}$, number of Monte Carlo samples $M$, bandwidth $h>0$, and output grid $\{x_\ell\}_{\ell=1}^n$.\\

Draw $u_j \sim \mathrm{Unif}[0,1]$ for $j=1,\dots,M$. \tcp*{Source samples}

Set $y_j \gets Q(u_j)$ for $j=1,\dots,M$. \tcp*{Push forward through $Q$}

Form the Gaussian kernel density estimate
\[
\widetilde p_h(x)
=
\frac{1}{Mh}\sum_{j=1}^M
\frac{1}{\sqrt{2\pi}}
\exp\!\left(-\frac{1}{2}\Bigl(\frac{x-y_j}{h}\Bigr)^2\right).
\]

Evaluate $\widetilde p_h$ on the grid $\{x_\ell\}_{\ell=1}^n$.

Normalize:
\[
\widehat p(x_\ell)
\gets
\frac{\widetilde p_h(x_\ell)}{\sum_{r=1}^n \widetilde p_h(x_r)},
\qquad \ell=1,\dots,n.
\]

\Return $\{\widehat p(x_\ell)\}_{\ell=1}^n$. \tcp*{Estimated density on the output grid}
\end{algorithm}

\begin{figure}[t!]
  \centering
  \includegraphics[width=0.95\linewidth,trim={6pt 5pt 4pt 5pt},clip]{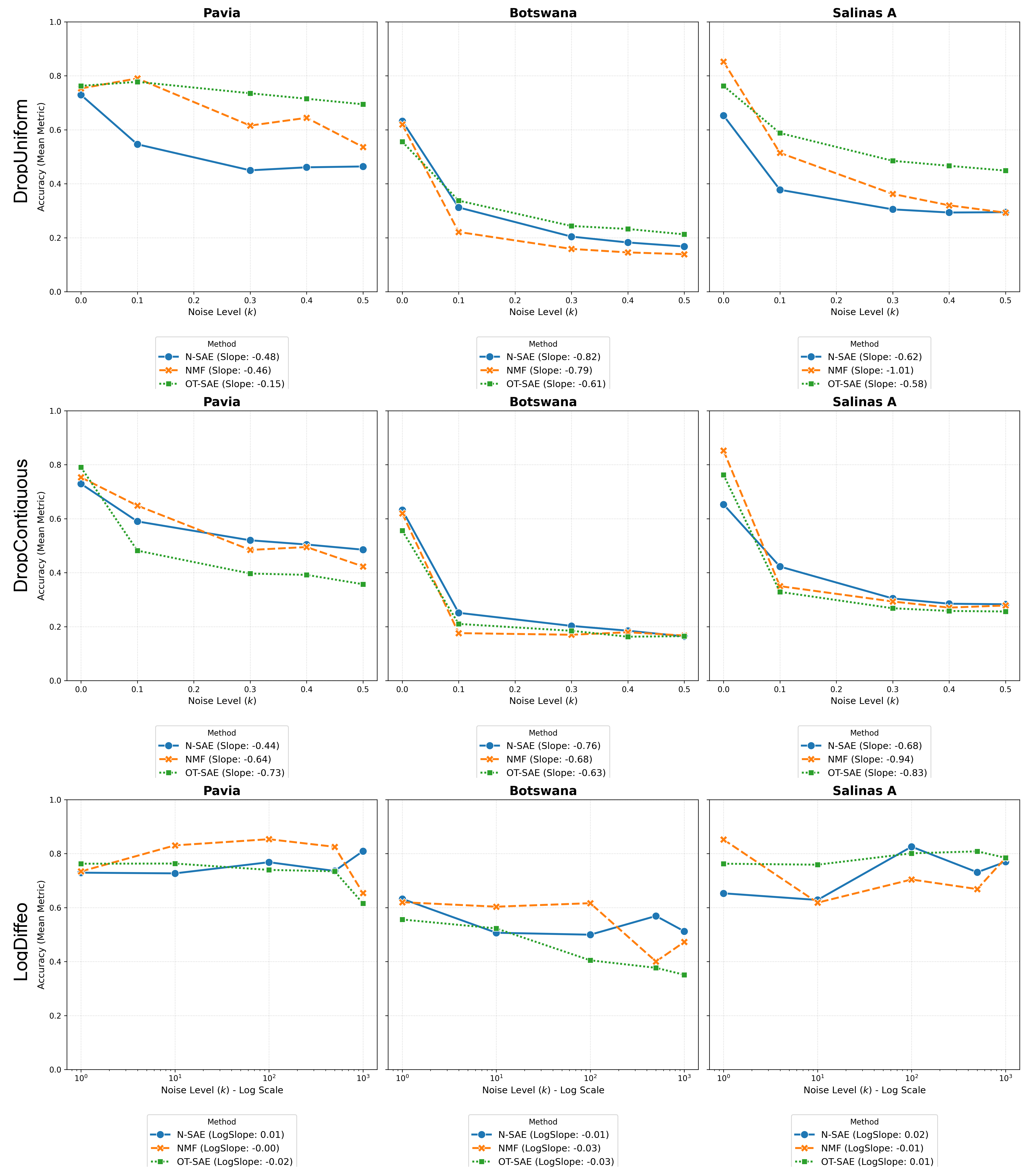}
 \caption{Clustering accuracy of LWDL-OT, NMF, and the nonnegative sparse autoencoder (N-SAE) on the Pavia Centre, Botswana, and Salinas A hyperspectral data sets under increasing corruption strength. Results are shown for uniform band dropout, contiguous band dropout, and logarithmic spectral warping. LWDL-OT is particularly robust to uniform band dropout, while performance under the other corruptions varies across data sets and methods.}
  \label{fig:robustness-combined}
\end{figure}

\begin{table}[htbp!]\label{table:fullHSI}
    \centering
    \caption{Performance comparison for DropUniform across data sets and $k$ values.}\label{table:DropUniform}
    \scriptsize
    \begin{subtable}{\linewidth}
    \begin{tabularx}{\linewidth}{l l l X X X X X X}
        \toprule
        \textbf{Data Set} & \textbf{Type} & \textbf{Key} & \textbf{k=0.0} & \textbf{k=0.1} & \textbf{k=0.2} & \textbf{k=0.3} & \textbf{k=0.4} & \textbf{k=0.5} \\
        \midrule
        Pavia & OT-SAE & $m=10$, $c=10^{-2}$ & 0.7627 & 0.7778 & 0.7537 & \textbf{0.7353} & \textbf{0.7151} & \textbf{0.6945} \\
        Pavia & OT-SAE & $m=10$, $c=5*10^{-1}$ & \textbf{0.7906} & 0.7161 & 0.6513 & 0.6034 & 0.5806 & 0.5337 \\
        Pavia & N-SAE & $m=10$, $c=10^{-3}$ & 0.6585 & 0.5341 & 0.5052 & 0.4792 & 0.4539 & 0.4453 \\
        Pavia & N-SAE & $m=10$, $c=10^{-4}$ & 0.7295 & 0.5462 & 0.5028 & 0.4498 & 0.4609 & 0.4641 \\
        Pavia & NMF & NMF (rank10) & 0.7533 & \textbf{0.7904} & \textbf{0.7850} & 0.6156 & 0.6443 & 0.5363 \\
        \midrule
        Botswana & OT-SAE & $m=15$, $c=10^{-1}$ & 0.5559 & \textbf{0.3375} & 0.2810 & 0.2435 & 0.2324 & 0.2124 \\
        Botswana & OT-SAE & $m=15$, $c=5*10^{-1}$ & 0.4887 & 0.3216 & \textbf{0.2977} & \textbf{0.2674} & \textbf{0.2499} & \textbf{0.2292} \\
        Botswana & N-SAE & $m=15$, $c=10^{-3}$ & 0.3149 & 0.1798 & 0.1546 & 0.1453 & 0.1391 & 0.1329 \\
        Botswana & N-SAE & $m=15$, $c=10^{-5}$ & \textbf{0.6325} & 0.3121 & 0.2457 & 0.2041 & 0.1824 & 0.1676 \\
        Botswana & NMF & NMF (rank15) & 0.6198 & 0.2213 & 0.1615 & 0.1586 & 0.1456 & 0.1388 \\
        \midrule
        Salinas A & OT-SAE & $m=7$, $c=10^{-3}$ & 0.7628 & \textbf{0.5879} & \textbf{0.5360} & \textbf{0.4850} & \textbf{0.4664} & \textbf{0.4487} \\
        Salinas A & OT-SAE & $m=7$, $c=10^{-5}$ & 0.7794 & 0.5580 & 0.5034 & 0.4752 & 0.4636 & 0.4432 \\
        Salinas A & N-SAE & $m=15$, $c=10^{-3}$ & 0.2853 & 0.2852 & 0.2852 & 0.2852 & 0.2852 & 0.2852 \\
        Salinas A & N-SAE & $m=15$, $c=10^{-5}$ & 0.6530 & 0.3776 & 0.3298 & 0.3051 & 0.2934 & 0.2946 \\
        Salinas A & NMF & NMF (rank7) & \textbf{0.8525} & 0.5144 & 0.4172 & 0.3622 & 0.3201 & 0.2928 \\
        \bottomrule
    \end{tabularx}
\end{subtable}

\vspace{1mm}

\begin{subtable}{\linewidth}
    \centering
    \caption{Performance comparison for DropContiguous across data sets and $k$ values.}
    \label{table:DropContiguous}
    \scriptsize
    \begin{tabularx}{\linewidth}{l l l X X X X X X}
        \toprule
        \textbf{Data Set} & \textbf{Type} & \textbf{Key} & \textbf{k=None} & \textbf{k=0.1} & \textbf{k=0.2} & \textbf{k=0.3} & \textbf{k=0.4} & \textbf{k=0.5} \\
        \midrule
        Pavia & OT-SAE & $m=10$, $c=10^{-2}$ & 0.7627 & 0.4841 & 0.3815 & 0.3731 & 0.3494 & 0.3603 \\
        Pavia & OT-SAE & $m=10$, $c=5\cdot 10^{-1}$ & \textbf{0.7906} & 0.4818 & 0.4114 & 0.3968 & 0.3918 & 0.3572 \\
        Pavia & N-SAE & $m=10$, $c=10^{-3}$ & 0.6585 & 0.4574 & 0.4495 & 0.4453 & 0.4453 & 0.4453 \\
        Pavia & N-SAE & $m=10$, $c=10^{-4}$ & 0.7295 & 0.5900 & 0.4980 & \textbf{0.5200} & \textbf{0.5044} & \textbf{0.4852} \\
        Pavia & NMF & NMF (rank10) & 0.7341 & \textbf{0.6447} & \textbf{0.5466} & 0.5006 & 0.4949 & 0.4230 \\
        \midrule
        Botswana & OT-SAE & $m=15$, $c=10^{-1}$ & 0.5559 & 0.2101 & 0.1970 & 0.1846 & 0.1627 & 0.1647 \\
        Botswana & OT-SAE & $m=15$, $c=5\cdot 10^{-1}$ & 0.4887 & 0.2079 & 0.1818 & 0.1848 & 0.1655 & 0.1502 \\
        Botswana & N-SAE & $m=15$, $c=10^{-3}$ & 0.3149 & 0.1850 & 0.1663 & 0.1708 & 0.1646 & \textbf{0.1690} \\
        Botswana & N-SAE & $m=15$, $c=10^{-5}$ & \textbf{0.6325} & \textbf{0.2511} & \textbf{0.2134} & \textbf{0.2027} & \textbf{0.1849} & 0.1638 \\
        Botswana & NMF & NMF (rank15) & 0.6201 & 0.1882 & 0.1970 & 0.1976 & 0.1791 & 0.1671 \\
        \midrule
        Salinas A & OT-SAE & $m=7$, $c=10^{-3}$ & 0.7628 & 0.3288 & 0.2906 & 0.2681 & 0.2581 & 0.2561 \\
        Salinas A & OT-SAE & $m=7$, $c=10^{-5}$ & 0.7794 & 0.3300 & 0.2772 & 0.2573 & 0.2418 & 0.2461 \\
        Salinas A & N-SAE & $m=7$, $c=10^{-3}$ & 0.2853 & 0.3064 & 0.3103 & \textbf{0.3170} & \textbf{0.3127} & \textbf{0.3018} \\
        Salinas A & N-SAE & $m=7$, $c=10^{-5}$ & 0.6530 & 0.4226 & \textbf{0.3401} & 0.3045 & 0.2849 & 0.2829 \\
        Salinas A & NMF & NMF (rank7) & \textbf{0.8523} & \textbf{0.4544} & 0.3189 & 0.2923 & 0.2705 & 0.2784 \\
        \bottomrule
    \end{tabularx}
\end{subtable}
\vspace{1mm}
\begin{subtable}{\linewidth}
    \centering
    \caption{Performance comparison for LogDiffeo across data sets and $k$ values.}\label{table:LogDiffeo}
    \scriptsize
   \begin{tabularx}{\linewidth}{l l l X X X X X X}
    \toprule
    \textbf{Data Set} & \textbf{Type} & \textbf{Key} & \textbf{k=None} & \textbf{k=1} & \textbf{k=10} & \textbf{k=100.0} & \textbf{k=500} & \textbf{k=1000} \\
    \midrule
    Pavia & OT-SAE & $m=10$, $c=10^{-2}$ & 0.7627 & 0.7633 & 0.7397 & 0.7346 & 0.6157 & 0.6855 \\
    Pavia & OT-SAE & $m=10$, $c=5*10^{-1}$ & \textbf{0.7906} & 0.7958 & 0.7557 & 0.6517 & 0.6118 & 0.6240 \\
    Pavia & N-SAE & $m=10$, $c=10^{-3}$ & 0.6585 & 0.5823 & 0.5665 & 0.5785 & 0.5416 & 0.5395 \\
    Pavia & N-SAE & $m=10$, $c=10^{-4}$ & 0.7295 & 0.7269 & 0.7683 & 0.7359 & \textbf{0.8093} & \textbf{0.8199} \\
    Pavia & NMF & NMF (rank10) & 0.7341 & \textbf{0.8307} & \textbf{0.8535} & \textbf{0.8253} & 0.6539 & 0.7295 \\
    \midrule
    Botswana & OT-SAE & $m=15$, $c=10^{-1}$ & 0.5559 & 0.5229 & 0.4047 & 0.3768 & 0.3510 & 0.3600 \\
    Botswana & OT-SAE & $m=15$, $c=5*10^{-1}$ & 0.4887 & 0.4453 & 0.3558 & 0.3174 & 0.3161 & 0.3178 \\
    Botswana & N-SAE & $m=15$, $c=10^{-3}$ & 0.3149 & 0.2956 & 0.4666 & 0.5683 & 0.5110 & 0.5193 \\
    Botswana & N-SAE & $m=15$, $c=10^{-5}$ & \textbf{0.6325} & 0.5068 & 0.4996 & \textbf{0.5690} & 0.5122 & \textbf{0.5607} \\
    Botswana & NMF & NMF (rank15) & 0.6201 & \textbf{0.6034} & \textbf{0.6164} & 0.4002 & 0.4726 & 0.4289 \\
    \midrule
    Salinas A & OT-SAE & $m=7$, $c=10^{-3}$ & 0.7628 & 0.7590 & 0.8012 & \textbf{0.8085} & \textbf{0.7846} & 0.7765 \\
    Salinas A & OT-SAE & $m=7$, $c=10^{-5}$ & 0.7794 & \textbf{0.7780} & 0.7757 & 0.7900 & 0.7553 & 0.7411 \\
    Salinas A & N-SAE & $m=7$, $c=10^{-3}$ & 0.2853 & 0.5139 & 0.5830 & 0.6240 & 0.6518 & 0.6406 \\
    Salinas A & N-SAE & $m=7$, $c=10^{-5}$ & 0.6530 & 0.6283 & \textbf{0.8258} & 0.7310 & 0.7697 & \textbf{0.8191} \\
    Salinas A & NMF & NMF (rank7) & \textbf{0.8523} & 0.6184 & 0.7042 & 0.6687 & 0.7844 & 0.7392 \\
    \bottomrule
\end{tabularx}
\end{subtable}
\end{table}

\section{Additional Details: 3D Point-Clouds}\label{sec:Additional Details Point Clouds}

The point-cloud experiments were run on a MacBook Pro with an Apple M2 Pro processor and 16 GB of unified memory, using the PyTorch MPS backend.

We use $\epsilon=0.025$, and $\rho$ given by the uniform measure from 1000 points uniformly sampled from $[0,1]^3$. We run Algorithm \ref{alg:entropic-lwdl} on this data, with $\mathcal{Y}$ given by a $17\times 17\times 17$ uniformly spaced grid in $[0,1]^3$, $\epsilon=0.025$, $c=0.0001$, $m=30$, $L=20$ and ReLU activations. We train for 2000 epochs, achieving a final mean squared $L^2(\rho)$ reconstruction error of 0.00173.

\begin{algorithm}[htbp]
\caption{Sample a ModelNet mesh as a point-cloud measure}\label{alg:modelnet_mesh_to_measure}
\textbf{Input:} ModelNet mesh file $M$ in \texttt{.off} format, support size $n$, random seed $s$.\\

$V,F \gets \textsc{ParseOFF}(M)$ \tcp*{Read vertices and polygonal faces}

$\Delta \gets \textsc{Triangulate}(F)$ \tcp*{Fan-triangulate polygonal faces}

$A_\ell \gets \frac{1}{2}\left\|(v_{\ell,2}-v_{\ell,1}) \times (v_{\ell,3}-v_{\ell,1})\right\|_2
\quad \text{for each triangle } \Delta_\ell$ \tcp*{Triangle areas}

$p_\ell \gets A_\ell / \sum_k A_k$ \tcp*{Sample triangles proportional to surface area}

\For{$i=1,\dots,n$}{
    $\ell_i \sim \mathrm{Categorical}(p)$ \tcp*{Choose a surface triangle}

    $r_1,r_2 \sim \mathrm{Unif}(0,1)$

    $a \gets 1-\sqrt{r_1}, \quad
     b \gets \sqrt{r_1}(1-r_2), \quad
     c \gets \sqrt{r_1}r_2$ \tcp*{Uniform barycentric coordinates}

    $y_i \gets a v_{\ell_i,1} + b v_{\ell_i,2} + c v_{\ell_i,3}$
}

$\{\tilde y_i\}_{i=1}^n \gets \textsc{NormalizeToUnitCube}(\{y_i\}_{i=1}^n)$ \tcp*{Center and rescale into $[0,1]^3$}

$w_i \gets 1/n \quad \text{for all } i=1,\dots,n$ \tcp*{Uniform masses}

$\mu_M \gets \sum_{i=1}^n w_i\,\delta_{\tilde y_i}$ \tcp*{Point-cloud probability measure}

\Return $\mu_M$
\end{algorithm}

\section{Details: Synthetic LBCM}

We provide implementation details for the experiment in Section~\ref{sec:recovering data from synthetic}.

This experiment was run on a MacBook Pro with an Apple M2 Pro processor and 16 GB of unified memory, using the PyTorch MPS backend.

We run Algorithm~\ref{alg:entropic-lwdl} with \(m=5\) atoms, entropic parameter \(\epsilon=0.025\), and a 20-layer LISTA encoder. As described in Algorithm \ref{alg:entropic-lwdl}, we embed the data with the displacement parameterization $V_{\rho\rightarrow \mu_i}$ and center the displacements using the mean computed from the training split. Sparsity is encouraged using a ReLU encoder together with an \(\ell_1\) coefficient penalty \(c=10^{-4}\). We train for 1000 epochs using AdamW, batch size 1024, learning rate \(1.5\times 10^{-4}\), and a cosine schedule with minimum learning rate \(5\times 10^{-6}\). We report the checkpoint at epoch 862 from the selected run with model seed 0.

In the LWDL-EOT model, each learned atom is represented as a free-support empirical measure by pushing forward \(\rho\) through a learned entropic map. Concretely, we choose target grid points
\(\mathcal Y=\{y_q\}_{q=1}^{64^2}\) on a uniform grid in \([0,1]^2\). For atom \(j\), the learned parameters are centered grid potentials \(h_j\in\mathbb R^{64^2}\), and the map is
\[
\begin{aligned}
A_j^\theta(x_\ell)
&=\sum_{q=1}^{64^2}\pi_{j\ell q}y_q,\\
\pi_{j\ell q}
&\propto\exp\left(
\frac{-\frac12\|x_\ell-y_q\|^2+h_{jq}}{\epsilon}
\right).
\end{aligned}
\]
The corresponding learned atom is \((A_j^\theta)_\#\rho\), supported by a set of points \(\{A_j^\theta(x_\ell)\}_{\ell=1}^M\) in the convex hull of $\{y_q\}_{q=1}^{64^2}.$

Qualitatively, the reconstructions in Figure \ref{fig:mnist-ebcm-reconstructions} appear to capture ``large-scale'' features of data points, but fail to resolve fine details, which are obscured by local jitter. We posit that the jitter is due to the reconstructions inheriting properties of the base measure $\rho$, which is significantly more diffuse than the data, as is seen in the top right of Figure \ref{fig:mnist-ebcm-synthetic-recovery}. We give more evidence towards this hypothesis in Section \ref{subsec:ELOTDL}.

\label{appendix:SyntheticLBCMDetails}
\section{Details: WDL Comparison}
\label{appendix:ComparisonWDLDetails}

We provide implementation details for the comparison in Section~\ref{subsec:ComparisonWDL}.

All methods in this comparison were run on the same server equipped with four NVIDIA Tesla K80 GPUs.

In the experiments, we evaluate WDL and LWDL on data sets of size $n\in \{100,1000\}$, with $m=10$ atoms, and note that we scale the batch size from $B=100$ to $B=256$ for the LWDL methods on the $n=1000$ data sets. As the WDL implementation in \citet{schmitz2018wasserstein} relies on an additional regularization hyperparameter $\gamma$, we run two instances of the algorithm with $\gamma\in \{0.5,2\}$ in each trial. We use the WDL implementation released with \cite{schmitz2018wasserstein} without additional implementation-level optimization. Each method is run for 1000s, with additional time given to complete partial iterations, or failure due to numerical instability.

We note several practical differences between the WDL and LWDL implementations. As discussed in Section \ref{sec:ModelsMethodsRelatedWork}, WDL assumes a fixed support, whereas LWDL allows the support to vary. In addition, LWDL operates on precomputed transport-map embeddings of the data, while WDL works directly with the input measures. Accordingly, LWDL reconstruction proceeds in map space: for each data point, we form the coefficient-weighted combination of the learned map atoms, push the base measure $\rho$ forward through the resulting map, and evaluate the reconstructed measure using the squared $W_2$ distance. As LWDL-KP does not allow for synthesis of a probability measure in this way, we do not compare this directly to \citet{schmitz2018wasserstein}.

On the $n=100$ Pavia experiment, WDL reaches mean errors of $4.6\times 10^{-3}$ and $4.9\times 10^{-3}$ for $\gamma=0.5$ and $\gamma=2$, respectively, while LWDL-OT reaches $5\times 10^{-5}$. On the larger $n=1000$ experiment, both WDL runs terminate due to numerical instability at epoch 7, with final errors around $6\times 10^{-3}$, whereas LWDL-OT reaches $9.3\times 10^{-5}$ within the same time budget. For the comparison of WDL with LWDL-EOT, we use the two-dimensional MNIST data set generated in Section \ref{sec:recovering data from synthetic}. The same evaluation protocol is used: each method is run until the 1000s time budget is reached or until numerical failure, and performance is measured by the final mean squared $W_2$ reconstruction error. On MNIST with $n=100$, WDL obtains errors of $1.2\times 10^{-2}$ and $1.0\times 10^{-2}$ for $\gamma=0.5$ and $\gamma=2$, respectively, while LWDL-EOT obtains $2.6\times 10^{-3}$. For $n=1000$, both WDL runs obtain errors of $1.3\times 10^{-2}$, compared with $2.7\times 10^{-3}$ for LWDL-EOT. We note that the reported times for the LWDL methods include the time spent computing $\mu\rightarrow T^\epsilon_{\rho\rightarrow \mu}$, the initial map embeddings of the data measures.

\section{Additional Details: LLM Data}\label{sec:additional details llm}

 Pythia-410M has $\approx 410$ million parameters, with $d_{model}=1024$ and 24 layers, and is trained on a deduplicated version of The Pile \citep{gao2020pile} data set. 

\subsection{Corrupted Text}

The synthetic text-corruption experiments were run on a MacBook Pro with an Apple M2 Pro processor and 16 GB of unified memory, using the PyTorch MPS backend.

\subsubsection{Corruptions}
We define the following corruptions. Each corruption comes with parameters we call \textit{intensities}, denoted $r_i$:
\begin{itemize}
\item \textbf{$r$-Global Swap Words:} Uniformly select $2r$ distinct word tokens $w_1,\dots,w_{2r}$ in $D^{(0)}$. Pair them as
\[
(w_1,w_2), (w_3,w_4), \dots, (w_{2r-1}, w_{2r}),
\]
and swap the two words within each pair.
\item \textbf{$r$-Permute Letters Within Words:}
Uniformly select $r$ distinct word tokens in $D^{(0)}$ among those of length at least $3$. For each selected word, apply a uniformly random permutation to its letters. Words of length $1$-$2$ are not modified.
\item \textbf{$r$-Delete Words:}
View $D^{(0)}$ in word-position space $\{1,\dots,n\}$. Uniformly (subject to feasibility) select and delete $r$ words, leaving punctuation and whitespace tokens unchanged.
\item \textbf{$r$-Duplicate Words:}
Uniformly select $r$ distinct word tokens $w_1,\dots,w_{r}$ in $D^{(0)}$. For each selected word $w$, insert a duplicate copy of $w$ immediately after it, separated by a single space, and ensure a trailing space after the duplicate unless the following token already begins with whitespace.
\end{itemize}

Our base document $D^{(0)}$ is an 850 word review of a television show, which is in the original training set of our model. When tokenized and embedded as a measure $\rho$, $|\text{supp}(\rho)|=1130.$ We generate $19200=4\times 12\times 400$ data documents:
\begin{itemize}
\item \textbf{Global Swap Words} with $r\in \{1,25,50,75,100,125,150,175,200,225,250,275\}$;
\item \textbf{Delete Word Spans} with $r\in \{1,25,50,75,100,125,150,175,200,225,250,275\}$;
\item \textbf{Permute Letters} with $r\in \{1,25,50,75,100,125,150,175,200,225,250,275\}$;
\item \textbf{Duplicate Word} with $r\in \{1,25,50,75,100,125,150,175,200,225,250,275\}$.
\end{itemize}

We run Algorithm \ref{alg:KP_SAE_embedding} on the corrupted data with four architectures:
\begin{enumerate}
\item $10$ atoms and $\chi=\text{TopK}_2$.
\item $40$ atoms and $\chi=\text{TopK}_3$.
\item $10$ atoms and $\chi=\text{JumpReLU}_\tau$, with learned $\tau\geq0$ and $\ell^1$ regularization $c=10^{-3}.$
\item $40$ atoms and $\chi=\text{JumpReLU}_\tau$, with learned $\tau\geq 0$ and $\ell^1$ regularization $c=10^{-3}$.
\end{enumerate}
We set $\ell^1$ regularization to zero for the trials with $\text{TopK}$ activation, due to the hard sparsity constraint in $\text{TopK}$. We train for 2000 epochs with the AdamW optimizer, with an initial learning rate of $5*10^{-6}$ and batch size $1024$.
\begin{figure}[t!]
    \centering

    \begin{subfigure}{0.49\linewidth}
        \centering
        \includegraphics[width=\linewidth]{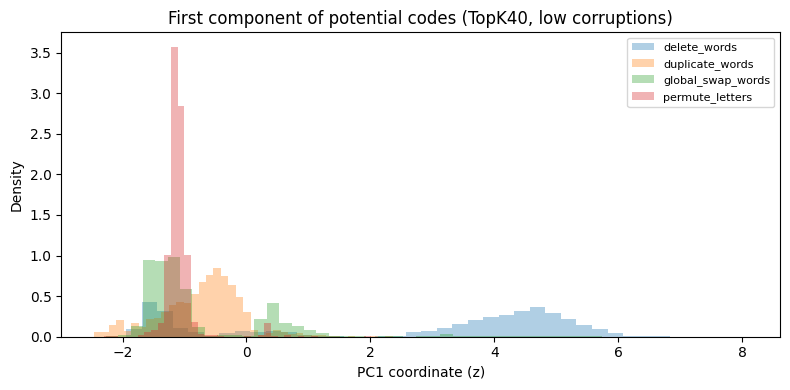}
        \label{fig:topk40_low}
    \end{subfigure}
    \hfill
    \begin{subfigure}{0.49\linewidth}
        \centering
        \includegraphics[width=\linewidth]{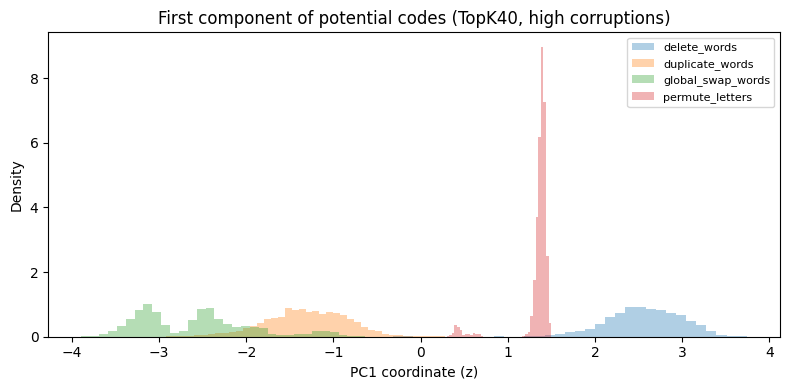}
        \label{fig:topk40}
    \end{subfigure}

    \caption{(Left) Projection of subset of data with $r\leq 75$ to principal component of $\text{TopK}_{40}$ SAE. (Right) Projection of subset of data with $r> 75$ to principal component of $\text{TopK}_{40}$ SAE.}
    \label{fig:topk40_combined}
\end{figure}

To further investigate the effect of $r$ on the embeddings, in Figure \ref{fig:topk40_combined} we plot the codes projected onto the first principal component of the $\text{TopK}_{40}$ atoms, where the plot on the left visualizes the projection of data points with a low degree of corruption ($r\leq 75$), and the plot on the right visualizes projection of highly corrupted data points ($r> 75$). We see that there is significantly less overlap between classes in the highly corrupted data, which aligns with the jump in accuracy at higher corruption strength.

\subsection{The Pile}\label{sec:pile details}

The Pile experiment was run on a server equipped with four NVIDIA Tesla K80 GPUs. In Table \ref{tab:feature_purity}, we display the composition of sources for top ranked features by total coefficient magnitude.

\begin{table}[t!]
\centering
\caption{Top learned features ranked by source purity. Purity is computed as the fraction of labeled top-activating documents belonging to the most common Pile source.}
\label{tab:feature_purity}
\small
\setlength{\tabcolsep}{4pt}
\begin{tabular}{rrrrlp{0.35\linewidth}}
\toprule
Rank & Feature & Purity & Docs & Top source & Source distribution \\
\midrule
1 & 83 & 1.00 & 100 & Pile-CC & Pile-CC (100\%) \\
2 & 263 & 1.00 & 100 & StackExchange & StackExchange (100\%) \\
3 & 82 & 1.00 & 100 & Pile-CC & Pile-CC (100\%) \\
4 & 324 & 1.00 & 100 & USPTO Backgrounds & USPTO Backgrounds (100\%) \\
5 & 314 & 1.00 & 100 & PubMed Abstracts & PubMed Abstracts (100\%) \\
6 & 74 & 1.00 & 100 & PubMed Abstracts & PubMed Abstracts (100\%) \\
7 & 134 & 1.00 & 100 & StackExchange & StackExchange (100\%) \\
8 & 506 & 1.00 & 100 & Github & Github (100\%) \\
9 & 210 & 1.00 & 100 & Github & Github (100\%) \\
10 & 407 & 1.00 & 100 & StackExchange & StackExchange (100\%) \\
11 & 469 & 1.00 & 100 & Wikipedia (en) & Wikipedia (en) (100\%) \\
12 & 334 & 1.00 & 100 & Pile-CC & Pile-CC (100\%) \\
13 & 479 & 1.00 & 100 & Github & Github (100\%) \\
14 & 327 & 1.00 & 100 & Wikipedia (en) & Wikipedia (en) (100\%) \\
15 & 178 & 1.00 & 100 & FreeLaw & FreeLaw (100\%) \\
16 & 435 & 1.00 & 100 & Pile-CC & Pile-CC (100\%) \\
17 & 359 & 1.00 & 100 & HackerNews & HackerNews (100\%) \\
18 & 98 & 1.00 & 100 & Wikipedia (en) & Wikipedia (en) (100\%) \\
19 & 260 & 1.00 & 100 & Wikipedia (en) & Wikipedia (en) (100\%) \\
20 & 245 & 0.99 & 100 & Github & Github (99\%), StackExchange (1\%) \\
21 & 132 & 0.99 & 100 & Wikipedia (en) & Wikipedia (en) (99\%), PubMed Abstracts (1\%) \\
22 & 44 & 0.99 & 100 & Wikipedia (en) & Wikipedia (en) (99\%), Pile-CC (1\%) \\
23 & 338 & 0.99 & 100 & Github & Github (99\%), StackExchange (1\%) \\
24 & 341 & 0.98 & 100 & Github & Github (98\%), StackExchange (2\%) \\
25 & 89 & 0.98 & 100 & PubMed Abstracts & PubMed Abstracts (98\%), NIH ExPorter (2\%) \\
26 & 501 & 0.97 & 100 & Github & Github (97\%), StackExchange (3\%) \\
27 & 419 & 0.97 & 100 & Wikipedia (en) & Wikipedia (en) (97\%), Pile-CC (1\%), ArXiv (1\%), Github (1\%) \\
28 & 32 & 0.97 & 100 & StackExchange & StackExchange (97\%), Github (3\%) \\
29 & 373 & 0.97 & 100 & PubMed Abstracts & PubMed Abstracts (97\%), NIH ExPorter (2\%), Pile-CC (1\%) \\
30 & 90 & 0.97 & 100 & Github & Github (97\%), Wikipedia (en) (1\%), Pile-CC (1\%), PubMed Abstracts (1\%)
\end{tabular}
\end{table}

As stated in the main body, atom specialization goes beyond simply identifying the source: An interesting case study is with $f_{469},f_{327},f_{132}$ and $f_{44}$, which are all strongly $(>99\%)$ associated with Wikipedia articles, but have noticeably different content among their respective top 20 documents: $f_{469}$ predominantly activates on athletes, $f_{327}$ on general notable figures, $f_{132}$ on population centers (villages, counties, islands), and $f_{44}$ on disambiguation pages \citep{wikipedia_help_disambiguation} for proper nouns.  We note that not all features can be distinguished in this way, for instance, $f_{98}$ and $f_{469}$ have nearly identical top documents, indicating substantial redundancy among the learned atoms, particularly among highly source-concentrated atoms. This could possibly be mitigated by increasing sparsity or filtering correlated features.

Towards more quantitative measures of atom specificity, we rank atoms by the enrichment of a given $n$-gram in their top documents. If $\mathcal{D}_j$ denotes the top documents associated with $f_j$, $\mathcal{D}$ denotes the full corpus, and $w$ denotes an $n$-gram, define $E(w,j):=
\frac{\mathbb{P}_{D\sim \mathcal{D}_j}(w\in D)}
{\mathbb{P}_{D\sim \mathcal{D}}(w\in D)}.$  We find, for instance, that while $f_{134},f_{263},$ and $f_{407}$ are entirely associated with documents from StackExchange, $E(\verb|x^2|,407)=133.82$, whereas $E(\verb|x^2|,j)=0$ for $j\in \{134,263\}$. Inspecting $\mathcal{D}_{407}$, we indeed see that it is dominated by Math StackExchange posts, again indicating that atoms can be associated with highly specific documents.

\section{Impact of Base Measure}
\label{appendix:ImpactBaseMeasure}

This section supplements the discussion in Section~\ref{subsec:ELOTDL}.

Finally, we illustrate the dependence of the learned atoms on the base measure $\rho$. In the bottom panel of Figure \ref{fig:atoms-comparison}, we plot a selection of (pushforwards of) five atoms learned by Algorithm \ref{alg:entropic-lwdl} when the base measure is supported on a uniform sample from $[0,1]^3$, versus the atoms learned with a base measure uniformly supported on a point-cloud in the data set. We see that both sets of atoms inherit qualitative properties of their respective base measures. To see why this is natural, we recall the Riemannian perspective on the Wasserstein space, under which $L^2(\rho)$ represents the tangent space at the point $\rho$. Under this interpretation, the atoms are elements of a Hilbert space that locally approximates the neighborhood of $\rho$ in the Wasserstein ``manifold'', and hence it is expected that these atoms strongly resemble $\rho$.

\begin{figure}[htbp]

    \includegraphics[width=0.95\linewidth]{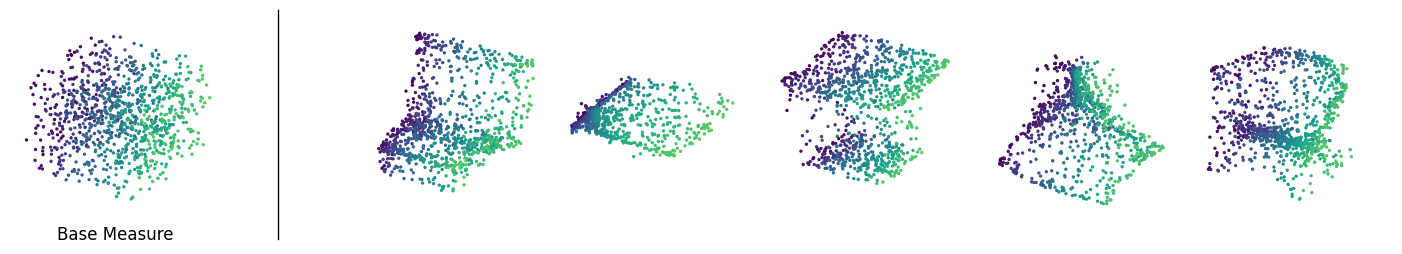}

    \medskip

    \includegraphics[width=0.95\linewidth]{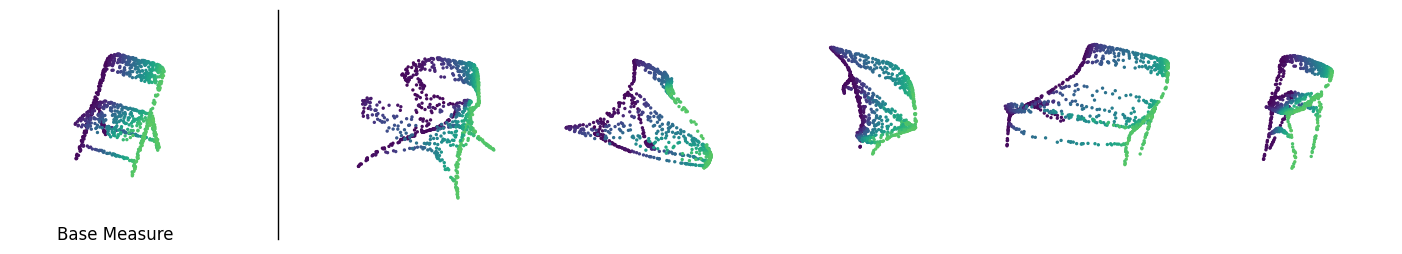}

    \caption{Dictionary atoms learned by our method when trained on point-cloud data, illustrating the dependence on the base measure. Colors indicate the value of the functional embedding on the point cloud.}

    \label{fig:atoms-comparison}
\end{figure}

\section{Proofs for the Generalization Error Bounds}
\label{appendix:ProofsOfTheorems}

\subsection{Preliminary Discussion of Theorems \ref{thm:samp_complex_dl} and \ref{thm:summary theorem}}\label{appendix:generalization proof sketch}

Theorem \ref{thm:samp_complex_dl} establishes that the generalization error for LWDL-OT and LWDL-EOT can be controlled by quantities depending on the moments of the data, the strength of regularization and the covering number of the set of candidate dictionary atoms. This is related to prior work in \citet{vainsencher2011sample} and \citet{gribonval2015sample}, which establish similar bounds for dictionary learning in both finite and infinite dimensions. Theorem 14 in \citet{vainsencher2011sample} is particularly relevant, which establishes generalization error bounds for dictionary learning of data living in a metric space $\mathcal{R}$, embedded into a Hilbert space $\mathcal{H}$ via a feature map. Here, the error bounds scale according to the covering numbers of $\mathcal{R}$, which are assumed to be of the order $(\frac{C}{\epsilon})^n$ for some $n$, and the H\"older continuity of the feature map, among other factors. This is a similar setting to our result, as our data consists of probability measure valued data in the Wasserstein space, which is embedded in a Hilbert space $L^2(\rho)$ via the transport map $T_{\rho\rightarrow \mu}$. However, we note that Theorem 14 in \citet{vainsencher2011sample} does not apply to our applications of Theorem \ref{thm:samp_complex_dl}--for instance, $\mathcal{P}_2(\mathcal{M}_2(\vartheta))$ strongly violates the covering number assumption.

The basic strategy for Theorem \ref{thm:summary theorem} is as follows.  To derive covering numbers for the class of OT maps, we leverage known results for covering numbers of convex functions \citep{guntuboyina2013covering}, which immediately yield covering numbers for the class of Kantorovich potentials under consideration.  Leveraging Brenier's theorem to identify OT maps with gradients of convex functions, we then deduce covering numbers on OT maps themselves by applying recent technical results of \citet{delalande2023quantitative} that allow for a comparison of the metrics $\|\phi-\psi\|_{L^{2}(\rho)}$ and $\|\nabla\phi-\nabla\psi\|_{L^{2}(\rho)}$.  The consequent rates in 1. of Theorem \ref{thm:summary theorem} are cursed by dimensionality, unfortunately.  On the other hand, the rates for EOT maps are parametric, owing to the strong degree of regularity exhibited by EOT potentials.  Our proof for part 2. of Theorem \ref{thm:summary theorem} follows essentially from known results of \cite{mena2019statistical} that establish (i) strong regularity for EOT potentials under subgaussian assumptions on the base and target measures and (ii) covering numbers for very regular functions.

\subsection{Technical Lemmas}
An important technical result we will use repeatedly is the following.

\begin{lemma}
\label{lem:CompressionUnderEOT}
Let $\mu,\nu\in \mathcal P_2(\mathbb R^d)$, and let $T_{\mu\rightarrow\nu}^{\epsilon}$ be the associated EOT map for $\epsilon\ge 0$.  Then $M_{2}((T_{\mu\rightarrow\nu}^{\epsilon})_{\#}\mu)\le M_{2}(\nu)$.

\end{lemma}

\begin{proof}Let $\pi_\epsilon$ be the associated EOT coupling.  The result follows from Jensen's inequality:
\begin{align*}
M_{2}((T_{\mu\rightarrow\nu}^{\epsilon})_{\#}\mu)=&\int_{\mathbb{R}^{d}}
\|x\|^{2}d((T_{\mu\rightarrow\nu}^{\epsilon})_{\#}\mu)(x)\\
=&\int_{\mathbb{R}^d} \|T_{\mu\rightarrow\nu}^{\epsilon}(x)\|^2 d\mu(x)\\
=&\int_{\mathbb{R}^d} \left\|\int_{\mathbb{R}^{d}}yd\pi_{\epsilon}(y \ | \ x)\right\|^2 d\mu(x)\\
\le&\int_{\mathbb{R}^{d}\times\mathbb{R}^{d}}\|y\|^{2}d\pi_{\epsilon}(y \ | \ x)d\mu(x)\\
= & \int_{\mathbb{R}^d\times\mathbb{R}^d}\|y\|^2 d\pi_\epsilon(x,y)\\
= & \int_{\mathbb{R}^{d}}\|y\|^{2} d\nu(y)\\
=&M_{2}(\nu).
\end{align*}
\end{proof}

We note that equality holds when $\epsilon=0$ and that this result follows from a special case of Lemma 9.2 in \cite{yang2023estimating}.

We begin by proving the existence of solutions to (\ref{eqn:ptwise DL objective}):

\begin{proposition}\label{proposition:existence of DL}
Let $\Theta\in \mathcal{P}(\mathcal{M}_2(\vartheta))$ for some $\vartheta\geq 0$. Fix $\epsilon\ge 0$ and assume $T_{\rho\rightarrow\mu}^{\epsilon}$ exists for $\Theta$-a.e. $\mu$. Let $g:\mathbb{R}^m_{\geq 0}\rightarrow \mathbb{R}_{\geq 0}$ be lower semicontinuous, such that $g(0)=0$ and $g(x)\geq c\|x\|_1$ for some $c> 0$. Let $\mathbb{T}\subseteq L^2(\rho)$ be compact such that for all $T\in \mathbb{T}$, we have $\|T\|_{L^2(\rho)}=1$. Define $h_\rho^g$ as in (\ref{eqn:ptwise DL objective}). Then $\inf_{\mathcal{T}\in \mathbb{T}^m}\mathbb{E}_{\mu\sim \Theta}[h^g_\rho(\mu,\epsilon,\mathcal{T})]$ admits a minimizer.
\end{proposition}

To prove Proposition \ref{proposition:existence of DL}, we will require the following technical results, that establish quantitative controls on (\ref{eqn:ptwise DL objective}):

\begin{lemma}\label{lemma:b/c}
Let $g:\mathbb{R}^m_{\geq 0}\rightarrow \mathbb{R}_{\geq 0}$ be lower semicontinuous, such that $g(0)=0$ and $g(x)\geq c\|x\|_p$ for some $c>0$ and $p\geq 1$. Then, for any $\epsilon\geq 0$:
\begin{enumerate}
\item $\argmin_{\lambda\in \mathbb{R}^m_{\geq 0}}\frac{1}{2}\|T_{\rho\rightarrow \mu}^\epsilon-\sum_{j=1}^m\lambda_jT_j\|^2_{L^2(\rho)}+g(\lambda)$ is nonempty.
\item For any $\lambda^*\in\argmin_{\lambda\in \mathbb{R}^m_{\geq 0}}\frac{1}{2}\|T_{\rho\rightarrow \mu}^\epsilon-\sum_{j=1}^m\lambda_jT_j\|^2_{L^2(\rho)}+g(\lambda)$, we have\begin{equation*}\|\lambda^*\|_p\leq \frac{M_2(\mu)}{2c}.\end{equation*}
\end{enumerate}
\end{lemma}
\begin{proof}
Observe that we have $\inf_{\lambda\in \mathbb{R}_{\ge 0}^{m}}\frac{1}{2}\|T_{\rho\rightarrow \mu}^\epsilon-\sum_{j=1}^m\lambda_jT_j\|^2_{L^2(\rho)}+g(\lambda)\leq \frac{1}{2}\|T_{\rho\rightarrow \mu}^\epsilon\|_{L^2(\rho)}^2\leq \frac{1}{2}M_2(\mu)$. The first inequality is obtained by setting $\lambda=0$, and the second is by applying Lemma \ref{lem:CompressionUnderEOT}. Let $\{\lambda^{(i)}\}_{i\in \mathbb{N}}$ be an infimizing sequence for $F(\lambda):=\frac{1}{2}\|T_{\rho\rightarrow \mu}^\epsilon-\sum_{j=1}^m\lambda_jT_j\|^2_{L^2(\rho)}+g(\lambda).$  We compute:
\begin{align*}
    \lim_{i\rightarrow\infty}F(\lambda^{(i)})= \inf_{\lambda\in \mathbb{R}_{\ge 0}^{m}}F(\lambda)\le F(0)\le \frac{1}{2}M_{2}(\mu).
\end{align*}
Hence, for all $i$ sufficiently large,
\begin{align*}
    c\|\lambda^{(i)}\|_{p}\le F(\lambda^{(i)})\le \frac{1}{2}M_{2}(\mu)+1
\end{align*}
We can thus extract a convergent subsequence $\{\lambda^{(i')}\}\subset\{\lambda^{(i)}\}$ with limit $\lambda^{*}$.  By lower semicontinuity of $F$, we have that $F(\lambda^{*})\le\lim_{i'\rightarrow\infty}F(\lambda^{(i')})=\inf_{\lambda\in\mathbb{R}_{\ge 0}^{m}}F(\lambda)$, which demonstrates that $\lambda^{*}$ is a minimizer for $F$.  This establishes 1.

Furthermore, for any minimizer $\lambda^{*}$ we have that
\begin{align*}&c\|\lambda_{*}\|_{p}\le F(\lambda^{*})\le F(0)\le \frac{1}{2}M_{2}(\mu),
\end{align*}which establishes 2.

\end{proof}

\begin{lemma} \label{lemma:lipschitz objective} Let $g:\mathbb{R}^m_{\geq 0}\rightarrow \mathbb{R}_{\geq 0}$ be lower semicontinuous, such that $g(0)=0$ and $g(x)\geq c\|x\|_p$ for some $c>0$ and $p\geq 1$. Let $\mathcal{T},\tilde{\mathcal{T}}\in \mathbb{T}^m$, where $\mathbb{T}\subset\{ T\in  L^2(\rho), \|T\|_{L^2(\rho)}=1\}$. Then, for any $\mu\in \mathcal{P}_2(\mathbb{R}^d)$:
\begin{align*}
&|h^g_\rho(\mu,\epsilon,\mathcal{T})-h^g_\rho(\mu,\epsilon,\tilde{\mathcal{T}})|\\&\leq \max_{\alpha\in \{\lambda^*,\tilde{\lambda}^*\}}\left\{\|\alpha\|_1\sum_{k=1}^m\alpha_k\|T_k-\tilde{T}_k\|_{L^2(\rho)}+M_2(\mu)^{1/2}\sum_{j=1}^m\alpha_j\|T_j-\tilde{T}_j\|_{L^2(\rho)}\right\}
\end{align*}
where $\lambda^*\in\argmin_{\lambda\in \mathbb{R}^m_{\geq 0}}\frac{1}{2}\|T_{\rho\rightarrow \mu}^\epsilon-\sum_{j=1}^m\lambda_jT_j\|^2_{L^2(\rho)}+g(\lambda)$ and $\tilde{\lambda}^*\in\argmin_{\lambda\in \mathbb{R}^m_{\geq 0}}\frac{1}{2}\|T_{\rho\rightarrow \mu}^\epsilon-\sum_{j=1}^m\lambda_j\tilde{T}_j\|^2_{L^2(\rho)}+g(\lambda)$.
\end{lemma}
\begin{proof}
 For any $\lambda\in \mathbb{R}^m_{\geq 0}$, we have:
\begin{align*}
&\frac{1}{2}\|T_{\rho\rightarrow \mu}^\epsilon-\sum_{j=1}^m\lambda_j T_j\|_{L^2(\rho)}^2-\frac{1}{2}\|T_{\rho\rightarrow \mu}^\epsilon-\sum_{j=1}^m\lambda_j \tilde{T}_j\|^2_{L^2(\rho)}\\& =\langle T_{\rho\rightarrow \mu}^\epsilon,\sum_{j=1}^m\lambda_j (\tilde{T}_j- T_j)\rangle_{L^2(\rho)}+\frac{1}{2}\|\sum_{j=1}^m\lambda_j T_j\|_{L^2(\rho)}^2-\frac{1}{2}\|\sum_{j=1}^m\lambda_j \tilde{T}_j\|_{L^2(\rho)}^2\\&=\langle T_{\rho\rightarrow \mu}^\epsilon,\sum_{j=1}^m\lambda_j (\tilde{T}_j-T_j)\rangle_{L^2(\rho)} +\frac{1}{2}\langle \sum_{j=1}^m \lambda_j T_j,\sum_{k=1}^m\lambda_k (T_k-\tilde{T}_k)\rangle_{L^2(\rho)}\\&+\frac{1}{2}\langle \sum_{k=1}^m\lambda_k(T_k- \tilde{T}_k),\sum_{j=1}^m\lambda_j \tilde{T}_j\rangle_{L^2(\rho)}
\\&\leq \|T_{\rho\rightarrow \mu}^\epsilon\|_{L^2(\rho)}\|\sum_{j=1}^m\lambda_j (\tilde{T}_j-T_j)\|_{L^2(\rho)} + \frac{1}{2}\|\sum_{j=1}^m\lambda_j T_j\|_{L^2(\rho)}\|\sum_{k=1}^m\lambda_k(T_k-\tilde{T}_k)\|_{L^2(\rho)}\\&+\frac{1}{2}\|\sum_{j=1}^m\lambda_j \tilde{T}_j\|_{L^2(\rho)}\|\sum_{k=1}^m\lambda_k(T_k-\tilde{T}_k)\|_{L^2(\rho)}.\tag{Cauchy-Schwarz}
\end{align*}

We bound the first term as:

\begin{align*}
&\|T_{\rho\rightarrow \mu}^\epsilon\|_{L^2(\rho)}\|\sum_{j=1}^m\lambda_j (T_j-\tilde{T}_j)\|_{L^2(\rho)}\\& \leq M_2(\mu)^{1/2}\sum_{j=1}^m\lambda_j\|T_j-\tilde{T}_j\|_{L^2(\rho)}\tag{Lemma \ref{lem:CompressionUnderEOT}}
\end{align*}
We bound the second term as:

\begin{align*} &\frac{1}{2}\|\sum_{j=1}^m\lambda_j T_j\|_{L^2(\rho)}\|\sum_{k=1}^m\lambda_k(T_k-\tilde{T}_k)\|_{L^2(\rho)}\\ &\leq\frac{1}{2}\sum_{j=1}^m \lambda_j \|T_j\|_{L^2(\rho)}\sum_{k=1}^m \lambda_k \|T_k-\tilde{T}_k\|_{L^2(\rho)}\tag{Triangle ineq. and nonnegativity of $\lambda_j$}\\& =\frac{1}{2}\|\lambda\|_{1}\sum_{k=1}^m \lambda_k\|T_k-\tilde{T}_k\|_{L^2(\rho)}\tag{$\|T_j\|_{L^2(\rho)}=1$} \end{align*}
The same bound holds for the third term.

Combining our bounds, we conclude:
\begin{align*}
&\frac{1}{2}\|T_{\rho\rightarrow \mu}^\epsilon-\sum_{j=1}^m\lambda_{j}T_j\|_{L^2(\rho)}^2-\frac{1}{2}\|T_{\rho\rightarrow \mu}^\epsilon-\sum_{j=1}^m\lambda_{j}\tilde{T}_j\|_{L^2(\rho)}^2\\&\leq\|\lambda\|_{1}\sum_{k=1}^m\lambda_k\|T_k-\tilde{T}_k\|_{L^2(\rho)}+M_2(\mu)^{1/2}\sum_{j=1}^m\lambda_j\|T_j-\tilde{T}_j\|_{L^2(\rho)}\\&=:Q(\lambda)
\end{align*}

Let $\lambda^*\in \argmin_{\lambda\in \mathbb{R}^m}\frac{1}{2}\|T_{\rho\rightarrow \mu}^\epsilon-\sum_{j=1}^m\lambda_j T_j\|_{L^2(\rho)}^2+g(\lambda)$. Then:
\begin{align*}
h^g_\rho(\mu,\epsilon,\tilde{\mathcal{T}})&=\inf_{\lambda\in \mathbb{R}^m}\frac{1}{2}\|T_{\rho\rightarrow \mu}^\epsilon-\sum_{j=1}^m\lambda_j \tilde{T}_j\|_{L^2(\rho)}^2+g(\lambda)\\& \leq \frac{1}{2}\|T_{\rho\rightarrow \mu}^\epsilon-\sum_{j=1}^m\lambda^*_j\tilde{T}_j\|_{L^2(\rho)}^2+g(\lambda^*) \tag{Suboptimality of $\lambda^*$}\\&= \frac{1}{2}\|T_{\rho\rightarrow \mu}^\epsilon-\sum_{j=1}^m\lambda^*_j\tilde{T}_j\|_{L^2(\rho)}^2+g(\lambda^*)\\&  -\frac{1}{2}\|T_{\rho\rightarrow \mu}^\epsilon-\sum_{j=1}^m\lambda^*_jT_j\|_{L^2(\rho)}^2\\& +\frac{1}{2}\|T_{\rho\rightarrow \mu}^\epsilon-\sum_{j=1}^m\lambda^*_jT_j\|_{L^2(\rho)}^2 \\& \leq Q(\lambda^*)+h^g_\rho(\mu,\epsilon,\mathcal{T}).
\end{align*}
We conclude:
\begin{align*}
h^g_\rho(\mu,\epsilon,\tilde{\mathcal{T}})-h^g_\rho(\mu,\epsilon,\mathcal{T})\leq Q(\lambda^*).
\end{align*}
Via a symmetric argument we conclude:
\begin{equation*}
|h^g_\rho(\mu,\epsilon,\tilde{\mathcal{T}})-h^g_\rho(\mu,\epsilon,\mathcal{T})|\leq \max\{Q(\lambda^*),Q(\tilde{\lambda}^*)\}.
\end{equation*}
\end{proof}

Combining Lemma \ref{lemma:b/c} and Lemma \ref{lemma:lipschitz objective}, we establish that $h_\rho^g(\mu,\epsilon,\mathcal{T})$ is Lipschitz in $\mathcal{T}$, with a constant depending on the second moments of $\mu$ and the coercivity of $g:$
\begin{corollary}\label{cor:bounded_diam}
Let $g:\mathbb{R}^m_{\geq 0}\rightarrow \mathbb{R}_{\geq 0}$ be lower semicontinuous, such that $g(0)=0$ and $g(x)\geq c\|x\|_1$. Then, defining $h_\rho^g(\mu,\epsilon,\mathcal{T})$ as in (\ref{eqn:ptwise DL objective}), for any $\mu \in \mathcal{P}_2(\mathbb{R}^d)$ we have:
\begin{equation}
|h^g_\rho(\mu,\epsilon,\mathcal{T})-h^g_\rho(\mu,\epsilon,\tilde{\mathcal{T}})|\leq \frac{M_2(\mu)^2+M_2(\mu)^{3/2}}{2\min\{c,c^2\}}\max_{1\leq j\leq m}\|T_j-\tilde{T}_j\|_{L^2(\rho)}.\label{eqn:lipschitz h bound}
\end{equation}
Moreover, let $\Theta\in \mathcal{P}(\mathcal{M}_2(\vartheta))$ for some $\vartheta\geq 0$. Then:
\begin{align}
\mathbb{E}_{\mu\sim\Theta}[|h_\rho^g(\mu,\epsilon,\mathcal{T})-h_\rho^g(\mu,\epsilon,\tilde{\mathcal{T}})|]\leq \frac{\vartheta^2+\vartheta^{3/2}}{2\min\{c,c^2\}}\max_{1\leq \ell\leq  m}\|T_\ell-\tilde{T}_\ell\|_{L^2(\rho)}.\label{eqn:lipschitz exp h bound}
\end{align}
\end{corollary}
\begin{proof}
To establish (\ref{eqn:lipschitz h bound}), by Lemma \ref{lemma:lipschitz objective} we have:
\begin{align*}
&|h^g_\rho(\mu,\epsilon,\mathcal{T})-h^g_\rho(\mu,\epsilon,\tilde{\mathcal{T}})|\\& \leq \max_{\alpha\in \{\lambda^*,\tilde{\lambda}^*\}}\left\{\|\alpha\|_1\sum_{k=1}^m\alpha_k\|T_k-\tilde{T}_k\|_{L^2(\rho)}+M_2(\mu)^{1/2}\sum_{j=1}^m\alpha_j\|T_j-\tilde{T}_j\|_{L^2(\rho)}\right\}\\&\leq   \max_{\alpha\in \{\lambda^*,\tilde{\lambda}^*\}}\left\{\max_{1\leq \ell\leq m}\|T_\ell-\tilde{T}_\ell\|_{L^2(\rho)}\|\alpha\|_1\sum_{k=1}^m \alpha_k+M_2(\mu)^{1/2}\max_{1\leq \ell \leq m}\|T_\ell-\tilde{T}_\ell\|_{L^2(\rho)}\sum_{j=1}^m\alpha_j\right\}\\&=\max_{\alpha\in \{\lambda^*,\tilde{\lambda}^*\}}\left\{\max_{1\leq \ell\leq m}\|T_\ell-\tilde{T}_\ell\|_{L^2(\rho)}\left(\|\alpha\|_{1}^2+M_2(\mu)^{1/2}\|\alpha\|_1\right)\right\}\\&\leq \max_{1\leq \ell\leq m}\|T_\ell-\tilde{T}_\ell\|_{L^2(\rho)}\left(\frac{M_2(\mu)^2}{4c^2}+\frac{M_2(\mu)^{3/2}}{2c}\right).\tag{Lemma \ref{lemma:b/c}}
\end{align*}
For the second assertion, we have
\begin{align*}
& \mathbb{E}_{\mu\sim\Theta}\left[|h_\rho^g(\mu,\epsilon,\mathcal{T})-h_\rho^g(\mu,\epsilon,\tilde{\mathcal{T}})|\right]\\
\leq &\mathbb{E}_{\mu\sim\Theta}\left[\frac{M_2(\mu)^2+M_2(\mu)^{3/2}} {2\min\{c,c^2\}}\right]\max_{1\leq \ell\leq  m}\|T_\ell-\tilde{T}_\ell\|_{L^2(\rho)}.\tag{ \ref{eqn:lipschitz h bound}}\end{align*}
As $\text{supp}(\Theta)\subset \mathcal{M}_2(\vartheta)$, we may conclude (\ref{eqn:lipschitz exp h bound}).
\end{proof}
We may now prove existence.

\textbf{Proof of Proposition \ref{proposition:existence of DL}:}
For any $\mu \in \mathcal{P}_2(\mathbb{R}^d)$, $h_\rho^g(\mu,\epsilon,\mathcal{T})\geq 0$. Furthermore, for any $\mathcal{T}\in \mathbb{T}^m$ we have $h_\rho^g(\mu,\epsilon,\mathcal{T})\leq \|T^\epsilon_{\rho\rightarrow \mu}\|^2_{L^2(\rho)}$ (by taking $\lambda=0$), and hence:
\begin{align*}
\mathbb{E}_{\mu\sim \Theta}[h^g_\rho(\mu,\epsilon,\mathcal{T})]&\leq \mathbb{E}_{\mu\sim\Theta}\left[\|T_{\rho\rightarrow \mu}^\epsilon\|^2_{L^2(\rho)}\right]\\&\leq \mathbb{E}_{\mu\sim\Theta}[M_2(\mu)]\tag{Lemma \ref{lem:CompressionUnderEOT}}
\\&\le \vartheta<\infty.\end{align*}
Let $F(\Theta,\rho,\epsilon,\mathcal{T}):= \mathbb{E}_{\mu\sim\Theta}[h_\rho^g(\mu,\epsilon ,\mathcal{T})]$, and let $\{\mathcal{T}^{(i)}\}_{i\in \mathbb{N}}$ be an infimizing sequence for $F(\Theta,\rho,\epsilon,\cdot)$. As $\mathbb{T}$ is compact in $L^2(\rho)$, $\mathbb{T}^m$ is compact with the metric $d_{2,m}(\mathcal{T},\tilde{\mathcal{T}}):=\max_{1\leq j \leq m}\|T_j-\tilde{T}_j\|_{L^2(\rho)}$. Hence there exists a subsequence $\{\mathcal{T}^{(i')}\}_{i'\in \mathbb{N}}$ that converges in $d_{2,m}$ to a limit $\mathcal{T}^*\in \mathbb{T}^m$. By (\ref{eqn:lipschitz exp h bound}), $F(\Theta,\rho,\epsilon,\cdot)$ is (Lipschitz) continuous in $\mathcal{T}$ with respect to $d_{2,m}$, and hence:
\begin{align*}
\lim_{i'\rightarrow \infty}F(\Theta,\rho,\epsilon,\mathcal{T}^{(i')})=F(\Theta,\rho,\epsilon,\mathcal{T}^*)
\end{align*}
and hence $\mathcal{T}^*\in \argmin_{T}F(\Theta,\rho,\epsilon, \mathcal{T}).$ \qed

\subsection{Proof of Theorem \ref{thm:samp_complex_dl}}
We begin by sketching the basic proof idea.  First, we employ a standard argument to replace the generalization error with a worst-case deviation bound over the class $H^{g,\epsilon}_\rho(\mathbb{T}^m):=\{h_\rho^g(\cdot,\epsilon,\mathcal{T})|\mathcal{T}\in \mathbb{T}^m\}$, where $h_\rho^g$ is defined in (\ref{eqn:ptwise DL objective}). We will then need an instance of a well-known symmetrization argument, relating worst-case deviations of empirical averages from the expectation to the Rademacher complexity \citep{wainwright2019high} of $H_\rho^{g,\epsilon}(\mathbb{T}^m).$
\begin{lemma}\label{lemma:symmetrization}
Let $\Theta\in \mathcal{P}(\mathcal{M}_2(\vartheta))$ for some $\vartheta\geq0$. Let $g:\mathbb{R}^m_{\geq 0}\rightarrow \mathbb{R}_{\geq 0}$ be lower semicontinuous, such that $g(0)=0$ and $g(x)\geq c\|x\|_1$, for some $c> 0$. Let $\mathbb{T}\subseteq L^2(\rho)$ be compact, such that $\|T\|_{L^2(\rho)}=1$ for all $T\in \mathbb{T}$. Let $H^{g,\epsilon}_\rho(\mathbb{T}^m):=\{h_\mathcal{T} \ | \ \mathcal{T}\in \mathbb{T}^m\}$, where $h_\mathcal{T}(\mu):=h_\rho^g(\mu,\epsilon,\mathcal{T})$. Then:
\begin{align*}
&\mathbb{E}_{\mu_i\sim\Theta}\left[\sup_{h_\mathcal{T}\in H_\rho^{g,\epsilon}(\mathbb{T}^m)}\left|\frac{1}{n}\sum_{i=1}^n h_\mathcal{T}(\mu_i)-\mathbb{E}_{\mu\sim\Theta}[h_\mathcal{T}(\mu)]\right|\right]\\
\leq &2\mathbb{E}_{\sigma_i,\mu_i}\left[\sup_{h_\mathcal{T}\in H^{g,\epsilon}_\rho(\mathbb{T}^m)}\left|\frac{1}{n}\sum_{i=1}^n \sigma_i h_\mathcal{T}(\mu_i)\right|\right],
\end{align*}
where $\sigma_i$ are Rademacher random variables, i.e., \begin{align*}
\sigma_i=
\begin{cases}
1, & \text{with probability } \frac{1}{2},\\[4pt]
-1, & \text{with probability } \frac{1}{2},
\end{cases}
\qquad i=1,\dots,n,
\end{align*}
such that $\sigma_i$ is independent of $\sigma_j$ for all $i\neq j$, and are furthermore independent of $\mu_1,...,\mu_n$.
\end{lemma}
\begin{proof}
See Section \ref{subsec:ProofLemmaSymmetrization}.
\end{proof}

We then apply Dudley's integral bound to bound the Rademacher complexity of $H_\rho^{g,\epsilon}(\mathbb{T}^m)$:

\begin{lemma}\label{lemma:dudley}
(Theorem 5.22 in \citet{wainwright2019high}) Let $(\mathcal{F},d)$ be a metric space, and let $\{X_f\}_{f\in \mathcal{F}}$ be a set of zero mean random variables indexed by $\mathcal{F}$. Furthermore, assume that
\begin{equation}\label{eqn:subgaussian process}
\mathbb{E}\left[\exp(\lambda(X_f-X_{f'})\right]\leq \exp\left(\frac{\lambda^2}{2}d(f,f')^2\right)
\end{equation}
for all $f,f'\in \mathcal{F}$ and $\lambda \in \mathbb{R}$. Let $\alpha_0=\sup_{f,f'\in \mathcal{F}}d(f,f')$. Then:
\begin{align*}
\mathbb{E}\!\left[\sup_{f,f' \in \mathcal{F}} \bigl(X_f - X_{f'}\bigr)\right]
\leq
2\,\mathbb{E}\!\left[
\sup_{\substack{f,f' \in \mathcal{F} \\ d(f,f') \leq \delta}}
\bigl(X_f - X_{f'}\bigr)
\right]
+
C \int_{\delta}^{\alpha_0}
\sqrt{\log \mathcal{N}(\tau/4;\mathcal{F},d)}\,d\tau,
\end{align*}
where $C>0$ is a universal constant.
\end{lemma}
Continuing with our sketch, let $\hat{\Theta}^n$ denote the empirical measure supported on $\{\mu_i\}_{i=1}^n$.  We verify (using a similar argument to Example 5.24 in \citet{wainwright2019high}) that the family of zero mean random variables $\hat{H}^{g,\epsilon}_{\rho,n}:=\left\{\frac{1}{\sqrt{n}}\sum_{i=1}^n\sigma_i h(\mu_i) \; | \; h\in H^{g,\epsilon}_\rho(\mathbb{T}^m)\cup(-H^{g,\epsilon}_\rho(\mathbb{T}^m))\right\}$, satisfies (\ref{eqn:subgaussian process}), with $d(h,\tilde{h})=\frac{1}{\sqrt{n}}\|h-\tilde{h}\|_{L^2(\hat{\Theta}^n)}$ for $h,\tilde{h}\in H_\rho^{g,\epsilon}(\mathbb{T}^m)\cup(-H^{g,\epsilon}_\rho(\mathbb{T}^m))$.  We consider this enlarged set of functions, as opposed to just $H_\rho^{g,\epsilon}(\mathbb{T}^m)$ for technical reasons, and note that it only increases covering numbers by at most a factor of two. We then apply Lemma \ref{lemma:dudley} to $\hat{H}^{g,\epsilon}_{\rho,n}$ to obtain an upper bound in terms of the covering numbers of $H^{g,\epsilon}_\rho(\mathbb{T}^m)$. We conclude the argument using Corollary \ref{cor:bounded_diam} and a technical result controlling covering numbers under Lipschitz transformations (Lemma \ref{lemma:covering under lipschitz}) to replace $\mathcal{N}(\tau;H_\rho^{g,\epsilon}(\mathbb{T}^m),L^2(\hat{\Theta}^n))$ in the final bound with $\mathcal{N}(\tau;\mathbb{T}^m,L^2(\rho))$.

We now proceed with the full proof in detail:

\noindent\textbf{Proof of Theorem \ref{thm:samp_complex_dl}:}  Write $h_\mathcal{T}(\cdot):=h_\rho^g(\cdot,\epsilon,\mathcal{T})\in H_\rho^{g,\epsilon}(\mathbb{T}^m)$. Observe that since $\mathcal{T}^{n}\in\argmin_{\mathcal{T}\in \mathbb{T}^m}\frac{1}{n}\sum_{i=1}^n h_\mathcal{T}(\mu_i)$, we have $\frac{1}{n}\sum_{i=1}^n h_{\mathcal{T}^n}(\mu_i)-\frac{1}{n}\sum_{i=1}^nh_{\mathcal{T}^*}(\mu_i)\leq 0$ for any $\mu_1,...,\mu_n$. Hence:
\begin{align*}
& \mathbb{E}_{\mu_i,\mu\sim \Theta}[h_{\mathcal{T}^n}(\mu)-h_{\mathcal{T}^*}(\mu)] \\
&=\mathbb{E}_{\mu_i\sim \Theta}[\mathbb{E}_{\mu\sim\Theta}[h_{\mathcal{T}^n}(\mu)]-\mathbb{E}_{\mu\sim\Theta}[h_{\mathcal{T}^*}(\mu)]]
\\&=
 \mathbb{E}_{\mu_i\sim \Theta}\left[
\mathbb{E}_{\mu\sim\Theta}[h_{\mathcal{T}^n}(\mu)]- \frac{1}{n}\sum_{i=1}^n h_{\mathcal{T}^n}(\mu_i)
\right] +  \mathbb{E}_{\mu_i,\mu\sim \Theta}\left[\frac{1}{n}\sum_{i=1}^n h_{\mathcal{T}^n}(\mu_i)-\frac{1}{n}\sum_{i=1}^n h_{\mathcal{T}^*}(\mu_i)\right] \\
&\qquad+ \mathbb{E}_{\mu_i,\mu\sim \Theta}\left[\frac{1}{n}\sum_{i=1}^n h_{\mathcal{T}^*}(\mu_i)-\mathbb{E}_{\mu\sim\Theta}[h_{\mathcal{T}^*}(\mu)]\right] \\
&\leq \mathbb{E}_{\mu_i\sim \Theta}\left[\mathbb{E}_{\mu\sim\Theta}[h_{\mathcal{T}^n}(\mu)]-\frac{1}{n}\sum_{i=1}^n h_{\mathcal{T}^n}(\mu_i)\right]+ \mathbb{E}_{\mu_i\sim \Theta}\left[\frac{1}{n}\sum_{i=1}^n h_{\mathcal{T}^*}(\mu_i)-\mathbb{E}_{\mu\sim\Theta}[h_{\mathcal{T}^*}(\mu)]\right] \\
&\leq2 \mathbb{E}_{\mu_i\sim \Theta}\left[\sup_{h_{\mathcal{T}}\in H^{g,\epsilon}_\rho(\mathbb{T}^m)}\left|\frac{1}{n}\sum_{i=1}^n h_{\mathcal{T}}(\mu_i)-\mathbb{E}_{\mu\sim\Theta}[h_{\mathcal{T}}(\mu)]\right|\right].
\end{align*}
Applying Lemma \ref{lemma:symmetrization}, we bound the third term as: \begin{align}
&\mathbb{E}_{\mu_i\sim\Theta}\left[\sup_{h_\mathcal{T}\in H^{g,\epsilon}_\rho(\mathbb{T}^m)}\left|\frac{1}{n}\sum_{i=1}^n h_\mathcal{T}(\mu_i)-\mathbb{E}_{\mu\sim\Theta}[h_\mathcal{T}(\mu)]\right|\right]\notag\\
\leq &2\mathbb{E}_{\mu_i\sim\Theta}\left[\mathbb{E}_{\sigma}\left[\sup_{h_\mathcal{T}\in H^{g,\epsilon}_\rho(\mathbb{T}^m)}\left|\frac{1}{n}\sum_{i=1}^n\sigma_ih_\mathcal{T}(\mu_i)\right|\right]\right], \label{eqn:ExpectationToBound}
\end{align}
where $\sigma=(\sigma_1,\sigma_2,...,\sigma_n)$ are i.i.d. Rademacher random variables. Let $-H_\rho^{g,\epsilon}(\mathbb{T}^m):=\{-h_\mathcal{T}\;|\;h_\mathcal{T}\in H_\rho^{g,\epsilon}(\mathbb{T}^m)\}$. Define the set of random variables $\hat{H}_{\rho,n}^{g,\epsilon}:=\{\frac{1}{\sqrt{n}}\sum_{i=1}^n\sigma_i h(\mu_i) \; | \; h\in H^{g,\epsilon}_\rho(\mathbb{T}^m)\cup (-H_\rho^{g,\epsilon}(\mathbb{T}^m))\}$, for any $n\in \mathbb{N}$. Notice that conditioned on $\{\mu_i\}_{i=1}^n$, for any $h\in H_\rho^{g,\epsilon}(\mathbb{T}^m)\cup (-H^{g,\epsilon}_\rho(\mathbb{T}^m))$, the random variable $h_{\sigma,n}:=\frac{1}{\sqrt{n}}\sum_{i=1}^n\sigma_i h(\mu_i)$ is a deterministic function of the random variable $\sigma:=(\sigma_1,\sigma_2,...,\sigma_n).$ We may write the integrand in the above expectation (\ref{eqn:ExpectationToBound}) as:
\begin{align}
\frac{1}{\sqrt{n}}\mathbb{E}_{\sigma}\left[\sup_{h_\mathcal{T}\in H^{g,\epsilon}_\rho(\mathbb{T}^m)}\left|\frac{1}{\sqrt{n}}\sum_{i=1}^n\sigma_ih(\mu_i)\right|\right]&=\frac{1}{\sqrt{n}}\mathbb{E}_{\sigma}\left[\sup_{h_{\sigma,n}\in \hat{H}^{g,\epsilon}_{\rho,n}} |h_{\sigma,n}|\right]\label{eqn:replace fnction class}\\&=\frac{1}{\sqrt{n}}\mathbb{E}_{\sigma}\left[\sup_{h_{\sigma,n}\in \hat{H}^{g,\epsilon}_{\rho,n}}h_{\sigma,n}\right].\label{eqn:remove abs}
\end{align}
The inequality in (\ref{eqn:replace fnction class}) follows from the fact that the set $\left\{\frac{1}{\sqrt{n}}\sum_{i=1}^n\sigma_ih_\mathcal{T}(\mu_i)\;|\;h_\mathcal{T}\in H_\rho^{g,\epsilon}(\mathbb{T}^m)\right\}\subseteq \hat{H}^{g,\epsilon}_{\rho,n}$. The equality in (\ref{eqn:remove abs}) follows from the fact that $\hat{H}^{g,\epsilon}_{\rho,n}=-\hat{H}^{g,\epsilon}_{\rho,n}$.

We claim that the family of random variables $\hat{H}^{g,\epsilon}_{\rho,n}$, with index set $H_\rho^{g,\epsilon}(\mathbb{T}^m)\cup (-H^{g,\epsilon}_\rho(\mathbb{T}^m))$, satisfies (\ref{eqn:subgaussian process}) with the metric induced by the norm $\|\cdot \|_{L^2(\hat{\Theta}^n)}$.  Indeed, we have, for any $h,\tilde{h}\in H^{g,\epsilon}_\rho(\mathbb{T}^m)\cup (-H_\rho^{g,\epsilon}(\mathbb{T}^m))$:
\begin{align*}
&\mathbb{E}_{\sigma}\left[\exp\left(\lambda \left(\frac{1}{\sqrt{n}}\sum_{i=1}^n\sigma_ih(\mu_i)-\frac{1}{\sqrt{n}}\sum_{i=1}^n\sigma_i\tilde{h}(\mu_i)\right)\right)\right]\\&=\prod_{i=1}^n\mathbb{E}_{\sigma_i}\left[\exp\left(\frac{\lambda\sigma_i}{\sqrt{n}}(h(\mu_i)-\tilde{h}(\mu_i))\right)\right]\\&\leq \prod_{i=1}^n\left[\exp\left(\frac{\lambda^2}{2n}(h(\mu_i)-\tilde{h}(\mu_i))^2\right)\right]\\&=\exp\left(\frac{\lambda^2}{2n}\sum_{i=1}^n (h(\mu_i)-\tilde{h}(\mu_i))^2\right)\\&=\exp\left(\frac{\lambda^2}{2}\|h-\tilde{h}\|^2_{L^2(\hat{\Theta}^n)}\right)
\end{align*}
where we applied the bound $\mathbb{E}_{\sigma_i}\left[\exp\left(\sigma_i a\right)\right]\leq \exp(a^2/2)$ for $a\geq 0$, establishing (\ref{eqn:subgaussian process}) with $\mathcal{F}=H^{g,\epsilon}_\rho(\mathbb{T}^m)\cup(-H^{g,\epsilon}_\rho(\mathbb{T}^m)).$ Hence, letting
\[\alpha_0':=\sup_{h,\tilde{h}\in H^{g,\epsilon}_\rho(\mathbb{T}^m)\cup(-H^{g,\epsilon}_\rho(\mathbb{T}^m))}\|h-\tilde{h}\|_{L^2(\hat{\Theta}^n)},\]  we may apply Lemma \ref{lemma:dudley} to bound, for any $\delta>0$:\begin{align*}
&\frac{2}{\sqrt{n}}\mathbb{E}_{\sigma}\left[\sup_{h_{\sigma,n}\in \hat{H}^{g,\epsilon}_{\rho,n}}h_{\sigma,n}\right]\\&\leq \frac{4}{\sqrt{n}}\mathbb{E}_\sigma\left[\sup_{\substack{h,\tilde{h}\in H_{\rho}^{g,\epsilon}(\mathbb{T}^m)\cup(-H_{\rho}^{g,\epsilon}(\mathbb{T}^m))\\ \|h-\tilde{h}\|_{L^2(\hat{\Theta}^n)} \leq \delta}}h_{\sigma,n}-\tilde{h}_{\sigma,n}\right]\\&+\frac{C}{\sqrt{n}}\int_\delta^{\alpha_0'}\sqrt{\log \mathcal{N}(\tau/4;H_\rho^{g,\epsilon}(\mathbb{T}^m)\cup(-H^{g,\epsilon}_\rho(\mathbb{T}^m)),L^2(\hat{\Theta}^n)) }d\tau \tag{Lemma \ref{lemma:dudley}}\\
&\leq 4\mathbb{E}_\sigma\left[\sup_{\substack{h,\tilde{h}\in H_{\rho}^{g,\epsilon}(\mathbb{T}^m)\cup(-H_{\rho}^{g,\epsilon}(\mathbb{T}^m))\\ \|h-\tilde{h}\|_{L^2(\hat{\Theta}^n)} \leq \delta}}\frac{h_{\sigma,n}-\tilde{h}_{\sigma,n}}{\sqrt{n}}\right]\\&+ \frac{C}{\sqrt{n}}\int_\delta^{\alpha_0'}\sqrt{\log \mathcal{N}(\tau/4;H_\rho^{g,\epsilon}(\mathbb{T}^m),L^2(\hat{\Theta}^n)) +\log(2)}\;d\tau,
\end{align*}
where we used

\[\mathcal{N}(\tau;H^{g,\epsilon}_\rho(\mathbb{T}^m)\cup(-H^{g,\epsilon}_\rho(\mathbb{T}^m)),L^2(\hat{\Theta}^n))\leq 2\mathcal{N}(\tau;H^{g,\epsilon}_\rho(\mathbb{T}^m),L^2(\hat{\Theta}^n)).\]

To control the first term, we bound:
\begin{align*}
&\frac{1}{\sqrt{n}}\left|\frac{1}{\sqrt{n}}\sum_{i=1}^n\sigma_i h(\mu_i)-\frac{1}{\sqrt{n}}\sum_{i=1}^n\sigma_i \tilde{h}(\mu_i)\right|\\
=&\frac{1}{\sqrt{n}}\left|\frac{1}{\sqrt{n}}\sum_{i=1}^n\sigma_i(h(\mu_i)-\tilde{h}(\mu_i))\right|\\
\leq &\frac{1}{\sqrt{n}}\left|\frac{1}{n}\sum_{i=1}^n\sigma_i^2\right|^{1/2}\left|\sum_{i=1}^n (h(\mu_i)-\tilde{h}(\mu_i))^2\right|^{1/2}\tag{Cauchy-Schwarz}\\
=&\left|\frac{1}{n}\sum_{i=1}^n(h(\mu_i)-\tilde{h}(\mu_i))^2\right|^{1/2}\tag{$\sigma_i^{2}=1$ a.s. }\\
=&\|h-\tilde{h}\|_{L^2(\hat{\Theta}^n)}.
\end{align*}
Hence we have:\begin{align*}
& 4\mathbb{E}_\sigma\left[\sup_{\substack{h,\tilde{h}\in H_{\rho}^{g,\epsilon}(\mathbb{T}^m)\cup(-H_{\rho}^{g,\epsilon}(\mathbb{T}^m))\\ \|h-\tilde{h}\|_{L^2(\hat{\Theta}^n)} \leq \delta}}\frac{h_{\sigma,n}-\tilde{h}_{\sigma,n}}{\sqrt{n}}\right]\\&\leq 4\mathbb{E}_\sigma\left[\sup_{\substack{h,\tilde{h}\in H_{\rho}^{g,\epsilon}(\mathbb{T}^m)\cup(-H_{\rho}^{g,\epsilon}(\mathbb{T}^m))\\ \|h-\tilde{h}\|_{L^2(\hat{\Theta}^n)} \leq \delta}}\|h-\tilde{h}\|_{L^2(\hat{\Theta}^n)}\right]\\&=4\delta.
\end{align*}
Our upper bound now reads:
\begin{align}
&4\delta+\frac{C}{\sqrt{n}}\int_\delta^{\alpha_0'}\sqrt{\log \mathcal{N}(\tau/4;H_\rho^{g,\epsilon}(\mathbb{T}^m),L^2(\hat{\Theta}^n))+\log(2)}\;d\tau\notag\\
\leq &4\delta +\frac{C}{\sqrt{n}}\int_\delta^{\alpha_0'}\sqrt{\log \mathcal{N}(\tau/4;H_\rho^{g,\epsilon}(\mathbb{T}^m),L^2(\hat{\Theta}^n))}\;d\tau +\frac{C}{\sqrt{n}}\int_\delta^{\alpha_0'}1d\tau,\label{eqn:post dudley}
\end{align}
where we factored out $\log(2)$ and collapsed it into the constant.

Towards bounding $\mathcal{N}(\tau;H^{g,\epsilon}_\rho(\mathbb{T}^m),L^2(\hat{\Theta}^n))$, define $d_{2,m}(\mathcal{T},\mathcal{T}'):=\max_{1\leq j \leq m}\|T_j-T_j'\|_{L^2(\rho)}.$ We establish that the mapping $\mathcal{T}\rightarrow h_\mathcal{T}$ from $(\mathbb{T}^m,d_{2,m})$  to the (random) metric space $(H^{g,\epsilon}_\rho(\mathbb{T}^m),L^2(\hat{\Theta}^n))$ is almost-surely Lipschitz continuous. By Corollary \ref{cor:bounded_diam}, we have:
\begin{align}
&\|h_\mathcal{T}-h_{\tilde{\mathcal{T}}}\|_{L^2(\hat{\Theta}^n)}\notag\\
=&\sqrt{\frac{1}{n}\sum_{i=1}^n (h_\mathcal{T}(\mu_i)-h_{\tilde{\mathcal{T}}}(\mu_i))^2}\notag\\
 \leq &\sqrt{\frac{1}{n}\sum_{i=1}^n\frac{M_2(\mu_i)^4+M_2(\mu_i)^{3}}{2\min\{c^2,c^4\}} \left(\max_{1\leq j\leq m}\|T_j-\tilde{T}_j\|_{L^2(\rho)}\right)^2}\tag{Corollary \ref{cor:bounded_diam}}\\
=&\frac{\max_{1\leq j\leq m}\|T_j-\tilde{T}_j\|_{L^2(\rho)}}{\sqrt{2}\min\{c,c^2\}}\sqrt{\frac{1}{n}\sum_{i=1}^n(M_2(\mu_i)^4+M_2(\mu_i)^{3})}\notag\\
\leq &\frac{\max_{1\leq j\leq m}\|T_j-\tilde{T}_j\|_{L^2(\rho)}}{\sqrt{2}\min\{c,c^2\}}\sqrt{\frac{1}{n}\sum_{i=1}^n(\vartheta^4+\vartheta^3)}\tag{$\mu_i\in \mathcal{M}_2(\vartheta)$, almost-surely}\\
=&\frac{\max_{1\leq j\leq m}\|T_j-\tilde{T}_j\|_{L^2(\rho)}}{\sqrt{2}\min\{c,c^2\}}\sqrt{\vartheta^4+\vartheta^3},\label{eqn:lipschitz on hg}
\end{align}
almost-surely, which establishes the desired Lipschitz property with constant $\frac{\sqrt{\vartheta^{4}+\vartheta^{3}}}{\sqrt{2}\min\{c,c^{2}\}}$ . We apply (\ref{eqn:lipschitz on hg}) and Lemma \ref{lemma:covering under lipschitz} to conclude:
\begin{equation*}
\mathcal{N}(\tau/4;H^{g,\epsilon}_\rho(\mathbb{T}^m),L^2(\hat{\Theta}^n))\leq  \mathcal{N}\left(\frac{\tau}{4} \left(\frac{\sqrt{\vartheta^4+\vartheta^3}}{\sqrt{2}\min\{c,c^2\}}\right)^{-1};\mathbb{T}^m,d_{2,m}\right)
\end{equation*}
almost-surely.

To conclude the overall argument, let $D_\tau=\{x^i\}_{i=1}^{N_\tau}$ denote a $\tau$-net for an arbitrary metric space $(\mathcal{X},d)$. Then for any $(x_1,x_2,...,x_m)\in \mathcal{X}^m$, there exists $(x_1^{i_1},x_2^{i_2},...,x^{i_m}_m)\in D_\tau^m$ such that:
\begin{align*}
&\max_{1\leq j\leq m}d(x_j,x_j^{i_j})\leq \max_{1\leq j\leq m}\tau=\tau
\end{align*}
and hence $D_\tau^m$ is a $\tau$-net for $\mathcal{X}^m$ equipped with this metric, which has size $N_\tau^m$. Applying this to our situation, we have:
\begin{equation*}
\mathcal{N}\left(\frac{\tau}{4} \left(\frac{\sqrt{\vartheta^4+\vartheta^3}}{\sqrt{2}\min\{c,c^2\}}\right)^{-1};\mathbb{T}^m,d_{2,m}\right)\leq \mathcal{N}\left(\frac{\tau}{4} \left(\frac{\sqrt{\vartheta^4+\vartheta^3}}{\sqrt{2}\min\{c,c^2\}}\right)^{-1};\mathbb{T},L^2(\rho)\right)^m.
\end{equation*}
Finally, we bound $\alpha_0'$ in (\ref{eqn:post dudley}). For any $h,\tilde{h}\in H_{\rho}^{g,\epsilon}(\mathbb{T}^m)\cup(-H_{\rho}^{g,\epsilon}(\mathbb{T}^m))$, there exist tuples $\mathcal{T},\tilde{\mathcal{T}}\in \mathbb{T}^m$ such that $h\in \{h_\mathcal{T},-h_\mathcal{T}\}$ and $\tilde{h}\in \{h_\mathcal{T},-h_{\tilde{\mathcal{T}}}\}$.
\begin{align*}
\|h-\tilde{h}\|_{L^2(\hat{\Theta}^n)}&\leq \|h_\mathcal{T}\|_{L^2(\hat{\Theta}^n)}+\|h_{\tilde{\mathcal{T}}}\|_{L^2(\hat{\Theta}^n)}\\&\leq 2\vartheta,
\end{align*}
where the bound follows from
\begin{align*}
\|h_\mathcal{T}\|_{L^2(\hat{\Theta}^n)}&=\sqrt{\frac{1}{n}\sum_{i=1}^nh_{\mathcal{T}}(\mu_i)^2}\\&=\sqrt{\frac{1}{n}\sum_{i=1}^n\left(\inf_{\lambda\in \mathbb{R}^m_{\geq 0}}\|T^\epsilon_{\rho\rightarrow \mu_i}-\sum_{j=1}^m\lambda_j T_j\|^2_{L^2(\rho)}+g(\lambda)\right)^2}\\&\leq \sqrt{\frac{1}{n}\sum_{i=1}^n\|T^\epsilon_{\rho\rightarrow \mu_i}\|^4_{L^2(\rho)}}\tag{$\lambda=0$}\\&\leq \sqrt{\frac{1}{n}\sum_{i=1}^nM_2(\mu_i)^2}\tag{Lemma \ref{lem:CompressionUnderEOT}}\\&\leq \vartheta.
\end{align*}
Taking the supremum over $g,\tilde{h}$, we have $\alpha_{0}'\le 2\vartheta$.
Let $C_{\vartheta,c}:=\frac{\sqrt{(\vartheta^4+\vartheta^3)}}{\sqrt{2}\min\{c,c^2\}}$. Combining (\ref{eqn:post dudley}) with the reduction to the uniform deviation bound above, we have established that: \begin{align*}
&\mathbb{E}_{\mu_i,\mu\sim \Theta}[h_{\mathcal{T}^n}(\mu)-h_{\mathcal{T}^*}(\mu)]\\
&\leq 8\delta+ \frac{C}{\sqrt{n}}\int_\delta^{2\vartheta}\sqrt{m\log\left(\mathcal{N}\left(\tau (4C_{\vartheta,c})^{-1};\mathbb{T},L^2(\rho)\right)\right)}\;d\tau
+ \frac{C}{\sqrt{n}}\int_\delta^{2\vartheta}1d\tau\\
&= 8\delta+\frac{4C\cdot C_{\vartheta,c}}{\sqrt{n}}
\int_{\delta/(4C_{\vartheta,c})}^{\vartheta/(2C_{\vartheta,c})}
\sqrt{m\log\left(\mathcal{N}\left(u;\mathbb{T},L^2(\rho)\right)\right)}du
+ \frac{C}{\sqrt{n}}\left(2\vartheta-\delta\right)
\tag{$u:=(4C_{\vartheta,c})^{-1}\tau$}\\
&\leq 8\delta+\frac{4C \cdot C_{\vartheta,c}}{\sqrt{n}}
\int_{\delta/(4C_{\vartheta,c})}^{\vartheta/(2C_{\vartheta,c})}
\sqrt{m\log\left(\mathcal{N}\left(u;\mathbb{T},L^2(\rho)\right)\right)}du
+ \frac{2C\vartheta}{\sqrt{n}}.
\end{align*}We conclude by collapsing constant factors into $C$.
\qed

\begin{remark}
We note that the bounds in Theorem \ref{thm:samp_complex_dl} degrade as $c\rightarrow 0$. This is due to the $\frac{1}{\min\{c,c^2\}}$-dependence in Corollary \ref{cor:bounded_diam}, which is in turn due to the estimate in Lemma \ref{lemma:b/c}. To expand on why our argument fails for $c=0$, for any $y\in \mathbb{R}^d$ with $\|y\|=1$, define the map $T_y(x)=y$ for all $x\in \mathbb{R}^d$. Then for any absolutely continuous $\rho$ and $t>0$ we have:
\begin{align*}
\left\|T_{\rho\rightarrow \delta_0}-\left(\frac{t}{2}T_{y}+\frac{t}{2}T_{-y}\right)\right\|_{L^2(\rho)}&=0.
\end{align*}
We conclude that $\{(t/2,t/2)|t>0\}$ is a one-dimensional set of minimizers realizing $h_\rho^g(\delta_0,0,(T_y,T_{-y}))$, with $g=0$. Hence we cannot obtain a data dependent bound on $\|\lambda^*\|_1$, even in the case when $\mathbb{T}=\{T_y,T_{-y}\}$ and $M_2(\delta_0)=0$, which renders Lemma \ref{lemma:lipschitz objective} vacuous. To patch our argument, an alternative mechanism would be required to establish stability of $h_\rho^g(\mu,\epsilon ,\mathcal{T})$ for non-coercive costs, such as restricting $\lambda$ to live in a compact subset of $\mathbb{R}^m_{\geq 0}$.
\end{remark}

\begin{remark}
    Theorem \ref{thm:samp_complex_dl} establishes that the generalization error for LWDL-OT and LWDL-EOT can be controlled by quantities depending on the moments of the data, the strength of regularization and the covering number of the set of candidate dictionary atoms. This is related to prior work in \citet{vainsencher2011sample} and \citet{gribonval2015sample}, which establish similar bounds for dictionary learning in both finite and infinite dimensions. Theorem 14 in \citet{vainsencher2011sample} is particularly relevant, which establishes generalization error bounds for dictionary learning of data living in a metric space $\mathcal{R}$, embedded into a Hilbert space $\mathcal{H}$ via a feature map. Here, the error bounds scale according to the covering numbers of $\mathcal{R}$, which are assumed to be of the order $(\frac{C}{\epsilon})^n$ for some $n$, and the H\"older continuity of the feature map, among other factors. This is a similar setting to our result, as our data consists of probability measure valued data in the Wasserstein space, which is embedded in a Hilbert space $L^2(\rho)$ via the transport map $T_{\rho\rightarrow \mu}$. However, we note that Theorem 14 in \citet{vainsencher2011sample} does not apply to our applications of Theorem \ref{thm:samp_complex_dl}--for instance, $\mathcal{P}_2(\mathcal{M}_2(\vartheta))$ strongly violates the covering number assumption.
\end{remark}
\subsection{Covering Number Bounds}\label{sec:covering number bounds}
In this section, we will discuss applications of Theorem \ref{thm:samp_complex_dl} for different choices of dictionaries and data distributions $\Theta.$ As shown by (\ref{eqn:sample linear bound}), it is crucial to obtain explicit bounds for the covering number of the dictionaries, and hence in this section we discuss some standard results surrounding covering number bounds for common function classes.
The standard mechanism for obtaining covering number bounds on function classes is via uniform regularity bounds. The following result is an archetypal example of such a result. Recall our notation $\mathcal{C}^\alpha_M(\Omega)=\{f:\Omega\rightarrow \mathbb{R}| f\in \mathcal{C}^{\alpha}(\Omega) \; \text{and} \; \max_{|\beta|\leq \alpha}\sup_{x\in \mathring{\Omega}}|D^\beta f(x)|\leq M,\; \beta\in \mathbb{N}^d\}$, where $\Omega\subset \mathbb{R}^d$ is open, $\alpha\geq 0$, $\beta$ is a multi-index and $M\geq 0$. Then:

\begin{theorem}\label{theorem:wellner non compact}
(Corollary 2.7.4 in \citet{wellner2013weak}) Let $\rho\in \mathcal{P}(\mathbb{R}^d)$ and let $p\geq 1$. Let $\mathbb{R}^d=\cup_{j=1}^\infty I_j$ be a partition of $\mathbb{R}^d$ into bounded, convex sets with nonempty interior, and let $\mathcal{F}\subset L^p(\rho)$ be a family of functions such that the restrictions $\mathcal{F}|_{I_j}$ belong to $\mathcal{C}^\alpha_{M_j}(I_j)$ for every $j$, where $M_j\geq 0$ for all $j$. Let $Q\geq d/\alpha$. Then there exists a constant $K_{\alpha,Q,p,d}$, depending only on $\alpha,Q,p,d$ such that, for any $\rho\in \mathcal{P}(\mathbb{R}^d)$:
\begin{equation*}
\log \left(\mathcal{N}(\tau;\mathcal{F},L^p(\rho))\right)\leq  K_{\alpha,Q,p,d}\cdot\left(\sum_{j=1}^\infty \mathcal{L}(I_j^1)^{\frac{p}{Q+p}} M_j^{\frac{Qp}{Q+p}}\rho(I_j)^{\frac{Q}{Q+p}}\right)^{\frac{Q+p}{p}}\cdot\tau^{-Q}
\end{equation*}
where $\mathcal{L}$ denotes the Lebesgue measure, and for any $I\subset \mathbb{R}^d$, the set $I^1:=\{x\in \mathbb{R}^d|\exists y\in I\text{ s.t. } \|x-y\|\leq 1\}$.
\end{theorem}
\begin{remark}
Corollary 2.7.4 in \citet{wellner2013weak} is stated for bracketing numbers $\mathcal{N}_{[]}$ as opposed to covering numbers $\mathcal{N}$. It is easy to show that for any normed vector space $\mathcal{X}$ we have $\mathcal{N}(\tau;\mathcal{X},\|\cdot\|)\leq \mathcal{N}_{[]}(2\tau;\mathcal{X},\|\cdot\|)$, and hence the bound applies to covering numbers after modifying $K_{\alpha,Q,p,d}.$
\end{remark}

In order for the sum in Theorem \ref{theorem:wellner non compact} to converge, it is necessary to have strong concentration bounds on $\rho$. The following corollary, extracted from the proof of Proposition 3 in \citet{mena2019statistical}, gives an example when $\rho$ is subgaussian, which will be relevant in our analysis of covering numbers for EOT maps.

\begin{corollary}\label{cor:regular function class subgaussian}
Let $\rho\in \mathcal{G}_\sigma(\mathbb{R}^d)$, let $s\in \mathbb{N}$ and let $q>0$. Let $\mathcal{F}\subseteq L^2(\rho)\cap C^s (\mathbb{R}^d)$, such that there exists a constant $C$ with
\begin{equation*}
\|D^\alpha f(x)\|\leq C (1+\|x\|)^{s+q}
\end{equation*}
for all $\alpha \in \mathbb{N}^d$ with $|\alpha|\leq s$ and $f\in \mathcal{F}$. Then:
\begin{align*}
\log \mathcal{N}(\tau;\mathcal{F},L^2(\rho))&\leq C_{*}(1+\sigma^{s+q})^{d/s}(1+\sigma^d)\tau^{-d/s}.
\end{align*}

As a special case, if $s = \lceil \frac{d}{2}\rceil +1$, then:
\begin{equation*}
\log \mathcal{N}(\tau; \mathcal{F},L^2(\rho))\leq  C_{d,q}(1+\sigma^{2q+d+2})(1+\sigma^d)\tau^{-2}
\end{equation*}
where $C_{d,q}\geq 0$ depends only on $d$ and $q$.
\end{corollary}
\begin{proof}
See Appendix \ref{subsec:ProofCor3.10}.
\end{proof}

We now investigate covering number bounds for maps $\mathbb{T}_\rho(\mathcal{U}):=\{T^\epsilon_{\rho\rightarrow \nu}|\nu\in \mathcal{U}\}$ where $\epsilon\geq 0$, $\rho\in \mathcal{P}_2(\Omega)$ and $\mathcal{U}\subseteq \mathcal{P}_2(\Omega)$, for some $\Omega\subseteq \mathbb{R}^d$.

\subsubsection{Covering Numbers of OT Maps}

In the absence of smoothness, one may still get covering number bounds for structured classes of functions. An important example is given by convex functions. For any convex set $\Omega\subset \mathbb{R}^d$ and $B\geq 0$, let $\mathcal{C}_{cvx}(\Omega,B)$ be the set of convex functions on $\Omega$ with $\sup_{x\in \Omega}|u(x)|\leq B$ for all $u\in \mathcal{C}_{cvx}(\Omega,B).$

\begin{theorem}\label{thm:convex covering}
(Theorem 3.1 in \citet{guntuboyina2013covering}) Let $\Omega=\prod_{j=1}^d I_j\subset \mathbb{R}^d$, where $I_j\subset \mathbb{R}$ is a closed interval with $|I_j|=D$. Let $U$ be the uniform measure on $\Omega.$ Then, for any $p\geq 1$:
\begin{equation*}
\log \mathcal{N}(\tau; \mathcal{C}_{cvx}(\Omega,B),L^p(U))\leq  K_{D,d,p,B} \tau^{-d/2}
\end{equation*}
for all $0<\tau<\tau_{d,p}$ for some constant depending on $d$ and $p$, where the constant $K_{D,d,p,B}$ depends only on $D,d, p $ and $B.$
\end{theorem}

This result allows us to derive covering numbers for OT maps.

\begin{proposition}\label{prop:OT map covering global}
Let $\Omega=\prod_{j=1}^d I_j\subset \mathbb{R}^d$, where $I_j\subset \mathbb{R}$ is a closed interval with $|I_j|=D$. Let $U$ be the uniform measure on $\Omega$, and define $$\mathbb{T}_U(\mathcal{P}(\Omega)):=\{T_{U\rightarrow \nu}|\nu\in \mathcal{P}(\Omega)\}\subset L^2(U).$$ Then \begin{equation*} \log \mathcal{N}(\tau;\mathbb{T}_U(\mathcal{P}(\Omega)),L^2(U))\leq C^*_{D,d}\tau^{-3d/2}\end{equation*} for all $\tau>0$ small enough, where $C^*_{D,d}$ depends only on $D$ and $d.$
\end{proposition}
Before giving the proof of Proposition \ref{prop:OT map covering global}, we mention a corollary.

\begin{corollary}\label{cor:bounded density covering}
Let $\Omega$ be defined as in Proposition \ref{prop:OT map covering global}. Let $\rho\in \mathcal{P}(\Omega)$ be absolutely continuous such that $\text{supp}(\rho)=\Omega$ and $0<\alpha<\frac{d\rho}{dx}(x)<\beta<\infty$ for Lebesgue-a.e. $x\in \Omega$. Define $$\mathbb{T}_\rho(\mathcal{P}(\Omega)):=\{T_{\rho\rightarrow \nu}|\nu\in \mathcal{P}(\Omega)\}\subset L^2(\rho).$$ Then \begin{equation*} \log \mathcal{N}(\tau;\mathbb{T}_\rho(\mathcal{P}(\Omega)),L^2(\rho))\leq C^*_{D,d}((\mathcal{L}(\Omega)\beta)^{-1/2} \tau)^{-3d/2}\end{equation*} for all $\tau>0$ small enough, where $C^*_{D,d}$ depends only on $D$ and $d.$
\end{corollary}

\begin{proof}  See Appendix \ref{subsec:ProofCorC.13}.

\end{proof}

\begin{remark}  We make the following comments about these results.

\begin{enumerate}

\item We may explicitly evaluate $\mathcal{L}(\Omega)=D^d$ and absorb this factor into the constant $C_{D,d}^*$. We note that $\mathcal{L}(\Omega)=\left(\frac{dU}{dx}\right)^{-1}$ and hence the factor $\mathcal{L}(\Omega)\beta$ is the ratio of the upper bound on $\frac{d\rho}{dx}$ and the density of $\frac{dU}{dx}.$

\item We constrain the domain solely in order to apply Theorem \ref{thm:convex covering}, which is established on a product of intervals. We expect the same proof should hold for any compact, convex $\Omega\subset \mathbb{R}^d$ with rectifiable boundary.

\item Although Proposition \ref{prop:OT map covering global} and Corollary \ref{cor:bounded density covering} are very natural from the perspective of statistics in Wasserstein space, we were not able to find these exact results in the literature.

\end{enumerate}
\end{remark}

Towards proving Proposition \ref{prop:OT map covering global}, we will need auxiliary results.  We recall (Chapter 10, \citet{maggi2012sets}) that $\Omega\subseteq \mathbb{R}^d$ has \emph{rectifiable boundary} if there exists a countable set of Lipschitz maps $\{q_i\}_{i\in\mathbb{N}}$ such that $\mathcal{H}^{d-1}(\partial \Omega\setminus \cup_{i\in \mathbb{N}}q_i (\mathbb{R}^{d-1}))=0$,
where $\mathcal{H}^{d-1}$ denotes the $(d-1)$-dimensional Hausdorff measure.
\begin{example}
Let $\Omega$ be a convex polytope. We may write $\partial \Omega = \cup_{i=1}^R F_i$, where each $F_i\subset \mathbb{R}^d$ is a $(d-1)$-dimensional convex polytope. Each $F_i$ is the intersection of a $(d-1)$-dimensional hyperplane $H_i\cong \mathbb{R}^{d-1}$ with $\Omega$, and since $F_i$ is convex we can define a 1-Lipschitz retraction $q_i$ from $H_i$ to $F_i$. Then $\partial \Omega\subset \cup_{i=1}^R q_i(H_i)$ and hence $\mathcal{H}^{d-1}(\partial \Omega\setminus \cup_{i=1}^R q_i(H_i))=0$. Hence $\Omega$ is rectifiable.
\end{example}
The above example shows that $\Omega$ in Proposition \ref{prop:OT map covering global} is rectifiable.
\begin{proposition}\label{prop:reverse poincare}
(Proposition 4.1 in \citet{delalande2023quantitative}) Let $\Omega$ be the closure of a bounded, open set with rectifiable boundary. Let $\phi,\psi:\Omega\rightarrow \mathbb{R}$ be convex and $L$-Lipschitz on any line segment contained in $\Omega$. Let $U$ be the uniform measure on $\Omega$. Then:
\begin{equation*}
\|\nabla \phi-\nabla\psi\|^2_{L^2(U)}\leq C_{\Omega,d}L^{4/3}\|\phi-\psi\|_{L^2(U)}^{2/3}
\end{equation*}
where $C_{\Omega,d}$ is a constant depending on $\Omega$ and $d$.
\end{proposition}

Finally, we will need the following regularity result for Kantorovich potentials on a bounded set.
\begin{lemma}\label{lemma:kp lipschitz}
(Corollary 3.19 in \citet{ambrosio2021lectures}) Let $\Omega$ be bounded and let $\mu,\nu\in \mathcal{P}(\Omega)$. Then there exist $(f_{\mu\rightarrow \nu},g_{\mu\rightarrow \nu})$ which solve (\ref{eqn:kantorovich duality}) which are bounded and Lipschitz on $\Omega$.
\end{lemma}
\begin{remark}
We have two remarks on Lemma \ref{lemma:kp lipschitz}:
\begin{enumerate}
\item The result in \citet{ambrosio2021lectures} is stated more generally for dual solutions to the optimal transport problem with general Lipschitz and bounded cost function. Since $\Omega$ is bounded, $\frac{1}{2}\|x-y\|^2$ is Lipschitz and bounded, and hence the result applies.
\item The Kantorovich potentials are a priori defined only as elements of $L^1(\mu)$ and $L^1(\nu)$, whereas the potentials in Lemma \ref{lemma:kp lipschitz} are defined on all of $\Omega$.  To make sense of this, we note that the supremum in (\ref{eqn:kantorovich duality}) can be taken over bounded and Lipschitz functions (see Theorem 3.1 in \citet{ambrosio2021lectures}). Lemma \ref{lemma:kp lipschitz} is proved by constructing $(f_{\mu\rightarrow \nu},g_{\mu\rightarrow \nu})$ which are bounded and Lipschitz functions on $\Omega$ and satisfying $W_2(\mu,\nu)=\int f_{\mu\rightarrow \nu}d\mu+\int g_{\mu\rightarrow \nu}d\nu$, which together with formulating the dual problem over Lipschitz functions implies the result.
\end{enumerate}
\end{remark}

\textbf{Proof of Proposition \ref{prop:OT map covering global}:} Without loss of generality, suppose that $0\in \Omega$. By Theorem \ref{thm:brenier}, each map $T_{U\rightarrow \nu}\in\mathbb{T}_U(\mathcal{P}(\Omega))$ is the gradient of a lower semicontinuous convex function $\phi_{U\rightarrow \nu}$, which is unique $U$-almost everywhere, and differentiable Lebesgue almost everywhere on $\Omega$. Let $\overline{\mathbb{T}}^0$ denote the set of all such functions such that $\phi_{U\rightarrow \nu}(0)=0$. By Lemma \ref{lemma:kp lipschitz}, there exists a Kantorovich potential $f_{U\rightarrow \nu}$ which is bounded and Lipschitz on $\Omega$. There is a $U$-almost everywhere unique version of $f_{U\rightarrow \nu}$ with $f_{U\rightarrow \nu}(0)=0$, and since $f_{U\rightarrow \nu}$ is Lipschitz and $\text{supp}(U)=\Omega$, there is a unique $f_{U\rightarrow \nu}:\Omega\rightarrow \mathbb{R}$ for each $\nu$. By (\ref{eqn:BP KP relation 1}), $\phi_{U\rightarrow \nu}= \frac{\|\cdot\|^2}{2}-f_{U\rightarrow \nu}$, and hence $\phi_{U\rightarrow \nu}\in \overline{\mathbb{T}}^0$ is also uniquely determined by $\nu$. By $|\Omega|<\infty$ and Lemma \ref{lemma:kp lipschitz}, we have that $\phi_{U\rightarrow \nu}$ is also bounded and Lipschitz on $\Omega$. To obtain an explicit bound and Lipschitz constant, note that by Brenier's theorem, we have $\overline{\nabla\phi_{U\rightarrow \nu}(\Omega)}\subseteq\Omega$, and hence we have $\|\nabla \phi_{U\rightarrow \nu}(x)\|\leq |\Omega|$ exists for $U$-almost every $x\in \Omega$. Let $x,y\in \Omega$, and let $\{x_n\}_{n=1}^\infty \subset\Omega$ be such that $\nabla \phi_{U\rightarrow \nu}(x_n)$ exists for all $n\in \mathbb{N}$ and $\lim_{n\rightarrow \infty}x_n=x.$ Then:
\begin{align*}
\phi_{U\rightarrow \nu}(x_n)-\phi_{U\rightarrow \nu}(y)&\leq \langle \nabla \phi_{U\rightarrow \nu}(x_n),x_n-y\rangle \\&\leq |\Omega|\|x_n-y\|
\end{align*}
by Cauchy-Schwarz for all $n\in \mathbb{N}$. Since $\phi_{U\rightarrow \nu}$ is continuous, we have $$\lim_{n\rightarrow \infty}\phi_{U\rightarrow \nu}(x_n)-\phi_{U\rightarrow \nu}(y)=\phi_{U\rightarrow \nu}(x)-\phi_{U\rightarrow \nu}(y),$$
and hence $\phi_{U\rightarrow \nu}(x)-\phi_{U\rightarrow \nu}(y)\leq \limsup_{n\rightarrow \infty}|\Omega|\|x_n-y\|=|\Omega|\|x-y\|.$ By a symmetric argument applied to an approximating sequence $\lim_{n\rightarrow \infty} y_n=y$, we conclude that $\|\phi_{U\rightarrow \nu}(x)-\phi_{U\rightarrow \nu}(y)\|\leq |\Omega|\|x-y\|$ for all $x,y\in\Omega.$ Letting $y=0$, we also have that $|\phi_{U\rightarrow \nu}(x)|\leq |\Omega|\|x\|\leq |\Omega|^2$ for all $x\in  \Omega$. Since $\Omega=\prod_{i=1}^dI_j$ a product of closed intervals with $|I_j|=D$ for all $j$, we may explicitly calculate $|\Omega|=\sqrt{d}D$.

We may now apply Proposition \ref{prop:reverse poincare}, with $L=\sqrt{d}D$ to get:
\begin{align*}
\|\nabla \phi_{U\rightarrow \nu}-\nabla \phi_{U\rightarrow \tilde{\nu}}\|_{L^2(U)}\leq C_{D,d}\|\phi_{U\rightarrow \nu}-\phi_{U\rightarrow \tilde{\nu}}\|_{L^2(U)}^{1/3}.
\end{align*}
for a constant $C_{D,d}$ depending only on $D$ and $d.$ Hence the map $\phi_{U\rightarrow \nu}\rightarrow \nabla \phi_{U\rightarrow \nu}$ is $\frac{1}{3}$-H{\"o}lder continuous, so if $\|\phi_{U\rightarrow \nu}-\phi_{U\rightarrow \tilde{\nu}}\|_{L^2(U)}\leq \tau^3/C_{D,d}$, then $\|\nabla \phi_{U\rightarrow \nu}-\nabla \phi_{U\rightarrow \tilde{\nu}}\|_{L^2(U)}\leq \tau$, and hence:
\begin{align*}
\log \mathcal{N}(\tau ;\mathbb{T}_U(\mathcal{P}(\Omega)),L^2(U))&\leq \log \mathcal{N}\left(\frac{\tau^3}{C_{D,d}};\overline{\mathbb{T}}^0,L^2(U)\right)\\&\leq \log \mathcal{N}\left(\frac{\tau^3}{C_{D,d}};\mathcal{C}_{cvx}(\Omega,dD^2),L^2(U)\right)
\end{align*}
Applying Theorem \ref{thm:convex covering} with $p=2$ and $B=dD^2$ to the right hand side and combining constants, we may upper bound the right hand side by $C^*_{D,d}\tau^{-3d/2}$. \qed

\subsubsection{Covering Numbers for EOT Maps}
We now turn to entropic map dictionaries, which enjoy much better quantitative regularity properties than general optimal transport maps. The following result is a restatement of Lemma 3 in \citet{mena2019statistical}:

\begin{lemma}\label{lemma:mena lemma 3}
(Lemma 3 in \citet{mena2019statistical}) Let $\mu,\nu\in \mathcal{G}_\sigma(\mathbb{R}^d)$, let $\epsilon>0$ and let $\alpha\in \mathbb{N}^d$. Then \[
|D^\alpha f^\epsilon_{\mu\rightarrow \nu}(x)| \leq
\begin{cases}
C_{|\alpha|, d,\epsilon} \, \sigma^{|\alpha|} (\sigma + \sigma^2)^{|\alpha|} & \text{if } \|x\| \leq \sqrt{d} \, \sigma, \\[8pt]
C_{|\alpha|, d,\epsilon} \sigma^{|\alpha|} \bigl(\sqrt{\sigma \|x\|} + \sigma \|x\|\bigr)^{|\alpha|} & \text{if } \|x\| > \sqrt{d} \, \sigma.
\end{cases}\label{eqn:piecewise bounds}
\]
\end{lemma}

\begin{theorem}\label{thm:ent map covering number}
Fix $\epsilon>0$ and $\sigma>0$.  Let $\rho\in\mathcal{G}_\sigma(\mathbb{R}^d)$ and let $\mathbb{T}^\epsilon_\rho(\mathcal{G}_\sigma(\mathbb{R}^d)):=\{T^\epsilon_{\rho\rightarrow \nu}|\nu \in \mathcal{G}_\sigma(\mathbb{R}^d)\}$ for some $\epsilon>0$. Then
\begin{equation*}
\log \mathcal{N}(\tau;\mathbb{T}^\epsilon_\rho(\mathcal{G}_\sigma(\mathbb{R}^d)),L^2(\rho))\leq C_{d,\epsilon,\sigma}\tau^{-d/s}
\end{equation*}
where $s=\lceil \frac{d}{2}\rceil +1$ and $C_{d,\epsilon,\sigma}$ is a constant depending on $d,\epsilon$ and $\sigma.$
\end{theorem}
\begin{proof}
We will use $C_*$ to denote a constant depending on $*$ that may differ from line to line. For all $1\leq i \leq d$, let $f^{\epsilon,i}_\nu:=\frac{d}{dx_i}f_{\rho\rightarrow \nu}^\epsilon$. By Lemma \ref{lemma:mena lemma 3}, we have for any multi-index $\alpha\in \mathbb{N}^d$:
\begin{align*}
|D^\alpha f^{\epsilon,i}_\nu(x)| \leq
\begin{cases}
C_{|\alpha|+1, d,\epsilon} \, \sigma^{|\alpha|+1} (\sigma + \sigma^2)^{|\alpha|+1} & \text{if } \|x\| \leq \sqrt{d} \, \sigma, \\[8pt]
C_{|\alpha|+1, d,\epsilon} \sigma^{|\alpha|+1} \bigl(\sqrt{\sigma \|x\|} + \sigma \|x\|\bigr)^{|\alpha|+1} & \text{if } \|x\| > \sqrt{d} \, \sigma,
\end{cases}
\end{align*}
for a constant $C_{|\alpha|+1,d,\epsilon}$, depending on $|\alpha|,d$ and $\epsilon$. Since $\sqrt{\|x\|}\leq \|x\|$ for all $\|x\|\geq 1$, we can coarsely upper bound this by:
\begin{align*}
|D^\alpha f^{\epsilon,i}_\nu(x)|\leq C_{|\alpha|+1,d,\epsilon,\sigma}\sigma^{|\alpha|+1}(1+\|x\|)^{|\alpha|+1}
\end{align*}
for some constant $C_{|\alpha|+1,d,\epsilon,\sigma}$ depending on $|\alpha|,d,\epsilon$ and $\sigma.$ If we define $s=\lceil\frac{d}{2}\rceil+1$, then for all $\alpha$ with $|\alpha|\leq s$ we have:
\begin{equation*}
\|D^\alpha f^{\epsilon,i}_\nu(x)\|\leq C_{d,\epsilon,\sigma}(1+\|x\|)^{s+1}
\end{equation*}
for all $1\leq i \leq d$ and $\nu\in \mathcal{G}_\sigma(\mathbb{R}^d).$ Hence we can apply Corollary \ref{cor:regular function class subgaussian} with $q=1$ to $\mathcal{F}^{\epsilon,\sigma}_\rho:=\{f^{\epsilon,i}_\nu|\nu\in \mathcal{G}_\sigma(\mathbb{R}^d),1\leq i\leq d\}$ to bound:
\begin{align}
\log \mathcal{N}(\tau; \mathcal{F}^{\epsilon,\sigma}_\rho,L^2(\rho))&\leq C_{d,\epsilon,\sigma}(1+\sigma^{d+4})(1+\sigma^d)\tau^{-d/s}\notag\\&\leq  C_{d,\epsilon,\sigma}\cdot\tau^{-d/s},\label{eqn:coordinate class cover}
\end{align}
where $C_{d,\epsilon,\sigma}$ depends on $d,\epsilon$ and $\sigma.$ We now consider $\mathbb{T}^\epsilon_\rho(\mathcal{G}_\sigma(\mathbb{R}^d))$: Recall that we may write:
\begin{align*}
T^\epsilon_{\rho\rightarrow \nu}(x)&=\nabla \phi^\epsilon_{\rho\rightarrow \nu}(x)\\&=\nabla\frac{\|x\|^2}{2}-\nabla f_{\rho\rightarrow \nu}^\epsilon(x)\\&=(x_1-f_\nu^{\epsilon,1}(x),x_2-f_\nu^{\epsilon,2}(x),...,x_d-f_\nu^{\epsilon,d}(x))^\top,
\end{align*}
so that $$\mathbb{T}^\epsilon_{\rho}(\mathcal{G}_\sigma(\mathbb{R}^d))\subseteq  \prod_{i=1}^d \mathcal{F}^{\epsilon,\sigma}_\rho-Id:=\{f-Id|f\in \prod_{i=1}^d \mathcal{F}^{\epsilon,\sigma}_\rho\}.$$
We compute, for any $\nu,\tilde{\nu}\in \mathcal{G}_\sigma(\mathbb{R}^d)$:
\begin{align*}
\|T^\epsilon_{\rho\rightarrow \nu}-T_{\rho\rightarrow \tilde{\nu}}^\epsilon\|_{L^2(\rho)}&=\left(\int \|T^\epsilon_{\rho\rightarrow \nu}(x)-T^\epsilon_{\rho\rightarrow \tilde{\nu}}(x)\|^2d\rho(x)\right)^{1/2}\\&=\left(\int \sum_{i=1}^d |x_i-f^{\epsilon,i}_\nu(x)-x_i+f^{\epsilon,i}_{\tilde{\nu}}(x)|^2 d\rho(x)\right)^{1/2}\\&=\left(\sum_{i=1}^d\|f^{\epsilon,i}_{\tilde{\nu}}-f^{\epsilon,i}_\nu\|^2_{L^2(\rho)}\right)^{1/2},
\end{align*}
which gives the following covering number bound: For any $\tau>0$, we have:
\begin{align*}
\mathcal{N}(\tau;\mathbb{T}^\epsilon_\rho(\mathcal{G}_\sigma(\mathbb{R}^d)),L^2(\rho)) &\leq \mathcal{N}\left(\tau\;; \prod_{i=1}^d\mathcal{F}^{\epsilon,\sigma}_\rho-Id,L^2(\rho)\right)\\&=\mathcal{N}\left(\sqrt{\sum_{i=1}^d\tau_i^2}\;; \prod_{i=1}^d\mathcal{F}^{\epsilon,\sigma}_\rho,L^2(\rho)\right)\\&\leq \prod_{i=1}^d\mathcal{N}\left(\tau_i;\mathcal{F}^{\epsilon,\sigma}_\rho,L^2(\rho)\right)\\&=\mathcal{N}\left(\tau_i;\mathcal{F}^{\epsilon,\sigma}_\rho,L^2(\rho)\right)^d
\end{align*}
for any $\tau_1,\tau_2,...,\tau_d\geq 0$ with $\sqrt{\sum_{i=1}^d\tau_i^2}=\tau.$ Taking $\tau_i=\frac{\tau}{\sqrt{d}}$ we conclude:
\begin{align*}
\mathcal{N}(\tau;\mathbb{T}^\epsilon_\rho(\mathcal{G}_\sigma(\mathbb{R}^d)),L^2(\rho)) \leq \mathcal{N}\left(\frac{\tau}{\sqrt{d}};\mathcal{F}^{\epsilon,\sigma}_\rho,L^2(\rho)\right)^d.
\end{align*}
Applying (\ref{eqn:coordinate class cover}), we have:
\begin{align*}
\log\mathcal{N}(\tau;\mathbb{T}^\epsilon_\rho(\mathcal{G}_\sigma(\mathbb{R}^d)),L^2(\rho))&\leq C_{d,\epsilon,\sigma}\cdot d\cdot \left(\frac{\tau}{\sqrt{d}}\right)^{-d/s}\\&=C_{d,\epsilon,\sigma}\tau^{-d/s}.
\end{align*}
\end{proof}

\subsection{Example Generalization Error Bounds for LWDL}\label{sec:example generalization bounds}

The following lemma allows us to bound the integral in Theorem \ref{thm:samp_complex_dl} based on covering number bounds on $\mathbb{T}:$

\begin{lemma}\label{lemma:explicit rates}
Let $\mathbb{T}\subset L^2(\rho)$, and define $\mathcal{N}_*:=\inf_{\delta\geq 0}\left\{C_1\delta +\frac{C_2}{\sqrt{n}}\int_\delta^{C_3}\sqrt{\log\left(\mathcal{N}(\tau;\mathbb{T},L^2(\rho))\right)}d\tau \right\}$. Suppose that $\log\left(\mathcal{N}(\tau;\mathbb{T},L^2(\rho))\right)\leq( C_4\tau)^{-p}$ for some $p>0$. Then:
\[
\mathcal{N}_* \;\leq\;
\begin{cases}
\dfrac{p}{p-2}C_1^{1-2/p}C_2^{2/p}C_4^{-1}n^{-1/p},
& p>2, \\[1.2em]
\dfrac{2C_2\log n}{C_4\sqrt{n}}+o(\log(n)n^{-1/2}),
& p=2, \\[1.2em]
\dfrac{C_2C_3^{1-p/2}C_4^{-p/2}}{(1-p/2)\sqrt{n}}+o(n^{-1/2}),
& 0<p<2.
\end{cases}
\]
\end{lemma}
Notice that Lemma \ref{lemma:explicit rates} can be applied to the integral in Theorem \ref{thm:samp_complex_dl}, as
\begin{align*}
&\inf_{\delta\geq 0}8\delta+\frac{C\cdot C_{\vartheta,c}}{\sqrt{n}}\int_{\delta/(4C_{\vartheta,c})}^{\vartheta/(2C_{\vartheta,c})}\sqrt{m\log\left(\mathcal{N}\left(u;\mathbb{T},L^2(\rho)\right)\right)}du\\&=\inf_{\delta\geq 0}32C_{\vartheta,c}\delta+\frac{C\cdot C_{\vartheta,c}}{\sqrt{n}}\int_{\delta}^{\vartheta/(2C_{\vartheta,c})}\sqrt{m\log\left(\mathcal{N}\left(u;\mathbb{T},L^2(\rho)\right)\right)}du,
\end{align*}
and hence we may take $C_1:=32C_{\vartheta,c}$, $C_2:=C\cdot C_{\vartheta,c}\sqrt{m}$ and $C_3:=\vartheta/(2C_{\vartheta,c}).$

\subsubsection{Sample Complexity of LWDL-OT}

We apply the bounds in Proposition \ref{prop:OT map covering global} and Corollary \ref{cor:bounded density covering} to Theorem \ref{thm:samp_complex_dl} to obtain generalization error bounds for LWDL with optimal transport map atoms. We first establish compactness of the set of transport maps satisfying the assumptions in Corollary \ref{cor:bounded density covering}:

\begin{lemma}\label{lemma:maps compact}
 Let $\Omega\subset \mathbb{R}^d$ be compact, let $\rho\in \mathcal{P}(\Omega)$ be absolutely continuous such that $\text{supp}(\rho)=\Omega$ and $0<\alpha<\frac{d\rho}{dx}(x)<\beta<\infty$ for Lebesgue a.e $x\in \Omega$. Then $\mathbb{T}_\rho(\mathcal{P}(\Omega)):=\{T_{\rho\rightarrow \nu}|\nu\in \mathcal{P}(\Omega)\}$ is compact in $L^2(\rho)$.
\end{lemma}
\begin{proof}
See Section \ref{sec:proof maps compact}.
\end{proof}
By Corollary \ref{cor:bounded density covering}, Theorem \ref{thm:samp_complex_dl} and Lemma \ref{lemma:maps compact}, we have the following:

\begin{corollary}\label{cor:ot map DL final}
Let $\Theta,g,h_\rho^g$ be defined as in Theorem \ref{thm:samp_complex_dl}. Let $\Omega=\prod_{j=1}^d I_j\subset \mathbb{R}^d$, where $I_j\subset \mathbb{R}$ is a closed interval with $|I_j|=D$, and let $\rho\in \mathcal{P}(\Omega)$ be absolutely continuous such that $\text{supp}(\rho)=\Omega$ and $0<\alpha<\frac{d\rho}{dx}(x)<\beta<\infty$ for Lebesgue a.e. $x\in \Omega$.  Furthermore, let $\mathbb{T}=\{T\in \mathbb{T}_\rho(\mathcal{P}_2(\Omega))| \|T\|_{L^2(\rho)}=1\}$, and define $\mathcal{T}^n,\mathcal{T}^*$ as in Theorem \ref{thm:samp_complex_dl}. Then we have:
\begin{align}\label{eqn:sample linear bound}
\mathbb{E}_{\mu,\mu_1,\mu_2,\ldots,\mu_n\sim\Theta}
\Bigl[
h^g_\rho(\mu,0,\mathcal{T}^{n})-h^g_\rho(\mu,0,\mathcal{T}^*)
\Bigr]
\leq
\begin{cases}
\text{$\mathcal{O}(m^{1/2}n^{-1/2})$}, & \text{if $d=1$},\\[0.8em]
\text{$\mathcal{O}(m^{1/d}n^{-2/(3d)})$}, & \text{if $d>1$}.
\end{cases}
\end{align}
where the constant factors depend on $D,d$, $\beta$ and $\min\{c,c^2\}^{-1}.$
\end{corollary}
\begin{proof}
Since $\mathbb{T}\subset \mathbb{T}_\rho(\mathcal{P}_2(\Omega))$ and $\mathbb{T}$ is closed, $\mathbb{T}$ is compact in $L^2(\rho)$ by Lemma \ref{lemma:maps compact}, and $\mathcal{N}(\tau;\mathbb{T},L^2(\rho))\leq \mathcal{N}(\tau;\mathbb{T}_\rho(\mathcal{P}_2(\Omega)),L^2(\rho))$ for all $\tau>0$. The result then follows by Theorem \ref{thm:samp_complex_dl}, Corollary \ref{cor:bounded density covering} and Lemma \ref{lemma:explicit rates}, where $m^{1/2}$ is treated as a factor in the $C_2$ term.
\end{proof}

\subsubsection{Sample Complexity of LWDL-EOT}
We now apply Theorem \ref{thm:samp_complex_dl} and Theorem \ref{thm:ent map covering number} to obtain generalization bounds for LWDL with entropic map atoms. One way to proceed would be to establish an analog of Lemma \ref{lemma:maps compact} for entropic maps. We conjecture that $\mathbb{T}_\rho^\epsilon(\mathcal{P}(\Omega))$ is indeed compact in $L^2(\rho)$, but for simplicity will use an alternate approach using the following lemma:

\begin{lemma}\label{lemma:closure cover}
Let $(\mathcal{X},d)$ be a complete metric space and let $\mathcal{Y}\subset \mathcal{X}$, with $\mathcal{N}(\tau;\mathcal{Y},d)<\infty$ for all $\tau>0$. Then $\overline{\mathcal{Y}}$ is compact and $\mathcal{N}(\tau+\delta,\overline{\mathcal{Y}},d)\leq \mathcal{N}(\tau,\mathcal{Y},d)$ for all $0<\tau,\delta$.
\end{lemma}
\begin{proof}
Since $(\mathcal{X},d)$ is complete and $\overline{\mathcal{Y}}$ is closed, $\overline{\mathcal{Y}}$ is complete. For any $\delta>0$, and define $S:\overline{\mathcal{Y}}\rightarrow \mathcal{Y}$ so that $S(y)=y'$ with $d(y,y')<\delta$ for any $y\in \overline{\mathcal{Y}}$. Let $\mathcal{U}_\tau:=\{U_i\}_{i=1}^{n_\tau}$ be a $\tau$-cover of $\mathcal{Y}$. We trivially have $\cup_{y\in \overline{\mathcal{Y}}}S(y)\subset \mathcal{U}_\tau$, so if we thicken each set in $U_i^\delta$ by $\delta$ and define $\mathcal{U}_{\tau+\delta}:=\{U_i^{\delta}\}_{i=1}^{n_\tau}$, then $\overline{\mathcal{Y}}\subset\mathcal{U}_{\tau+\delta}$ by the triangle inequality, which implies the inequality. Hence $\overline{\mathcal{Y}}$ is a closed and totally bounded subset of a metric space, and hence $\overline{\mathcal{Y}}$ is compact.
\end{proof}

\begin{corollary}\label{cor:EOT map dl final}
Let $\Theta,g,h_\rho^g$ be defined as in Theorem \ref{thm:samp_complex_dl}. Let $\rho\in \mathcal{G}_\sigma(\mathbb{R}^d)$. Furthermore, let $\mathring{\mathbb{T}}=\{T\in \mathbb{T}^\epsilon_\rho(\mathcal{G}_\sigma(\mathbb{R}^d))| \|T\|_{L^2(\rho)}=1\}$ for some $\epsilon>0$, let $\mathbb{T}$ be the closure of $\mathring{\mathbb{T}}$ in $L^2(\rho)$, and define $\mathcal{T}^n,\mathcal{T}^*$ as in Theorem \ref{thm:samp_complex_dl}. Then we have, for any $\tilde{\epsilon}>0$:
\begin{align*}
\mathbb{E}_{\mu,\mu_1,\mu_2,...,\mu_n\sim\Theta}[h^g_\rho(\mu,\tilde{\epsilon},\mathcal{T}^{n})-h^g_\rho(\mu,\tilde{\epsilon},\mathcal{T}^*)]&\leq \mathcal{O}(m^{1/2}n^{-1/2}),
\end{align*}
where the constant factors depend on $\sigma,d$, $\epsilon$, $\vartheta$, and $\min\{c,c^2\}^{-1}.$
\end{corollary}
\begin{proof}
First note that, by Lemma \ref{lemma:normsubgauss}, $\mathcal{G}_\sigma(\mathbb{R}^d)\subset \mathcal{M}_2(q_G^2d\sigma^2)$ for an absolute constant $q_G$. Since $\mathring{\mathbb{T}}\subset \mathbb{T}^\epsilon_\rho(\mathcal{G}_\sigma(\mathbb{R}^d))$, we have $\mathcal{N}(\tau;\mathring{\mathbb{T}},L^2(\rho))\leq \mathcal{N}(\tau;\mathbb{T}^\epsilon_\rho (\mathcal{G}_\sigma(\mathbb{R}^d),L^2(\rho)).$ By Lemma \ref{lemma:closure cover}, $\mathbb{T}$ is compact and we have $\mathcal{N}(\tau(1+\delta);\mathbb{T},L^2(\rho))\leq \mathcal{N}(\tau;\mathbb{T}^\epsilon_{\rho}(\mathcal{G}_\sigma(\mathbb{R}^d)),L^2(\rho))$ for any $\delta>0$. Replacing $\tau$ by $\tau/(1+\delta)$ in the preceding inequality and applying Theorem \ref{thm:ent map covering number}, we obtain
\begin{equation*}
\log \mathcal{N}(\tau;\mathbb{T},L^2(\rho))
\leq
C_{d,\epsilon,\sigma}(1+\delta)^{d/s}\tau^{-d/s}
\end{equation*}
for any $\delta>0$. We conclude by applying Theorem \ref{thm:samp_complex_dl} and Lemma \ref{lemma:explicit rates}.
\end{proof}

\subsection{Proof of Theorem \ref{thm:summary theorem}}

\begin{proof}The result is a direct combination of Corollary \ref{cor:ot map DL final} and Corollary \ref{cor:EOT map dl final}.

\end{proof}

\subsection{Proof of Lemma \ref{lemma:symmetrization}}
\label{subsec:ProofLemmaSymmetrization}

Formally, the proof of Lemma \ref{lemma:symmetrization} is easily found in the literature (see e.g., Chapter 6 in \citet{vershynin2018high}). However, additional details are needed to make it rigorous for function classes with uncountable cardinality, as discussed in Section 4.4 in \citet{wainwright2019high} and Chapter 2.3 in \citet{wellner2013weak}. We establish Lemma \ref{lemma:symmetrization} by first proving the result under the assumption that $\mathbb{T}$ has finitely many elements.
We apply Corollary \ref{cor:bounded_diam} and Lemma \ref{lemma:covering under lipschitz} to transfer the compactness of $\mathbb{T}^m$ to $H_\rho^{g,\epsilon}(\mathbb{T}^m)$. We conclude the proof  of Lemma \ref{lemma:symmetrization} with an approximation argument using $\delta$-covers of $H_\rho^{g,\epsilon}(\mathbb{T}^m)$.

\noindent\textbf{Proof of Lemma \ref{lemma:symmetrization}:}

\textit{Proof of Lemma \ref{lemma:symmetrization}:}
Suppose first that $\mathbb{T}$ is finite. Then
$\mathcal{H}_{\rho}^{g,\epsilon}(\mathbb{T}^m)$ is finite, and hence
the maxima below exist almost surely.

Let $\widetilde{\mu}_1,\ldots,\widetilde{\mu}_n\sim\Theta$ be i.i.d.
and independent of $\mu_1,\ldots,\mu_n$. Thus,
$(\widetilde{\mu}_1,\ldots,\widetilde{\mu}_n)$ and
$(\mu_1,\ldots,\mu_n)$ are independent random vectors, each with
distribution $\Theta^n$. For almost every
$\mu_1,\ldots,\mu_n\sim\Theta$, Jensen's inequality gives
\begin{align*}
&\max_{h_{\mathcal{T}}\in
\mathcal{H}_{\rho}^{g,\epsilon}(\mathbb{T}^m)}
\left|
\frac{1}{n}\sum_{i=1}^n h_{\mathcal{T}}(\mu_i)
-
\mathbb{E}_{\mu\sim\Theta}[h_{\mathcal{T}}(\mu)]
\right|
\\
&=
\max_{h_{\mathcal{T}}\in
\mathcal{H}_{\rho}^{g,\epsilon}(\mathbb{T}^m)}
\left|
\mathbb{E}_{\widetilde{\mu}_1,\ldots,\widetilde{\mu}_n\sim\Theta}
\left[
\frac{1}{n}\sum_{i=1}^n
\left(
h_{\mathcal{T}}(\mu_i)
-
h_{\mathcal{T}}(\widetilde{\mu}_i)
\right)
\right]
\right|
\\
&\le
\mathbb{E}_{\widetilde{\mu}_1,\ldots,\widetilde{\mu}_n\sim\Theta}
\left[
\max_{h_{\mathcal{T}}\in
\mathcal{H}_{\rho}^{g,\epsilon}(\mathbb{T}^m)}
\left|
\frac{1}{n}\sum_{i=1}^n
\left(
h_{\mathcal{T}}(\mu_i)
-
h_{\mathcal{T}}(\widetilde{\mu}_i)
\right)
\right|
\right].
\end{align*}
Taking expectations with respect to $\mu_1,\ldots,\mu_n$, we obtain
\begin{align*}
&\mathbb{E}_{\mu_i\sim\Theta}
\left[
\max_{h_{\mathcal{T}}\in
\mathcal{H}_{\rho}^{g,\epsilon}(\mathbb{T}^m)}
\left|
\frac{1}{n}\sum_{i=1}^n h_{\mathcal{T}}(\mu_i)
-
\mathbb{E}_{\mu\sim\Theta}[h_{\mathcal{T}}(\mu)]
\right|
\right]
\\
&\le
\mathbb{E}_{\mu_i,\widetilde{\mu}_i\sim\Theta}
\left[
\max_{h_{\mathcal{T}}\in
\mathcal{H}_{\rho}^{g,\epsilon}(\mathbb{T}^m)}
\left|
\frac{1}{n}\sum_{i=1}^n
\left(
h_{\mathcal{T}}(\mu_i)
-
h_{\mathcal{T}}(\widetilde{\mu}_i)
\right)
\right|
\right].
\end{align*}

Let $\sigma_1,\ldots,\sigma_n$ be the Rademacher random variables
from the statement of the lemma. For any fixed realization of
$\sigma_1,\ldots,\sigma_n$, swap $\mu_i$ and
$\widetilde{\mu}_i$ whenever $\sigma_i=-1$. Since each pair
$(\mu_i,\widetilde{\mu}_i)$ is exchangeable, this operation does
not change the joint distribution of
\[
(\mu_1,\widetilde{\mu}_1,\ldots,\mu_n,\widetilde{\mu}_n).
\]
Moreover, simultaneously for every
$h_{\mathcal{T}}\in
\mathcal{H}_{\rho}^{g,\epsilon}(\mathbb{T}^m)$, this operation
replaces
\[
h_{\mathcal{T}}(\mu_i)
-
h_{\mathcal{T}}(\widetilde{\mu}_i)
\]
by
\[
\sigma_i
\left(
h_{\mathcal{T}}(\mu_i)
-
h_{\mathcal{T}}(\widetilde{\mu}_i)
\right).
\]
Consequently,
\begin{align*}
&\mathbb{E}_{\mu_i,\widetilde{\mu}_i\sim\Theta}
\left[
\max_{h_{\mathcal{T}}\in
\mathcal{H}_{\rho}^{g,\epsilon}(\mathbb{T}^m)}
\left|
\frac{1}{n}\sum_{i=1}^n
\left(
h_{\mathcal{T}}(\mu_i)
-
h_{\mathcal{T}}(\widetilde{\mu}_i)
\right)
\right|
\right]
\\
&=
\mathbb{E}_{\sigma_i,\mu_i,\widetilde{\mu}_i}
\left[
\max_{h_{\mathcal{T}}\in
\mathcal{H}_{\rho}^{g,\epsilon}(\mathbb{T}^m)}
\left|
\frac{1}{n}\sum_{i=1}^n
\sigma_i
\left(
h_{\mathcal{T}}(\mu_i)
-
h_{\mathcal{T}}(\widetilde{\mu}_i)
\right)
\right|
\right]
\\
&\le
\mathbb{E}_{\sigma_i,\mu_i}
\left[
\max_{h_{\mathcal{T}}\in
\mathcal{H}_{\rho}^{g,\epsilon}(\mathbb{T}^m)}
\left|
\frac{1}{n}\sum_{i=1}^n
\sigma_i h_{\mathcal{T}}(\mu_i)
\right|
\right]
\\
&\quad+
\mathbb{E}_{\sigma_i,\widetilde{\mu}_i}
\left[
\max_{h_{\mathcal{T}}\in
\mathcal{H}_{\rho}^{g,\epsilon}(\mathbb{T}^m)}
\left|
\frac{1}{n}\sum_{i=1}^n
\sigma_i h_{\mathcal{T}}(\widetilde{\mu}_i)
\right|
\right]
\\
&=
2\mathbb{E}_{\sigma_i,\mu_i}
\left[
\max_{h_{\mathcal{T}}\in
\mathcal{H}_{\rho}^{g,\epsilon}(\mathbb{T}^m)}
\left|
\frac{1}{n}\sum_{i=1}^n
\sigma_i h_{\mathcal{T}}(\mu_i)
\right|
\right].
\end{align*}
We have therefore proved
\begin{equation}\label{eqn:symm bound}
\begin{aligned}
&\mathbb{E}_{\mu_i\sim\Theta}
\left[
\max_{h_{\mathcal{T}}\in
\mathcal{H}_{\rho}^{g,\epsilon}(\mathbb{T}^m)}
\left|
\frac{1}{n}\sum_{i=1}^n h_{\mathcal{T}}(\mu_i)
-
\mathbb{E}_{\mu\sim\Theta}[h_{\mathcal{T}}(\mu)]
\right|
\right]
\\
&\qquad\le
2\mathbb{E}_{\sigma_i,\mu_i}
\left[
\max_{h_{\mathcal{T}}\in
\mathcal{H}_{\rho}^{g,\epsilon}(\mathbb{T}^m)}
\left|
\frac{1}{n}\sum_{i=1}^n
\sigma_i h_{\mathcal{T}}(\mu_i)
\right|
\right].
\end{aligned}
\end{equation}

We now proceed with the general case. If
\[
\mathbb{E}_{\sigma_i,\mu_i}
\left[
\sup_{h_{\mathcal{T}}\in
\mathcal{H}_{\rho}^{g,\epsilon}(\mathbb{T}^m)}
\left|
\frac{1}{n}\sum_{i=1}^n
\sigma_i h_{\mathcal{T}}(\mu_i)
\right|
\right]
=
\infty,
\]
the result holds trivially. Hence, suppose that
\[
\mathbb{E}_{\sigma_i,\mu_i}
\left[
\sup_{h_{\mathcal{T}}\in
\mathcal{H}_{\rho}^{g,\epsilon}(\mathbb{T}^m)}
\left|
\frac{1}{n}\sum_{i=1}^n
\sigma_i h_{\mathcal{T}}(\mu_i)
\right|
\right]
<
\infty.
\]

Since $\mathbb{T}\subset L^2(\rho)$ is compact,
$\mathbb{T}^m$ is compact with respect to $d_{2,m}$. Thus, for
every $\delta>0$, there exists a finite $\delta$-cover
\[
\mathbb{T}_{\delta}^m
:=
\left\{
\mathcal{T}^{(1)},\ldots,\mathcal{T}^{(k_\delta)}
\right\}
\subseteq\mathbb{T}^m.
\]
By Corollary \ref{cor:bounded_diam} for every $\mathcal{T}\in\mathbb{T}^m$ there
exists $\mathcal{T}^{(r)}\in\mathbb{T}_{\delta}^m$ such that
\[
\sup_{\mu\in\mathcal{M}_2(\vartheta)}
\left|
h_{\mathcal{T}}(\mu)
-
h_{\mathcal{T}^{(r)}}(\mu)
\right|
\le
C_{\vartheta,c}\delta.
\]
Therefore, for every
$\mu_1,\ldots,\mu_n\in\mathcal{M}_2(\vartheta)$,
\begin{align*}
&
\left|
\frac{1}{n}\sum_{i=1}^n h_{\mathcal{T}}(\mu_i)
-
\mathbb{E}_{\mu\sim\Theta}[h_{\mathcal{T}}(\mu)]
\right|
\\
&\le
\left|
\frac{1}{n}\sum_{i=1}^n
h_{\mathcal{T}^{(r)}}(\mu_i)
-
\mathbb{E}_{\mu\sim\Theta}
[h_{\mathcal{T}^{(r)}}(\mu)]
\right|
\\
&\quad+
\frac{1}{n}\sum_{i=1}^n
\left|
h_{\mathcal{T}}(\mu_i)
-
h_{\mathcal{T}^{(r)}}(\mu_i)
\right|
+
\mathbb{E}_{\mu\sim\Theta}
\left[
\left|
h_{\mathcal{T}}(\mu)
-
h_{\mathcal{T}^{(r)}}(\mu)
\right|
\right]
\\
&\le
\left|
\frac{1}{n}\sum_{i=1}^n
h_{\mathcal{T}^{(r)}}(\mu_i)
-
\mathbb{E}_{\mu\sim\Theta}
[h_{\mathcal{T}^{(r)}}(\mu)]
\right|
+
2C_{\vartheta,c}\delta.
\end{align*}
It follows that
\begin{align*}
&\mathbb{E}_{\mu_i\sim\Theta}
\left[
\sup_{h_{\mathcal{T}}\in
\mathcal{H}_{\rho}^{g,\epsilon}(\mathbb{T}^m)}
\left|
\frac{1}{n}\sum_{i=1}^n h_{\mathcal{T}}(\mu_i)
-
\mathbb{E}_{\mu\sim\Theta}[h_{\mathcal{T}}(\mu)]
\right|
\right]
\\
&\le
\mathbb{E}_{\mu_i\sim\Theta}
\left[
\max_{h_{\mathcal{T}}\in
\mathcal{H}_{\rho}^{g,\epsilon}(\mathbb{T}_{\delta}^m)}
\left|
\frac{1}{n}\sum_{i=1}^n h_{\mathcal{T}}(\mu_i)
-
\mathbb{E}_{\mu\sim\Theta}[h_{\mathcal{T}}(\mu)]
\right|
\right]
+
2C_{\vartheta,c}\delta
\\
&\le
2\mathbb{E}_{\sigma_i,\mu_i}
\left[
\max_{h_{\mathcal{T}}\in
\mathcal{H}_{\rho}^{g,\epsilon}(\mathbb{T}_{\delta}^m)}
\left|
\frac{1}{n}\sum_{i=1}^n
\sigma_i h_{\mathcal{T}}(\mu_i)
\right|
\right]
+
2C_{\vartheta,c}\delta
\\
&\le
2\mathbb{E}_{\sigma_i,\mu_i}
\left[
\sup_{h_{\mathcal{T}}\in
\mathcal{H}_{\rho}^{g,\epsilon}(\mathbb{T}^m)}
\left|
\frac{1}{n}\sum_{i=1}^n
\sigma_i h_{\mathcal{T}}(\mu_i)
\right|
\right]
+
2C_{\vartheta,c}\delta,
\end{align*}
where the second inequality follows from (\ref{eqn:symm bound}). Sending
$\delta\to0^{+}$ proves the result.

\subsection{Proof of Corollary \ref{cor:regular function class subgaussian}}
\label{subsec:ProofCor3.10}
\begin{proof}
Below, let $C_*$ denote constants depending on $*$ whose precise value can vary line by line. Define dyadic shells $I_j$ by
\begin{equation*}
I_0=[-\sigma,\sigma]^d
\end{equation*}
\begin{equation*}
I_j=[-2^j\sigma,2^j\sigma]^d\setminus (-2^{j-1}\sigma,2^{j-1}\sigma)^d, \; \; \; j\geq 1
\end{equation*}
For each $j$, $\mathcal{L}(I_j^1)\leq C_d(1+\sigma^d2^{dj})$ for some dimension-dependent constant $C_d$. Furthermore, for $j\geq1$, $k\in\{1,\ldots,d\}$, and $\eta\in\{-1,1\}$, define
\begin{equation*}
I_j^{k,\eta}:=\{x\in\mathbb{R}^d:\|x\|_\infty\leq2^j\sigma,\ \eta x_k\geq2^{j-1}\sigma,\ \eta x_k\geq|x_\ell|\text{ for all }\ell\in\{1,\ldots,d\}\}.
\end{equation*}
Each $I_j^{k,\eta}$ is convex, and $I_j=\bigcup_{k=1}^d\bigcup_{\eta\in\{-1,1\}}I_j^{k,\eta}$, with overlaps only along their boundaries. Together with $I_0$, a fixed half-open tie-breaking of the common boundary hyperplanes gives a partition of $\mathbb{R}^d$ into bounded convex pieces, each contained in one of the sets just displayed. Moreover, for every $j\geq 1$, $k\in\{1,\ldots,d\}$, and $\eta\in\{-1,1\}$,
\begin{align*}
\mathcal{L}\bigl((I_j^{k,\eta})^1\bigr)
&\leq \mathcal{L}(I_j^1)
\leq C_d\bigl(1+\sigma^d2^{dj}\bigr), \\
\operatorname{diam}(I_j^{k,\eta})
&\leq \operatorname{diam}(I_j), \\
\rho(I_j^{k,\eta})
&\leq \rho(I_j).
\end{align*}

Since $\rho$ is a $\sigma$-subgaussian measure, $\|X\|$ with $X\sim\rho$ is a $q_G\sigma \sqrt{d}$-subgaussian random variable by Lemma \ref{lemma:normsubgauss}. Hence, for $j\geq 1$,
\begin{align*}
\rho(I_j)&\leq \mathbb{P}_{X\sim \rho}(\|X\|> 2^{j-1}\sigma)\\&\leq 2\exp\left(\frac{-c2^{2j}\sigma^2}{\left\|\|X\|\right\|_{\mathcal{G}}^2}\right)\\&\leq 2\exp\left(\frac{-c2^{2j}}{q_G^2d}\right).
\end{align*}
Furthermore $\rho(I_0)\leq 1$, which may be absorbed into the constants below. For any $f\in \mathcal{F}$ and any convex piece $J\subseteq I_j$ with $j\geq 0$ we have:
\begin{align*}
\|f\|_{\mathcal{C}^\alpha(J)}&\leq \sup_{x\in I_j} C (1+\|x\|)^{s+q}\\&\leq C_*\left(1+\sup_{x\in I_j}\|x\|^{s+q}\right)\\&\leq C_*(1+2^{j(s+q)}\sigma^{s+q})\\&\leq C_*2^{j(s+q)}(1+\sigma^{s+q})
\end{align*}
where we have absorbed a constant of $2^{s+q-1}$ into $C_*$.
Hence $f\in \mathcal{C}_{M_j}^\alpha(J)$ on every convex piece $J\subseteq I_j$, where $M_j:=C_*2^{j(s+q)}(1+\sigma^{s+q}).$ Defining $Q:=\frac{d}{s}$, we can apply Theorem \ref{theorem:wellner non compact} to the convex refinement above. The sum over the $2d$ pieces in each shell only contributes a dimension-dependent factor, so
\begin{align*}
\log \mathcal{N}(\tau;\mathcal{F},L^2(\rho))& \leq C_*\left(\mathcal{L}(I_0^1)^{\frac{2}{Q+2}}M_0^{\frac{2Q}{Q+2}}\rho(I_0)^{\frac{Q}{Q+2}}\right.\\
&\qquad\qquad\left.+\sum_{j=1}^\infty\sum_{k=1}^d\sum_{\eta\in\{-1,1\}}\mathcal{L}((I_j^{k,\eta})^1)^{\frac{2}{Q+2}}M_j^{\frac{2Q}{Q+2}}\rho(I_j^{k,\eta})^{\frac{Q}{Q+2}}\right)^{\frac{Q+2}{2}}\cdot\tau^{-Q}\\
& \leq C_*\left(\sum_{j= 0}^\infty (1+2^{jd}\sigma^d)^{\frac{2}{Q+2}}2^{\frac{2j(d+Qq)}{Q+2}}(1+\sigma^{s+q})^{\frac{2Q}{Q+2}}\right.\\
&\qquad\qquad\left.\cdot\exp\left(\frac{-cQ2^{2j}}{q_G^2d(Q+2)}\right)\right)^{\frac{Q+2}{2}}\cdot \tau^{-Q}.
\end{align*}
We can factor out a term $(1+\sigma^{s+q})^{\frac{2Q}{Q+2}\cdot \frac{Q+2}{2}}=(1+\sigma^{s+q})^Q$:
\begin{align*}
&C_*(1+\sigma^{s+q})^Q\left(\sum_{j= 0}^\infty (1+2^{jd}\sigma^d)^{\frac{2}{Q+2}}2^{\frac{2j(d+Qq)}{Q+2}}\exp\left(\frac{-cQ2^{2j}}{q_G^2d(Q+2)}\right)\right)^{\frac{Q+2}{2}}\cdot \tau^{-Q}\\&\leq C_*(1+\sigma^{s+q})^Q\left(\sum_{j= 0}^\infty(1+\sigma^d)^{\frac{2}{Q+2}} 2^{\frac{2j(d+1+Qq)}{Q+2}}\exp\left(\frac{-cQ2^{2j}}{q_G^2d(Q+2)}\right)\right)^{\frac{Q+2}{2}}\cdot \tau^{-Q}.
\end{align*}
We then factor out $(1+\sigma^d)^{\frac{2}{Q+2}\cdot\frac{Q+2}{2}}=(1+\sigma^d)$ to obtain:
\begin{equation*}
C_*(1+\sigma^{s+q})^Q(1+\sigma^d)\left(\sum_{j= 0}^\infty 2^{\frac{2j(d+1+Qq)}{Q+2}}\exp\left(\frac{-cQ2^{2j}}{q_G^2d(Q+2)}\right)\right)^{\frac{Q+2}{2}}\cdot \tau^{-Q}.
\end{equation*}
Since the summands
\begin{align*}
&2^{\frac{2(d+1+Qq)js}{Q+2}}\exp\left(\frac{-cQ2^{2j}}{q_G^2d(Q+2)}\right)\\& \leq \exp\left(\frac{2(d+1+Qq)}{Q+2}js - \frac{cQ2^{2j}}{q_G^2 d(\frac{d}{s}+2)}\right)
\end{align*}
decay superexponentially in $j$, \begin{align*}
\sum_{j= 0}^\infty \left(1+2^{\frac{2(d+1+Qq)js}{Q+2}}\right)\exp\left(\frac{-cQ2^{2j}}{q_G^2d(Q+2)}\right)<\infty,
\end{align*}
and by collapsing this term into the constant $C_*$ we conclude that:
\begin{align*}
\log \mathcal{N}(\tau;\mathcal{F},L^2(\rho))&\leq C_{*}(1+\sigma^{s+q})^Q(1+\sigma^d)\tau^{-Q}.
\end{align*}
When $s=\lceil \frac{d}{2}\rceil+1$, $Q\leq 2$, we can absorb another factor of $2$ into $C_*$ to conclude:
\begin{align*}
\log \mathcal{N}(\tau;\mathcal{F},L^2(\rho))&\leq C_{*}(1+\sigma^{2(s+q)})(1+\sigma^d)\tau^{-2}\\&\leq C_{*}(1+\sigma^{d+2+2q})(1+\sigma^d)\tau^{-2}.
\end{align*}
Tracking constants, we see that $C_*$ depends on $d$ and $q$, since $s$ is determined by $d.$
\end{proof}
\subsection{Proof of Corollary \ref{cor:bounded density covering}}
\label{subsec:ProofCorC.13}

\begin{proof}
The result follows from Proposition \ref{prop:OT map covering global} and purely metric properties of $(\mathbb{T}_\rho(\mathcal{P}(\Omega)),L^2(\rho))$ and $(\mathbb{T}_U(\mathcal{P}(\Omega)),L^2(U)).$ First note that, for any $\mu\ll \nu$ and $\nu\ll \mu$, any map $T:\Omega\rightarrow \mathbb{R}^d$ defined $\mu$ almost everywhere is uniquely defined $\nu$ almost everywhere, and vice-versa (note that this does not automatically give an identification of elements of $L^2(U)$ and $L^2(\rho)$, which requires establishing square integrability).

By Brenier's theorem, for any maps $T_1,T_2\in \mathbb{T}_U(\mathcal{P}(\Omega)),$ there exist $U$ (and hence $\rho$) almost everywhere unique convex functions $\phi_1,\phi_2$ with $\nabla \phi_i=T_i$ and $\overline{\nabla \phi_i(\Omega)}=\text{supp}([T_i]_{\#}(U))\subseteq \Omega$, for $i\in \{1,2\}$. Then $\nabla \phi_i\in L^2(\rho)$ and, by Brenier's theorem, $\nabla \phi_i\in \mathbb{T}_\rho(\mathcal{P}(\Omega))$. We bound:
\begin{align*}
\|\nabla \phi_1 -\nabla \phi_2 \|^2_{L^2(\rho)}&=\int_\Omega \|\nabla \phi_1(x)-\nabla \phi_2(x)\|^2 \frac{d\rho(x)}{dx}dx\\&\leq \beta\int_\Omega\|\nabla \phi_1(x)-\nabla\phi_2(x)\|^2dx\\&=\mathcal{L}(\Omega)\beta\|\nabla \phi_1-\nabla\phi_2\|_{L^2(U)}^2.
\end{align*}
We have thus shown the map $\iota:\nabla \phi\in L^2(U)\rightarrow \nabla \phi\in L^2(\rho)$ defines a $(\mathcal{L}(\Omega)\beta)^{1/2}$-Lipschitz embedding, and by $\rho$ almost everywhere uniqueness of Brenier potentials this map is injective.

The same argument with $U$ and $\rho$ interchanged shows that every element of $\mathbb{T}_\rho(\mathcal{P}(\Omega))$ also lies in $\mathbb{T}_U(\mathcal{P}(\Omega))$ after identifying almost everywhere representatives. Hence $\iota$ is a bijection between $\mathbb{T}_U(\mathcal{P}(\Omega))$ and $\mathbb{T}_\rho(\mathcal{P}(\Omega))$.

Applying Lemma \ref{lemma:covering under lipschitz} to $\iota$, we conclude that
\begin{align*}
\log  \mathcal{N}(\tau;\mathbb{T}_\rho(\mathcal{P}(\Omega)),L^2(\rho))& \leq  \log \mathcal{N}((\mathcal{L}(\Omega)\beta)^{-1/2}\tau;\mathbb{T}_U(\mathcal{P}(\Omega)),L^2(U))\\&\leq C^*_{D,d}((\mathcal{L}(\Omega)\beta)^{-1/2} \tau)^{-3d/2} \tag{Proposition \ref{prop:OT map covering global}}.
\end{align*}
\end{proof}

\subsection{Proof of Lemma \ref{lemma:maps compact}}\label{sec:proof maps compact}

We will use the following fact about stability for optimal transport maps, which is a special case of Corollary 5.23 in \citet{villani2009optimal}:
\begin{lemma}\label{lemma:maps converge in prob}
Let $\Omega\subset \mathbb{R}^d$, and let $\{\mu_k\}_{k=1}^\infty\subset \mathcal{P}_2(\Omega)$ be a sequence of probability measures, converging weakly to $\mu\in \mathcal{P}_2(\Omega).$ Let $\rho\in \mathcal{P}_2(\Omega)$, and assume that $T_{\rho\rightarrow \mu_k}$ exists for all $k\in \mathbb{N}$ and that $T_{\rho\rightarrow \mu}$ exists. Then:
\begin{align*}
\forall \epsilon>0, \; \; \rho(\{x\in \Omega|\|T_{\rho\rightarrow \mu_k}(x)-T_{\rho\rightarrow \mu}(x)\|\geq \epsilon\})\rightarrow 0, \; \; \; k\rightarrow \infty,
\end{align*}
i.e. $T_{\rho\rightarrow \mu_k}$ converges to $T_{\rho\rightarrow \mu}$ in $\rho$-probability.
\end{lemma}
We will also need a general convergence result, which will help us promote convergence in probability to convergence in $L^2(\rho)$:
\begin{lemma}\label{lemma:UI L2}
(Theorem 3.5 in \citet{billingsley2013convergence}) Let $\{f_n\}_{n=1}^\infty\subset L^2(\rho)$ be uniformly $\rho$-integrable, i.e.,
\begin{align*}
\lim_{C\rightarrow \infty}\sup_{n\in \mathbb{N}}\int |f_n (x)|\mathds{1}_{|f_n(x)|>C}(x)d\rho(x)= 0.\end{align*}
Then if $f_n$ converges to $f$ in $\rho$-probability, $\int f_nd\rho\rightarrow \int fd\rho$ as $n\rightarrow \infty.$
\end{lemma}
\noindent\textbf{Proof of Lemma \ref{lemma:maps compact}}:
Suppose that $\{\mu_k\}_{k=1}^\infty \subset \mathcal{P}_2(\Omega)$. Since $\Omega$ is compact, $\mathcal{P}_2(\Omega)$ is weakly compact, and hence there exists a subsequence $\{\mu_k'\}_{k=1}^\infty$ that weakly converges to some $\mu^*$. Since $\rho$ is absolutely continuous, Theorem \ref{thm:brenier} implies that $T_{\rho\rightarrow \mu'_k}$ exists for all $k\in \mathbb{N}$, and that $T_{\rho\rightarrow \mu^*}$ exists. By Lemma \ref{lemma:maps converge in prob}, we have convergence in $\rho$-probability of $T_{\rho\rightarrow \mu_k'}$ to $T_{\rho\rightarrow \mu^*}.$ Equivalently, we have convergence in $\rho$-probability of the sequence $G_k(x):=\|T_{\rho\rightarrow \mu_k'}(x)-T_{\rho\rightarrow \mu^*}(x)\|^2$ to $0$. We claim $\{G_k\}_{k=1}^\infty$ is uniformly integrable. This follows from the fact that $G_k(x)=\|T_{\rho\rightarrow \mu_k'}(x)-T_{\rho\rightarrow \mu^*}(x)\|^2\le |\Omega|^2$, $\rho$-almost everywhere, and we conclude:
\begin{align*}
\sup_{k\in \mathbb{N}}\int G_k(x)\mathds{1}_{G_k>|\Omega|^2}(x) d\rho(x)=0,
\end{align*}
which implies uniform integrability. Hence we may apply Lemma \ref{lemma:UI L2} to conclude that $\int G_kd\rho\rightarrow 0$, and hence $T_{\rho\rightarrow \mu_k'}$ converges to $T_{\rho\rightarrow \mu^*}$ in $L^2(\rho)$. \qed

\subsection{Proof of Lemma \ref{lemma:explicit rates}}

\begin{proof}
For convenience, let $C_4=1$, as we can recover the full result by scaling. For $p> 2$, we have, for any $\delta>0$:
\begin{align*}
\mathcal{N}& \leq C_1\delta + \frac{C_2}{\sqrt{n}}\int_\delta^{C_3} \tau^{-\frac{p}{2}}d\tau\\& =C_1\delta-\frac{C_2}{\sqrt{n}}\left[\frac{2}{p-2}\tau^{-\frac{(p-2)}{2}}\right]_{\delta}^{C_3}\\&= C_1\delta -\frac{C_2}{\sqrt{n}}\left(\frac{2}{p-2}C_3^{\frac{-(p-2)}{2}}-\frac{2}{p-2}\delta^{\frac{-(p-2)}{2}}\right)\\& \leq  C_1\delta +\frac{2C_2}{\sqrt{n}(p-2)}\delta^{\frac{-(p-2)}{2}}.
\end{align*}
Optimizing over $\delta$, we obtain:
\begin{equation*}
\delta^*=\left(\frac{C_2^2}{C_1^2n}\right)^{1/p}.
\end{equation*}
Substituting this into our bound, we have:
\begin{align*}
\mathcal{N}& \leq \frac{p}{p-2}C_1^{1-2/p}C_2^{2/p}n^{-1/p}.
\end{align*}

Similarly, if $p=2$ we have, for any $\delta>0$:
\begin{align*}
\mathcal{N}&\leq C_1\delta+\frac{C_2}{\sqrt{n}}\int_\delta^{C_3}\tau^{-1}d\tau\\&=C_1\delta+\frac{C_2}{\sqrt{n}}\left[\log(\tau)\right]_{\delta}^{C_3}\\&=C_1\delta+\frac{C_2}{\sqrt{n}}\left(\log(C_3)-\log(\delta)\right).
\end{align*}
Optimizing over $\delta$, we have:
\begin{align*}
\delta^*=\frac{C_2}{C_1\sqrt{n}}.
\end{align*}
Hence we have:
\begin{align*}
\mathcal{N}\leq \frac{C_2\log(n)}{2\sqrt{n}}+o\left(\frac{\log(n)}{\sqrt{n}}\right)
\end{align*}
which is asymptotically dominated by $\frac{C_2\log(n)}{2\sqrt{n}}.$

Finally, if $p<2$, we have:
\begin{align*}
\mathcal{N} &\leq C_1\delta +\frac{C_2}{\sqrt{n}}\int_\delta^{C_3}\tau^{-p/2}d\tau\\&=C_1\delta+\frac{C_2}{(1-\frac{p}{2})\sqrt{n}}\left[\tau^{1-p/2}\right]_\delta^{C_3}\\&=C_1\delta+\frac{C_2}{(1-\frac{p}{2})\sqrt{n}}\left[C_3^{1-p/2}-\delta^{1-p/2}\right].
\end{align*}
Optimizing over $\delta$, we have:
\begin{align*}
\delta^*&=\left(\frac{C_2^2}{C_1^2n}\right)^{1/p}.
\end{align*}
Hence we have:
\begin{align*}
\mathcal{N} \leq O\left(n^{-1/p}\right)+\frac{2C_2C_3^{1-p/2}}{(1-p/2)\sqrt{n}}
\end{align*}
which is asymptotically dominated by $\frac{C_2C_3^{1-p/2}}{(1-p/2)\sqrt{n}}$.
\end{proof}


\end{document}